\documentclass{article}
\pdfoutput=1   
\usepackage{iclr2027_conference}
\iclrfinalcopy

\usepackage{amsmath,amssymb,amsthm,mathtools}
\usepackage[T1]{fontenc}
\usepackage{mathptmx}
\usepackage{booktabs,multirow,array}
\usepackage{float}
\usepackage{algorithm}
\usepackage{algpseudocode}
\usepackage{placeins}
\usepackage{graphicx}
\usepackage{xcolor}
\newcommand{\rev}[1]{#1}   
\usepackage{microtype}
\usepackage{url}
\usepackage[colorlinks=true,linkcolor=blue!60!black,citecolor=blue!60!black,urlcolor=blue!60!black]{hyperref}
\hypersetup{
  pdftitle={The Loss Does Not See the Basis, but Adam Does},
  pdfauthor={Devender Singh},
  pdfsubject={Implicit bias of optimizers in factored models},
  pdfkeywords={implicit bias, orthogonal gauge equivariance, Adam, gradient descent,
    matrix factorization, low-rank, preconditioning, Muon, Shampoo},
  bookmarksnumbered=true,
}
\usepackage{caption}
\makeatletter
\def\thm@space@setup{\thm@preskip=5pt plus 2pt minus 2pt \thm@postskip=5pt plus 2pt minus 2pt}
\makeatother

\newtheorem{theorem}{Theorem}[section]
\newtheorem{proposition}[theorem]{Proposition}
\newtheorem{lemma}[theorem]{Lemma}
\newtheorem{corollary}[theorem]{Corollary}
\theoremstyle{definition}
\newtheorem{definition}[theorem]{Definition}
\theoremstyle{remark}
\newtheorem{remark}[theorem]{Remark}

\makeatletter
\AtBeginEnvironment{proof}{\let\orig@trivlist\trivlist\def\trivlist{\topsep=6pt plus 1pt\relax\orig@trivlist}}
\makeatother
\newcommand{\R}{\mathbb{R}}
\newcommand{\OO}{\mathrm{O}}
\newcommand{\msign}{\mathrm{msign}}
\newcommand{\er}{\mathrm{erank}}
\newcommand{\rec}{\mathrm{rec}}
\newcommand{\tr}{\mathrm{tr}}
\newcommand{\norm}[1]{\left\lVert #1 \right\rVert}
\newcommand{\ip}[2]{\left\langle #1, #2 \right\rangle}
\DeclareMathOperator*{\argmax}{arg\,max}
\newcommand{\pstep}[2]{\par\noindent\textbf{Step~#1: #2.}\enspace}
\numberwithin{equation}{section}
\newcolumntype{L}[1]{>{\raggedright\arraybackslash}p{#1}}

\newcommand{\methodlink}[2]{\hyperlink{cite.#1}{\textcolor{blue!60!black}{#2}}}

\newenvironment{spec}
  {\begin{list}{\textbullet}{%
     \setlength{\leftmargin}{1.5em}\setlength{\labelwidth}{1.0em}\setlength{\labelsep}{0.5em}
     \setlength{\itemsep}{1.5pt}\setlength{\parsep}{0pt}\setlength{\topsep}{3pt}
     \setlength{\partopsep}{0pt}\setlength{\listparindent}{0pt}}}
  {\end{list}}

\makeatletter
\newcommand{\panellabel}[2]{\def\@currentlabel{\thefigure\,(#2)}\label{#1}}
\makeatother

\title{The Loss Does Not See the Basis,\\
but Adam Does}

\author{Devender Singh\\
Department of Computer Science\\
Memorial University of Newfoundland\\
\texttt{devenders@mun.ca}}

\begin{document}
\maketitle
\lhead{Preprint}
\addtocontents{toc}{\protect\setcounter{tocdepth}{-1}}
\begin{abstract}
Different optimization algorithms can find very different solutions to the same problem. Gradient descent (GD) exhibits a low-rank bias in overparameterized problems when initialized near zero, while Adam does not. GD is rotationally equivariant: rotating both factors of a factored model $W = UV^\top$ by an orthogonal matrix leaves the loss unchanged, and GD's solution rotates accordingly. In contrast, Adam is not. We show that any algorithm using only the current gradient and respecting rotational symmetry amounts to a left-preconditioner depending only on the gradient's Gram matrix. This gives a norm-based test, which passes for the spectral norm in Muon and fails for Adam's $\ell_\infty$ norm. Empirically, nine optimizers fall into groups consistent with the test on underdetermined matrix sensing. Sharing Adam's second moment across coordinates gradually recovers a low-rank bias, and its solutions approach GD's as the stepsize decreases. On four hyperspectral scenes, Adam's solutions are hundreds to thousands of times further from GD's than two GD runs with different step sizes are from each other. In transformers and LoRA adapters, Adam separates two rotationally equivalent models from the first iteration, and learns higher-rank attention updates. Overall, recovering particular solutions depends on the optimizer's invisible coordinate basis, as well as how fast it grows small singular values relative to large ones.
\end{abstract}

\section{Introduction}

It is well-known that gradient descent (GD) dynamics exhibit an implicit bias towards low-rank solutions \citep{arora2019implicit,razin2020implicit,li2021towards}. Specifically, given an overparameterized model that predicts using the product of two tall matrices, $W = UV^\top$, the matrix that GD recovers with small initialization tends to have lower rank compared to other possible solutions with the same loss. Such bias emerges from GD's dynamical properties, and is not enforced by an explicit regularizer. Moreover, the particular solution found by GD is not always the minimum-nuclear-norm solution conjectured by \citet{gunasekar2017implicit}.

Several popular optimizers adaptively rescale updates using an estimate of the scale of gradients in each direction, for example, Adam, RMSProp, signSGD, Lion, Adafactor. They also generalize differently from GD \citep{wilson2017marginal,keskar2017improving}. In contrast to GD, their trajectories are basis-dependent, that is, change under a rotation of parameters \citep{xie2025linfty}. Interestingly, some variants that share information across directions can be shown to be equivariant (their trajectory rotates along with the parameters), for example, Adam with a single adaptive learning rate for the entire vector and adaptive optimizers that share information across vector-valued parameters \citep{ling2022vectoradam}.

Overparameterized matrix factorization objectives have a particular symmetry: if the loss depends on the tall matrices $U, V$ only through their product, $L(U,V) = f(UV^\top)$, then for any orthogonal $Q$, we have $L(U,V) = L(UQ,VQ)$. This symmetry, which we call the \emph{gauge}, does not change the resulting matrix, $(UQ)(VQ)^\top = UQQ^\top V^\top = UV^\top$, and is thus invisible to the model's predictions. We show that popular optimizers fall into two categories w.r.t.\ this symmetry: (i) \textit{Equivariant optimizers}, such as GD, Momentum, Adam with a single shared adaptive learning rate (\emph{scalar-Adam}), Muon, Shampoo, exhibit the same trajectory up to rotation from every start $(UQ,VQ)$, the \textit{gauge orbit} of $(U,V)$; (ii) \textit{Coordinate-wise optimizers} such as Adam, RMSProp, signSGD, Lion, and Adafactor, on the other hand, behave differently depending on the gauge-equivalent point from which they start. These latter optimizers lose GD's low-rank bias: on our $40\times40$ sensing task (\S\ref{sec:setup}), after $2\times10^4$ steps, they reach higher-rank solutions and recover the ground truth worse (Figure~\ref{fig:zoo}).

\begin{figure}[t]
\centering
\includegraphics[width=0.80\linewidth]{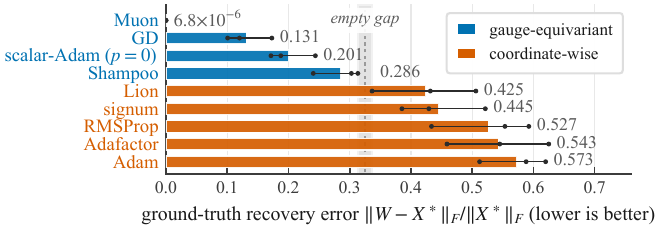}
\caption{\textbf{\rev{Recovery error of nine optimizers.}} Three-seed \emph{mean} recovery error (lower is
better) against the planted target for nine rules on underdetermined matrix sensing, all run to interpolation within $2\times10^{4}$ steps.
Dots are single seeds. The shaded band is the gap between the worst equivariant seed ($0.313$)
and the best coordinate-wise seed ($0.337$), so the two predicted classes never overlap. Details: Table~\ref{tab:zoo}.}
\label{fig:zoo}
\end{figure}

\paragraph{Contributions.}

\textbf{(1) A criterion, and the norm it corresponds to.} We give a criterion for which update rules respect the gauge in \S\ref{sec:theory}. Applied to popular optimizers, this recovers the two groups above, and shows that no rule applying a fixed entrywise nonlinearity at zero optimizer state is equivariant.
Moreover, \rev{using a classical characterization of equivariant functions \citep{villar2021scalars},} we characterize exactly which memoryless update rules (those that use only the current gradient $G$) are equivariant: precisely those that can be written $\Delta = H(GG^\top)G$. For updates defined as steepest descent under a norm, the test reduces to checking whether the norm is unchanged by rotations applied from the right (Corollary~\ref{cor:norm}). A transfer theorem and a momentum lemma extend gradient flow's pathwise conclusions to scalar-Adam-type methods and to momentum.

\textbf{(2) Nine optimizers sorted against ground truth, and an interpolation isolating the mechanism.} We then return to the matrix sensing problem to compare optimization methods (\S\ref{sec:zoo}). At $n{=}40$ there are two clusters of methods, and any method that passes the criterion has better mean-recovery than any method that fails the criterion. These differences persist as $n$ grows up to $n{=}256$, with GD, scalar-Adam, and Muon performing better than Adam, RMSProp, Lion and Adafactor. We find that the signum method recovers almost perfectly with a long enough anneal, Shampoo with fixed damping performs worse than Adam in $4/5$ problem sizes , and interpolating from Adam's coordinate-wise adaptivity to a single common adaptive learning-rate monotonically improves performance while controlling for the presence of adaptivity and momentum.
We control for the anisotropy of Adam's update, meaning that it gives each coordinate its own step size, by pooling over rows, which keeps anisotropy but restores equivariance: the resulting method closes about $30\%$ of Adam's gap to scalar-Adam.

\textbf{(3) Real data, and a gap that does not close at small step sizes.} We study the extent to which optimizers preserve the symmetry on real data and in modern architectures. First, on hyperspectral benchmarks at matched training loss we find that gradient descent and scalar-Adam agree in all sixteen seed-paired runs on the two main scenes (\S\ref{sec:realdata}), while Adam differs substantially. Adam's difference is not explained by how well each method optimizes. We measure how far Adam's fitted matrix is from gradient descent's. Two gradient-descent runs at different learning rates set a floor for that distance. On all four scenes, Adam's distance stops shrinking 2.7 to 3.8 orders of magnitude above the floor, so the gap does not close at smaller learning rates.

\textbf{(4) The same gauge inside attention and LoRA.} We then show that the same symmetry is present in modern architectures, and that Adam breaks it there too (\S\ref{sec:attention}). In transformers, head logits only depend on the query and key weights $W_Q, W_K$ through their product $W_Q^\top W_K$. While different pairs $(W_Q', W_K')$ can give rise to the same product, copies of transformers initialized identically except for such gauge-equivalent choices are indistinguishable under gradient descent or scalar-Adam. Adam separates these copies immediately, and at the end of training their $W_Q^\top W_K$ are $56\%$ apart in relative Frobenius distance. We also show that the $BA$ matrix in low-rank adaptation (LoRA) has the same symmetry, which Adam breaks at the first step while other methods preserve it.

\paragraph{Scope.} Note that our analysis is not about which optimizer should be used, but about describing which solution a given optimizer will choose. Section~\ref{sec:phase} maps where the low-rank bias helps and where it hurts. Our explanation is an empirically falsifiable hypothesis, so we also test it where it should break down.

\paragraph{Related work.}
Theory of factored models centers around the balancedness $U^\top U - V^\top V$, which gradient
flow conserves \citep{arora2018optimization,du2018algorithmic}. Much of the theory of Adam's
implicit bias is based on a margin-based viewpoint \citep{zhang2024implicit,baek2026implicit}. Recently, there is a unified steepest
descent perspective on optimizers where Adam corresponds to steepest descent in the $\ell_\infty$
geometry and Muon corresponds to steepest descent in the spectral geometry
\citep{gronich2026implicit,fan2025implicit,bernstein2024old}. We use this viewpoint only as
motivation for the norm used in this paper.

The closest work to this paper studies rotations in the ambient space of the parameters or features
rather than the gauge symmetry internal to the factorization. Specifically,
\citet{ling2022vectoradam} and \citet{xie2025linfty} study how Adam is impacted by the basis in
which parameters are represented, and \citet{depavia2025rotations} explain this effect with an
inversion of Adam's richness bias in classification problems. \citet{zhang2025rotation} show that Adam's training of transformers degrades under random rotations of the parameter space, so explaining Adam needs assumptions that depend on rotation. Finally, \citet{yen2025lorarite}
study the LoRA reparametrization as a deliberate design target. What we ask is which
interpolating solution the optimizer reaches when the symmetry is internal to the factorization.
Appendix~\ref{app:related} (Table~\ref{tab:litmap}) sorts these works by the symmetry they study.

\section{Setup: gauge, task, and protocol}
\label{sec:setup}

\paragraph{Factored model and gauge transformation.}
We consider a differentiable loss $f:\R^{n\times n}\to\R$ and the factored objective
$L(U,V) = f(UV^\top)$ with $U, V \in \R^{n\times k}$, overparameterized so that $k$ is at least
the rank of any solution of interest. The product $UV^\top$ does not change if both factors are
rotated together. For any orthogonal $Q \in \OO(k)$, define the gauge
transformation $\rho_Q(U,V) = (UQ,\, VQ)$. Since $(UQ)(VQ)^\top = UV^\top$, the loss is invariant:
$L \circ \rho_Q = L$ whatever $f$ is. Every claim below holds for non-square $UV^\top$ too, because the
gauge acts only on the shared latent dimension $k$.

\begin{definition}[Gauge equivariance of an optimizer]
\label{def:equivariance}
Consider an optimizer with internal state $s_t$ (e.g.\ momentum terms, preconditioning matrices)
and update rule
\[
(U_{t+1}, V_{t+1}, s_{t+1}) = \mathcal{A}\bigl(U_t, V_t, s_t;\, \nabla L(U_t,V_t)\bigr).
\]
It is \emph{gauge-equivariant} if for every $Q \in \OO(k)$ there is a corresponding transformation
$\sigma_Q$ on the state such that initializing from $\bigl(\rho_Q(U_0,V_0),\, \sigma_Q(s_0)\bigr)$
leads to states $\bigl(\rho_Q(U_t,V_t),\, \sigma_Q(s_t)\bigr)$ at every step $t$. The state is
taken only up to components that never reach the update.
\end{definition}

All optimizers we consider are initialized from a zero state; we adopt the convention that $\sigma_Q(0) = 0$.

For Adam, however, no such map exists for a general rotation $Q$. The problem lies in squaring the gradient, an operation that does not commute with rotations, that is, $Q(x \odot x) \ne (Qx) \odot (Qx)$. Adam is still equivariant under the special rotations that only reorder coordinates and flip their signs (signed permutations): reordering commutes with elementwise squaring, and squaring removes the signs. Our scalar version of Adam, scalar-Adam, fixes this problem by replacing coordinate-wise squared terms with a single, gauge-invariant scalar. So for every equivariant optimizer the product $W_t = U_tV_t^\top$ is the same from every gauge-equivalent start, at every step $t$; for Adam it is not.

\paragraph{Matrix sensing.}
To evaluate the impact of different optimizers' symmetries on their implicit biases, we follow the setup of \citet{gunasekar2017implicit} and consider low-rank matrix sensing with Gaussian measurements as studied by \citet{recht2010guaranteed}: we observe $m$ Gaussian linear measurements $y_i = \ip{A_i}{X^\ast}$ and minimize $f(W) = \mathrm{mean}_i(\ip{A_i}{W}-y_i)^2$. We seek to recover a planted, rank-$3$ matrix $X^\ast$. The dimension of the matrix is $n = 40$ and the number of measurements is twice the degrees of freedom in the solution. This puts us well within the interpolation regime: there are infinitely many matrices that fit the data, including ones of rank well above $3$. We report the relative Frobenius error with respect to the planted matrix, $\norm{UV^\top - X^\ast}_F/\norm{X^\ast}_F$, which we call \emph{recovery} (smaller is better), as well as the effective rank of the learned matrix, a measure of the spread of the singular values. Although we report the latter, our claims are based mainly on recovery with respect to $X^\ast$, as it is possible for a matrix with large distance to $X^\ast$ to have low effective rank. For all optimizers, we run them until they interpolate, meaning training loss below $10^{-7}$, unless they fail to do so. All experiments use no explicit regularization. Details, the protocol and the controls are in Appendices~\ref{app:details}, \ref{app:protocol} and~\ref{app:controls}.

\section{Theory: equivariance as a classifier of update rules}
\label{sec:theory}

Given a loss function $L(U,V)$, consider a reparameterization $(U,V) \to (UQ,VQ)$ for some
$Q\in\OO(k)$. The gradients of the loss w.r.t.\ each factor transform covariantly, meaning they rotate with the factors:
\[
\nabla_U L(UQ, VQ) = \bigl(\nabla_U L(U,V)\bigr)Q,
\qquad
\nabla_V L(UQ, VQ) = \bigl(\nabla_V L(U,V)\bigr)Q,
\]
as stated in Lemma~\ref{lem:grad-cov} of Appendix~\ref{app:proofs}, where we give all proofs. It is
this which determines which update rules preserve the intrinsic geometry of the factorization, and
which of them read an arbitrary coordinate system instead. We call optimization rules that respect
this transformation \emph{equivariant}; see Definition~\ref{def:equivariance}. We show in
Proposition~\ref{prop:equivariant} that four families qualify: gradient descent, with or without Polyak or
Nesterov momentum (more generally, any update that is a fixed weighted sum of past gradients);
scalar-Adam; Muon; and Shampoo, even with
positive damping. For scalar-Adam, this is because it replaces the per-coordinate second moment with one
bias-corrected scalar, built from the mean squared gradient entry, which does not change when the
factors are rotated. For Muon, the update, given by
$\Delta = \msign(M_t)$, is the orthogonal factor of the momentum (its SVD with every nonzero
singular value set to $1$), and it commutes with the gauge exactly, including for the finite
Newton--Schulz iteration used to compute it. Next we show that Adam, RMSProp,
signSGD, Lion and Adafactor are not equivariant. These optimizers all have some amount of
nonlinearity applied coordinate-wise at some point in the calculation. For the first four it happens
during the first optimizer update, meaning that the failure of gauge-equivariance shows up in the
represented matrix at the very first step. Let $W_1$ be the product $U_1V_1^\top$ after one step, and
$\widetilde W_1$ the same product for a run started from a rotated copy of the initialization.
Proposition~\ref{prop:adam-onestep} gives the gap $\widetilde W_1 - W_1$ exactly. For Adafactor the nonlinearity happens in the computation of the
factored statistics. The next question is whether this pattern is the whole story for memoryless
rules. Theorem~\ref{thm:structure} answers it
exactly.

\begin{proposition}[Nonlinear coordinate-wise maps break the gauge]
\label{prop:coordwise}
Let $k \ge 2$ and let a memoryless update have the entrywise form $\Phi(G)_{ij} = \phi(G_{ij})$ for a
fixed $\phi:\R\to\R$. If $\Phi(GQ) = \Phi(G)Q$ for every $G$ and every $Q\in\OO(k)$, then $\phi$ is
linear. No rule applying a fixed nonlinearity coordinate-wise at zero optimizer state is
gauge-equivariant.
\end{proposition}

\begin{theorem}[Structure: equivariant memoryless rules are Gram-determined left preconditioners; \rev{after \citealp{villar2021scalars}}]
\label{thm:structure}
Let $\Phi$, the map from a gradient $G \in \R^{n\times k}$ of any rank to an update, be memoryless. Then $\Phi(GQ) = \Phi(G)Q$ for all $Q \in \OO(k)$ if and only if
$\Phi(G) = H\!\left(GG^\top\right) G$ for some matrix-valued $H$ depending only on the Gram matrix
$GG^\top$, a quantity the gauge leaves invariant. The proof gives $H$ canonically, at every rank.
\end{theorem}

\rev{This is Proposition~4 of \citet{villar2021scalars}, applied to each row of the update. We include a proof in matrix notation in Appendix~\ref{app:proofs}.} Proposition~\ref{prop:equivariant} lists examples, and Theorem~\ref{thm:structure} says every one has the same form: any update that is equivariant and depends only on the current gradient is $H(GG^\top)G$ for some $H$. Rules that keep additional state, such as momentum, require separate treatment; for instance, such an update can be nonzero when the gradient happens to be zero. We discuss how to characterize these rules in Remark~\ref{rem:stateful}.

The remaining freedom sits in $H$, which may be \emph{any} matrix function of $GG^\top$, and even
a fixed nonscalar left multiplier $\Phi(G) = DG$ is equivariant. When $H$ applies one scalar function
to each eigenvalue of $GG^\top$, however, the whole rule collapses to one \emph{spectral transfer
function} $h(\sigma)$, the singular value the update assigns to a gradient singular value $\sigma$
(Proposition~\ref{prop:spectral}). Gradient descent has
$h(\sigma)=\eta\sigma$ and the polar map has $h\equiv\eta$. \rev{We call $h$ the rule's \emph{spectral
schedule}: it says how fast the rule grows small singular values relative to large ones.} GD and the
polar map are its two ends, and \S\ref{sec:phase} maps them.

\paragraph{Optimizers defined by a norm.} Several recent papers describe deployed
optimizers as steepest descent under a norm: Adam and the sign rules under $\ell_\infty$, Muon
under the spectral norm
\citep{xie2024implicit,xie2025linfty,bernstein2024old,fan2025implicit,gronich2026implicit}. For
such an optimizer, our test becomes a check on its norm (Corollary~\ref{cor:norm}). The check only
goes one way. Every norm that passes gives an equivariant optimizer, under the conditions stated
there. But the structure theorem allows any $H$, with no positivity condition, so some equivariant
rules do not come from any norm.

\begin{corollary}[Norm form of the criterion]
\label{cor:norm}
Call a norm \emph{right-invariant} if rotating a matrix from the right leaves it unchanged,
$\norm{MQ} = \norm{M}$ for every $Q\in\OO(k)$. Equivalently, the norm depends on $M$ only through
$MM^\top$.\par
(i) Steepest descent under a right-invariant norm is gauge-equivariant. If the steepest-descent
direction is not unique, this needs ties broken in an equivariant way.\par
(ii) If the norm also depends on optimizer state, such as momentum, (i) still holds provided that
state rotates with the factors (Definition~\ref{def:equivariance}).\par
Appendix~\ref{app:normform} states both parts formally.
\end{corollary}

Each entry of Proposition~\ref{prop:equivariant} then becomes a one-line check on a norm. The
spectral and Frobenius norms are functions of singular values alone, so they are right-invariant,
which places Muon and gradient descent inside the class, while entrywise $\ell_p$ is not
right-invariant for any $p\neq2$, which is Proposition~\ref{prop:coordwise} from the norm side.

\begin{theorem}[Transfer theorem, informal; Appendix~\ref{app:proofs}]
\label{thm:transfer}
Let the scalar-preconditioned flow $\dot\theta = -\nabla L(\theta)/a(t)$ have a clock $a(t)>0$ that
is finite almost everywhere\rev{, with} $1/a$ locally integrable. Then
the time change $\zeta(t) = \int_0^t du/a(u)$ turns any solution of it into a gradient-flow solution:
the two follow the \emph{same path} in parameter space, the whole of it and any limit point if
$\int_0^\infty du/a(u) = \infty$. Under locally Lipschitz $\nabla L$ that solution is unique, so it is
\emph{the} gradient flow from that initialization. \rev{If $a$ is determined solely by gauge-invariant
statistics of the trajectory, the time change is also the same from every rotated start.}
\end{theorem}

Any gradient-flow conclusion that rests only on the path or the limit point therefore transfers
\citep{gunasekar2017implicit,arora2019implicit,li2021towards}. Time-dependent conclusions such as
convergence rates do not. Lemma~\ref{lem:momentum} extends the transfer to momentum in the small-step limit. It does not cover scalar-Adam's second-moment recursion, which we
check numerically in \S\ref{sec:flowlimit}.

\begin{remark}[Scope: what equivariance classifies, and the axis left over]
\label{rem:transferscope}
\label{rem:necessary}
Equivariance says which basis changes a rule cannot see, and that is already testable: Adam is
blind to a signed permutation but not to a rotation of equal size
(Appendix~\ref{app:gaugeorbit}).

Equivariance is also what makes the transfer of Theorem~\ref{thm:transfer} uniform over the gauge
orbit, so that every gauge-equivalent start inherits the same conclusion. But it does not put a
method on gradient flow's path, and it does not by itself decide which interpolant a method reaches.

Inside the class of Theorem~\ref{thm:structure}, with $H$ a spectral function, the only remaining
freedom is the spectral schedule $h(\sigma)$ (Proposition~\ref{prop:spectral}). The pair
(equivariant, $h$) is then a complete description, with two axes. The second axis, $h$, sets how
greedy the sequential growth of singular values from small initialization is
\citep{li2021towards}, and \S\ref{sec:phase} maps it.

Both axes matter. The factor-aware preconditioner of Appendix~\ref{app:scaledgd}, our own variant in the style of ScaledGD \citep{tong2021accelerating}, preconditions each factor with the other factor's Gram matrix. It is equivariant, but it flattens the spectrum, a poor schedule, and it reaches recovery error $0.61$, worse than Adam. In the other direction, Appendix~\ref{app:c10}'s long-anneal signum carries momentum state, so it falls outside the memoryless class, and it reaches recovery without the symmetry.
\end{remark}

\section{Sorting nine optimizers, and isolating the cause}
\label{sec:zoo}

We run nine optimizers, three seeds each, shared across methods, on the sensing task from \S\ref{sec:setup}: gradient descent (GD), scalar-Adam, Muon \citep{jordan2024muon}, Shampoo \citep{gupta2018shampoo}, Adam \citep{kingma2015adam}, RMSProp \citep{tieleman2012rmsprop}, signum (that is, signSGD with momentum, \citealp{bernstein2018signsgd}), Lion \citep{chen2023lion}, Adafactor \citep{shazeer2018adafactor}. Of these, we predict GD, scalar-Adam, Muon and Shampoo are equivariant, and predict that Adam, RMSProp, signum, Lion and Adafactor break the gauge.

For each optimizer, we report the rate from its own grid that gives the best recovery among the rates that reach interpolation. \rev{Each of the four equivariant optimizers} gets error below $0.3$ for some learning rate, while none of the coordinate-wise optimizers can at any learning rate we tested. This shows every method at its best, and the two predicted groups still separate clearly. Adam, for example, stays between $0.56$ and $0.58$ at all five interpolating rates of Table~\ref{tab:lrcurves}, which span two orders of magnitude.

Using the fastest interpolating learning rate instead of the best recovery learning rate still sorts $8/9$ optimizers into the predicted groups; the exception is Muon, whose fastest-interpolating rate recovers poorly.

We also experimented with applying a shared cosine decay LR schedule to all nine methods, which still resulted in the same $9/9$ equivariant/coordinate-wise separation, ruling out a shared schedule as the confound (Table~\ref{tab:cosine9}).

We find that up to $n{=}256$ over ten seeds (Appendix~\ref{app:c9}) a restricted version of our observed ordering holds. With the exceptions of signum and fixed-damping Shampoo, every run of GD/scalar-Adam/Muon has lower recovery error than all runs of Adam/RMSProp/Lion/Adafactor at four of the five problem sizes. At $n{=}40$, every equivariant method outperforms all coordinate-wise methods with no seed that crosses the boundary (Figure~\ref{fig:zoo}). Muon also outperforms our convex nuclear-norm reference solve in additional controlled experiments. We summarize results for all nine methods in Table~\ref{tab:zoo}.

\paragraph{Isolating the cause: Adam-$p$.}\label{sec:dial}
We test whether Adam's poor performance is caused by its anisotropy. To test, we interpolate between Adam and scalar-Adam, adjusting only the preconditioner anisotropy while keeping adaptivity, step-size and momentum identical. We call this new optimizer Adam-$p$ and give its update rule in Algorithm~\ref{alg:dial}. Adam-$p$ interpolates the denominator of Adam's update, replacing $\sqrt{\hat v}+\epsilon$ with $(\sqrt{\hat v}+\epsilon)^p(\bar{s}+\epsilon)^{1-p}$, where $\bar{s} \triangleq \sqrt{\operatorname{mean}(\hat v)}$ is a scalar and $\epsilon=10^{-8}$ for stability. Since $\bar s$ averages $\hat v$ over every entry of both factors, it sees the gradients only through their Frobenius norms and is gauge-invariant (Lemma~\ref{lem:grad-cov}). The denominator reduces to the invariant scalar from Proposition~\ref{prop:equivariant}(2) when $p=0$, so Adam-$p$ at $p=0$ matches scalar-Adam up to $O(\epsilon)$. Adam-$p$ matches Adam when $p=1$.

Effective rank (Figure~\ref{fig:dial}) and recovery both improve monotonically as $p$ moves from Adam ($p=1$) towards scalar-Adam ($p=0$), holding all other hyper-parameters fixed. With a fixed learning rate of $10^{-3}$, recovery improves from $0.576$ for Adam ($p=1$) to $0.201$ for scalar-Adam ($p=0$).  With the best-performing learning rate for each value of $p$, we observe the same improvement. Monotonicity holds at four of the five fixed rates, and the amount of repair varies. So the gain does not come from retuning the step size.

\begin{figure}[t]
\centering
\includegraphics[width=0.90\linewidth]{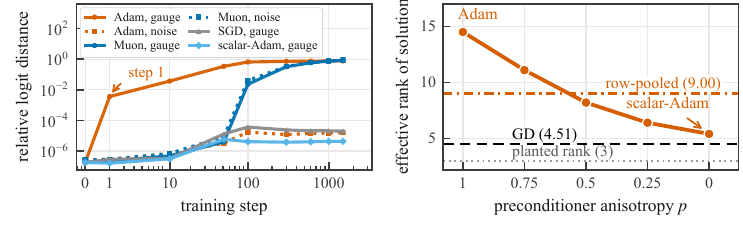}
\caption{\textbf{Left: Adam is basis-dependent inside an attention head.} Gauge twins start as the
same function up to a per-head rotation. Noise twins differ by a $10^{-7}$ perturbation. Adam
splits at step $1$, four orders of magnitude above its noise twin. The equivariant methods sit at that floor,
and Muon separates later, at the scale of its own noise twin.
\textbf{Right: Adam-$p$.} Effective rank falls to scalar-Adam's value as the preconditioner
anisotropy $p\to0$, and recovery improves with it, $0.570 \to 0.201$. Solid is the best rate for
each $p$ and dash-dot the row-pooled rule.}
\label{fig:panels}%
\panellabel{fig:attention}{left}\panellabel{fig:dial}{right}
\end{figure}

\paragraph{Pooling ablation.}
We study an ablation in which we pool the denominator of Adam only along the rows of a factor, thus maintaining Adam's anisotropy across rows while recovering gauge equivariance. On the benchmark testbed and with the same grid and seeds as before, this achieves a recovery error of $0.461$, intermediate between Adam's $0.570$ and scalar-Adam's $0.201$. Thus about $30\%$ of Adam's recovery gap is attributable to breaking the gauge symmetry, as discussed in Appendix~\ref{app:rowpool}, while the rest is due to its anisotropy. The row-pooled rule also closes about $60\%$ of the effective-rank gap, cutting Adam's $14.5$ to $9.0$ against scalar-Adam's $5.4$. The anisotropy is also what makes Adam fast: scalar-Adam needs about eight times as many steps to reach interpolation. \rev{This matches the known mechanism of the low-rank bias. From a small start, GD grows each singular direction at a rate proportional to its size, so it fits the largest directions first \citep{li2021towards}. Adam divides each coordinate by its own gradient scale. That one operation gives weak coordinates large steps, which weakens the ordering, and it also breaks the gauge. Adam-$p$ and the row-pooled rule undo the two effects separately.}

\section{\rev{The gauge in transformers and LoRA}}
\label{sec:attention}

\paragraph{Attention has the same symmetry.} An attention head computes logits through
$q_i^\top \kappa_j = x_i^\top W_Q^\top W_K x_j$ \citep{vaswani2017attention}, up to the usual
$1/\sqrt{d_h}$ scaling, with $d_h$ the head dimension. For any orthogonal $A_h$, replacing $(W_Q, W_K)$ with
$(A_h W_Q, A_h W_K)$ leaves every logit unchanged. This is the gauge of \S\ref{sec:theory} with
$W_Q^\top W_K$ in the role of $UV^\top$, under the same $\OO(k)$ convention
\S\ref{sec:setup} fixes. The covariance arguments carry over to this head-level action.

\paragraph{Testing it with twins.} To test whether an optimizer respects this symmetry, we train \emph{twins}: two copies of the same network that start from gauge-equivalent weights, so they compute exactly the same function at step $0$. Both copies then see the same data in the same order with the same optimizer. If the optimizer is equivariant, the copies keep computing the same function at every step. If it is not, they drift apart. We use four kinds of twin:
\begin{itemize}\setlength{\itemsep}{0pt}\setlength{\parsep}{0pt}\setlength{\topsep}{2pt}
\item \emph{gauge twin}: each head's $W_Q$ and $W_K$ are rotated by the same random orthogonal matrix. This is the test.
\item \emph{identity twin}: the rotation is the identity, so the copies should stay identical. They do, up to numerical precision, which checks the code.
\item \emph{noise twin}: no rotation, but a $10^{-7}$ perturbation is added to the original weights. It shows how far rounding error alone moves the copies apart, so it is the floor.
\item \emph{different-init twin}: an independent random initialization. It shows how far apart two unrelated runs end up.
\end{itemize}
An optimizer breaks the symmetry when its gauge twins drift apart more than its noise twins.

We run this on modular addition \citep{power2022grokking}, with a 2-layer, 4-head transformer trained on addition mod 47 at representation size 64 (Figure~\ref{fig:attention}, Table~\ref{tab:attention}, Appendix~\ref{app:c5}). We remove biases and QK norm so that the symmetry is exact.

Adam's gauge twins split at step $1$. At that point their logits already differ four orders of magnitude more than the noise twins'. By the end of training, the gauge-invariant product $W_Q^\top W_K$ is $56\%$ apart in relative Frobenius distance, so Adam takes the query/key product somewhere else and does not just rotate the heads. For SGD and scalar-Adam, the gauge twins stay at the noise floor. Muon's gauge twins do drift apart later in training, but its noise twins drift by the same amount, so this is numerical instability (likely in computing $\msign$), not a failure of equivariance. This holds for every seed and depth we ran. Under mini-batch noise, Adam's step-$1$ split still exceeds a $100\times$ stronger noise twin by $39\times$. Equivariance must hold at every step, so one broken step is enough to show it is broken, whatever happens later in training.

\emph{We propose this four-twin protocol as the diagnostic for any claim of equivariance during training}. When the equivariant methods are trained to convergence (at 6 layers and $d{=}256$), their gauge twins separate only as much as their noise twins.

\rev{\paragraph{The rank gap carries over.} The low-rank gap of \S\ref{sec:zoo} also appears here, once rank is measured on what training adds. Heads start at standard scale, so each $W_Q^\top W_K$ is already full rank at step $0$, and we take the effective rank of its change from initialization. Averaged over heads and three seeds, it is $12.3$ for SGD, $14.2$ for scalar-Adam and $17.1$ for Adam. Scalar-Adam closes about $60\%$ of Adam's gap to SGD. So scalar-Adam lands between Adam and SGD here, as it lands between Adam and gradient descent in matrix sensing.}

\paragraph{Merging twins.}
Merging means averaging the weights of two models. Its cost is the merged model's held-out loss minus the twins' mean held-out loss. Gauge twins compute the same function at the start, so if the optimizer kept them equivalent they could be merged after training. Even after Procrustes alignment, which rotates one twin to best match the other, merging Adam's gauge twins costs a $+0.052$ nats penalty on character-level text. For comparison, merging twins of the equivariant methods costs $-0.0004$ to $0$ nats, and merging Adam's noise twins costs $0$. The cost is specific to twins that split from scratch. Copies that share a short common prefix of training, or that are fine-tuned from the public GPT-2-small weights, merge at no cost.

\label{sec:lora}
\paragraph{LoRA has the same symmetry.} Low-rank adaptation (LoRA) trains only a low-rank update $\Delta W = BA$ on top of a frozen base model, with $B\in\R^{d\times r}$ and $A\in\R^{r\times d'}$ \citep{hu2022lora}. For any orthogonal $Q\in\OO(r)$, the adapters $(BQ, Q^\top A)$ give the same $\Delta W$, so LoRA has the same gauge. LoRA initializes $B=0$, so gauge twins start with the same $\Delta W=0$ and compute the same function. On both modular addition and character-level language modeling, Adam's twins differ from the first update onwards, while SGD, scalar-Adam and Muon stay within $10^{-16}$ to $10^{-15}$ of each other at step~$1$. Adam's gauge-to-noise ratios are $10^{3}$ to $10^5$, depending on the perturbation scale. The twins' held-out losses stay close ($\pm 0.046$ and $8\times 10^{-4}$ nats on modular addition and text, respectively), so the split is not a difference in quality of fit. Merging adapters is cheap, because the average of the products, $\frac{1}{2}(B_1 A_1 + B_2 A_2)$, does not depend on the gauge: text twins whose products $BA$ end $0.28$ apart merge with a penalty below $10^{-5}$ nats. Details are in Appendix~\ref{app:mergedetails} and Table~\ref{tab:lora}.

\section{Inside the equivariant class, and on real data}
\label{sec:phase}

\paragraph{The spectral-tail phase diagram.}\label{sec:phasediag}
Equivariance is a yes/no property. Behavior within the equivariant class varies continuously, and Proposition~\ref{prop:spectral} shows that the axis along which it varies is the spectral schedule $h(\sigma)$. For example, a single step of GD has $h(\sigma)=\eta\sigma$: it scales each singular component by a multiple of its singular value. In contrast, the polar mapping in Muon applies a uniform scale $h\equiv\eta$ to all nonzero singular components; its effect on the predictor depends on the dynamics of the factorization.

Some recent work provides complementary insight into these dynamics. For example, \citet{kang2026uniform} show that in LoRA-like factorizations under Muon all modes grow at the same rate, an update well-suited to targets that really are low-rank. Meanwhile, \citet{dragutinovic2026muon} show that Muon can give up GD's sequential simplicity bias in a way that, depending on the data, can make it more prone to fitting spurious features. Our sweep recovers both of these effects, with the spectral tail of the target as the axis between them.

We plant $X = \sqrt{1-\tau^2}\, X_3 + \tau E_\perp$, a rank-$3$ signal plus a dense tail
that shares no row or column directions with it; $\tau$ is the tail's amplitude and $\tau^2$ its share of
the energy. Sweeping $\tau$ over $10$ seeds shows both axes on one plot (Figure~\ref{fig:phase}): Adam
against the equivariant methods, and Muon against GD.

The crossing point of Muon and GD, $\tau^\ast \simeq 0.2$ ($95\%$ interval $[0.15, 0.25]$ over ten seeds), corresponds to about $4\%$ tail energy. On the symmetry axis, Adam is worst at every plotted $\tau$. Along the schedule axis, Muon and GD are the two ends of the spectral schedule; Muon, the more aggressive of the two (it grows small singular values faster relative to large ones), recovers the clean low-rank signal better but does worse on a tail. Damped Shampoo performs worse than GD for all $\tau$ except the largest, where the difference is within noise.

Overall, we can say that uniform growth is not optimal when the target has a heavy tail. A solvable two-mode model (Proposition~\ref{prop:boundary}) illustrates the contrast: by the time the large mode is fit, equal-rate growth has fit the small mode in full, while GD's has barely moved.

\begin{figure}[tp]
\centering
\includegraphics[width=\linewidth]{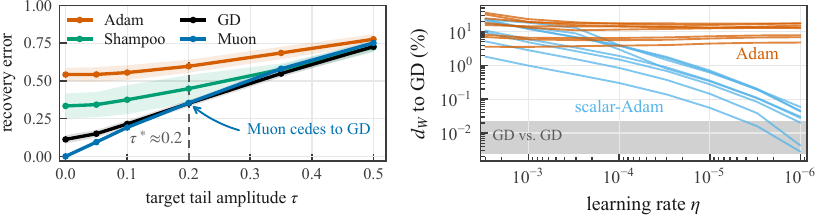}
\caption{\textbf{Left: recovery as the target stops being exactly low rank.} Moving right puts
more of the target outside its top three singular values. Adam is worst everywhere, and the
crossover inside the equivariant class is annotated. Ten seeds, bands $\pm1$ s.d.
\textbf{Right: distance of each method's fitted matrix from gradient descent's}, both stopped at
the same training loss. Each line is one scene at one sampling density, averaged over four seeds.
The step size \emph{shrinks} to the right. The grey band is GD against itself at two step sizes, so
a line inside it has reached GD's answer.
Digits: Tables~\ref{tab:phase} and~\ref{tab:pdistfour}.}
\panellabel{fig:phase}{left}\panellabel{fig:divergence}{right}
\end{figure}

\paragraph{Comparison at matched training loss.}\label{sec:realdata}
Since real data are only approximately low-rank, we use hyperspectral images reported in \citet{grana2011hsiscenes} and \citet{baumgardner2015indianpines}: we report results for Indian Pines and Pavia University, and perform similar experiments on Salinas and Botswana to check our conclusions on all four scenes, as detailed in Appendix~\ref{app:tailtest}. We perform rank-$48$ factored matrix completion on each, without weight decay, at two sampling densities per scene. Each scene is a matrix with one row per pixel and one column per spectral band. We observe a random subset of its entries and predict the rest.

We set how many entries to observe from the degrees of freedom of a rank-$24$ matrix. We read rank $24$ off the singular-value spectrum and did not tune it on validation data. Every method gets the same density. We compare methods at matched training loss, since otherwise a method that fits the observed entries more slowly can look deceptively better.

\paragraph{Extending the learning-rate comparison.}
\label{sec:flowlimit}
We expand our choice of learning rates three orders of magnitude lower, $\{3\times10^{-3},\dots,10^{-6}\}$, and add scalar-Adam to the comparison. For the smallest stepsize, scalar-Adam and GD agree on both statistics (held-out error and effective rank) in all sixteen seed-paired runs over four cells (two scenes at two densities), and the gap shrinks at the $O(\eta)$ rate expected of a discretization error. Additional learning rate results are in Table~\ref{tab:flowlimit}.

\paragraph{Distance between fitted matrices.}
Finally, we show that Adam converges to a weight matrix different from GD's, and not just a different held-out error. Define the relative matrix distance
\[
d_W = \norm{W - W_{\mathrm{GD}}}_F/\norm{W_{\mathrm{GD}}}_F,
\]
evaluated at the matched-loss checkpoints, paired on seed. As a baseline, compare GD's matrix to itself at the same losses using two rates from GD's own grid. Figure~\ref{fig:divergence} shows $d_W$ for Adam and scalar-Adam on all four scenes and both densities, against the learning rate $\eta$. Scalar-Adam reaches the calibration floor (grey band) on Pavia and stays within a factor of $15$ of it on the other three scenes. Table~\ref{tab:pdistfour} gives the numbers; Adam's $d_W$ is at least $293\times$ scalar-Adam's $d_W$ in all eight cells. Scalar-Adam's $d_W$ monotonically decreases over three orders of magnitude; Adam's $d_W$ plateaus $2.7$ to $3.8$ orders of magnitude above the calibration floor in all cells, and rises again in six.

We can argue against slow convergence by decreasing the learning rate by one more factor of ten, to $10^{-7}$, and seeing little change (Table~\ref{tab:limitrank}): Adam's coordinate-wise scaling of its effective step size based on second moments of past gradients persists even when its overall step size is reduced, changing its fitted matrix and not just its rate of convergence. This cannot come from rate selection, since the plateau spans the whole range of rates.

\paragraph{Held-out error.}
Table~\ref{tab:realdata} reports the held-out errors on Indian Pines and Pavia using train-only selected rates; GD beats Adam on both held-out error and effective rank in all cells, for every seed, at loss $10^{-5}$. Under best-rate selection the held-out ordering varies by cell and scene, while $d_W$ and small-step effective rank hold on every scene.

\section{Discussion, limitations and conclusion}
\label{sec:discussion}

Our claim concerns which interpolant is selected. \rev{For Adam, that choice depends on a basis the loss cannot see.} Tuned performance is a separate question, and there a properly tuned Adam is hard to beat at equal budgets \citep{schmidt2021descending,dahl2023algoperf}.

Our main-text experiments range from $40 \times 40$ to $2000 \times 200$. Sensing experiments were done up to $n = 256$ and averaged over ten seeds. The splitting of the attention twins (\S\ref{sec:attention}) was validated at GPT-2-small width. The gauge prediction holds without exception. In mean recovery, the seven methods other than Shampoo and signum separate into the predicted groups at every size we tested, and all nine separate at $n = 40$ (Appendices~\ref{app:c9} and~\ref{app:c10}).

Our results give a test: given an optimizer, verify whether the norm that generates the step is invariant to the action of the group. This tells us whether gauge-equivalent models stay gauge-equivalent under optimization, but to know which interpolant is selected, we also need the schedule. \rev{Weight decay fits this picture. The penalty $\lambda(\norm{U}_F^2+\norm{V}_F^2)$ does not change under the rotation, and decoupled decay commutes with it, so it moves no optimizer across the test. It adds an explicit pull toward low rank \citep{galanti2022sgd,kobayashi2024weight}, which our experiments leave out to isolate the implicit one. A natural next step is a theorem that reads the rank of the selected solution off the schedule $h$.}

\paragraph{In practice.} Two practical readings follow. To keep adaptive step sizes and still reach gradient descent's solution in the small-step limit, share one second-moment scalar across all coordinates, as scalar-Adam does (Theorem~\ref{thm:transfer}, \S\ref{sec:flowlimit}). Within the equivariant class, Muon recovers better than GD on exactly low-rank targets, and in our sweep GD does better once more than about $4\%$ of the target's energy lies outside its top three singular directions (Figure~\ref{fig:phase}).

\subsubsection*{Acknowledgments}
I thank Tarun Sheel for the collaboration on FlowAdam that motivated this line of work, and for his
support. This research was enabled in part by computing resources provided by ACENET and the Digital
Research Alliance of Canada (the Nibi cluster).

\section*{Reproducibility statement}
Complete proofs of all numbered results are in Appendix~\ref{app:proofs}, and their full statements are in Appendix~\ref{app:structure}. Appendix~\ref{app:details} gives the experimental setup, including the numerical precision and hardware of every run. It also describes how the hyperspectral data are preprocessed. Appendix~\ref{app:controls} gives the control battery. Code and logs are at \url{https://github.com/idevender/loss-basis-adam}. We use a fixed, per-run random seed.

\section*{Use of Large Language Models}
Large Language Models (LLMs) were used to implement and refactor experiment code, expand intermediate algebraic steps based on the authors' outlines, and polish prose for clarity and readability. LLMs were not used to generate data, experimental results, or references. The authors led the research, and the theoretical results and experimental design are their own. The authors ran and tested all code, verified every reported number against the experiment logs, and take full responsibility for the entirety of this work.

{\small

}

\clearpage
\clearpage
\renewcommand{\contentsname}{\centering Appendix}
\setcounter{tocdepth}{2}
\begingroup
\makeatletter
\renewcommand*\l@section[2]{%
  \ifnum \c@tocdepth >\z@
    \addpenalty\@secpenalty
    \addvspace{0.2em \@plus\p@}%
    \setlength\@tempdima{1.5em}%
    \begingroup
      \parindent \z@ \rightskip \@pnumwidth
      \parfillskip -\@pnumwidth
      \leavevmode \bfseries
      \advance\leftskip\@tempdima
      \hskip -\leftskip
      #1\nobreak\hfil \nobreak\hb@xt@\@pnumwidth{\hss #2}\par
    \endgroup
  \fi}
\makeatother
\setlength{\parskip}{0pt}
\linespread{0.95}\footnotesize\selectfont
\tableofcontents
\endgroup

\clearpage
\appendix
\setlength{\parskip}{5pt plus 2pt}
\addtocontents{toc}{\protect\setcounter{tocdepth}{2}}

\section{The numbered results in full}
\label{app:structure}

This appendix states in full every numbered result that is not stated in the body. Lemmas~\ref{lem:max-isometric} and~\ref{lem:grad-cov} and the formal statement of Theorem~\ref{thm:transfer} are stated with their proofs in Appendix~\ref{app:proofs}.

Throughout, $G$ is the gradient of one factor, an $n\times k$ matrix. The gauge acts on it from the
right, $G\mapsto GQ$ with $Q\in\OO(k)$ (Lemma~\ref{lem:grad-cov}). For a rule $\Phi$ that uses only
the current gradient, equivariance means $\Phi(GQ)=\Phi(G)Q$: rotating the gradient rotates the
update in the same way.

\subsection{The equivariant class}

Each statement is about the full-matrix
rule in exact (real) arithmetic.

\begin{proposition}[Equivariant class]
\label{prop:equivariant}
The following update rules satisfy gauge equivariance as defined in
Definition~\ref{def:equivariance}.
\begin{enumerate}\setlength{\itemsep}{1pt}\setlength{\parsep}{0pt}\setlength{\parskip}{0pt}
\item \textbf{Gradient descent and Polyak/Nesterov momentum}: any fixed linear combination of past
gradients. For Nesterov momentum, the gradient is taken
at the look-ahead point, which transforms in the same way as the iterates.

\item \textbf{Scalar-preconditioned Adam}: Adam with the per-coordinate second-moment estimate
replaced by a single scalar $\hat\nu_t$, a bias-corrected exponential moving average (EMA) of
gauge-invariant quantities, such as the average of all entries of $g^{\odot2}$ across both factors.
The update is $\Delta \propto \hat m_t/a_t$, where $\hat m_t$ is the bias-corrected first moment and
the denominator is $a_t = \sqrt{\hat\nu_t} + \epsilon$.

\item \textbf{Muon}: $\Delta = \msign(M_t)$, where $M_t$ is the momentum buffer and
\[
\msign(M) = AB^\top = \sum_{\sigma_i>0} a_ib_i^\top
\]
is the support-restricted polar factor, with $M = A\Sigma B^\top$ a compact SVD. Equivariance holds exactly
under real arithmetic, and it remains valid for any finite-order Newton--Schulz approximation.

\item \textbf{Shampoo}, with damping $\lambda>0$:
\[
\Delta = (L_t + \lambda I)^{-1/4}\, G_t \,(R_t + \lambda I)^{-1/4},
\qquad
L_t = \sum_{u\le t} G_uG_u^\top,
\quad
R_t = \sum_{u\le t} G_u^\top G_u .
\]
\end{enumerate}
\end{proposition}

\paragraph{Notes on the rules.}
\begin{itemize}\setlength{\itemsep}{1pt}\setlength{\parsep}{0pt}\setlength{\parskip}{0pt}
\item \emph{Muon.} Because the sum runs only over $\sigma_i>0$, $\msign(0)=0$ and $\msign(M)$ is the same for every compact SVD of $M$. The covariance argument needs this.
\item \emph{Shampoo.} The letters $L$ and $R$ in $L_t$ and $R_t$
name Shampoo's left and right accumulators, not the objective $L$.
\end{itemize}

\begin{remark}[Stateful rules and right preconditioners]
\label{rem:stateful}
For a stateful rule, each piece of optimizer state transforms in one of three ways under the gauge.
\begin{itemize}\setlength{\itemsep}{1pt}\setlength{\parsep}{0pt}\setlength{\parskip}{0pt}
\item \emph{Covariant} (picks up $Q$ on the right, like the gradient): momenta, $M_t \mapsto M_tQ$.
\item \emph{Invariant} (does not change): Frobenius norms and left-side accumulators
$L_t = \sum_u G_uG_u^\top$.
\item \emph{Conjugated} ($Q^\top$ on the left and $Q$ on the right): right-side quantities,
$R_t \mapsto Q^\top R_tQ$. This includes both $R_t = \sum_u G_u^\top G_u$ and the factor Gram
matrices $U^\top U$ and $V^\top V$.
\end{itemize}
Any update built as
\[
\Delta = H(\mathcal I_t)\, Z_t\, K(R_t),
\]
where $Z_t$ is a covariant quantity, like $G_t$ or $M_t$, $\mathcal I_t$ is gauge-invariant, and
$K$ is a spectral function, then transforms like the gradient, $\Delta\mapsto\Delta Q$, so it is
equivariant. Indeed,
$H(\mathcal I_t)\,(Z_tQ)\,(Q^\top K(R_t)Q) = H(\mathcal I_t)\,Z_t\,K(R_t)\,Q$. This form covers
momentum-based gradient descent ($Z_t = M_t$), scalar-Adam, Muon ($Z_t = M_t$), Shampoo, and ScaledGD.
This form is sufficient for equivariance, not necessary.
\end{remark}

\subsection{The exact first-step Adam defect}

\begin{proposition}[Adam has an exact first-step gauge defect]
\label{prop:adam-onestep}
Fix $\epsilon > 0$. Consider bias-corrected Adam with elementwise operations, a shared step size
$\eta$, initial moment estimates set to zero, and the update convention
\[
\theta^+ = \theta - \eta\, \hat m \oslash \bigl(\sqrt{\hat v} + \epsilon\bigr).
\]
Let $G_U = \nabla_U L(U_0,V_0)$ and $G_V = \nabla_V L(U_0,V_0)$ be the gradients at
initialization, and let $(\widetilde U_1, \widetilde V_1)$ be the first iterate from the
gauge-rotated start $(U_0Q, V_0Q)$. With $\lvert\cdot\rvert$ the entrywise absolute value, define
\begin{align*}
D_\epsilon(G) &\coloneqq G \oslash \bigl(\lvert G\rvert + \epsilon\bigr), &
E_Q(G) &\coloneqq D_\epsilon(GQ)Q^\top - D_\epsilon(G),\\
D_U &\coloneqq D_\epsilon(G_U), &
D_V &\coloneqq D_\epsilon(G_V).
\end{align*}
Then the products after one step satisfy the identity
\begin{equation}
\begin{aligned}
\widetilde W_1 - W_1
&= -\eta\bigl[E_Q(G_U)V_0^\top + U_0E_Q(G_V)^\top\bigr] \\
&\quad + \eta^2\bigl[
E_Q(G_U)D_V^\top + D_UE_Q(G_V)^\top + E_Q(G_U)E_Q(G_V)^\top
\bigr].
\end{aligned}
\label{eq:adam-defect}
\end{equation}

\end{proposition}

Here $\oslash$ is entrywise division, so each entry of $D_\epsilon(G)$ is
$G_{ij}/(\lvert G_{ij}\rvert+\epsilon)$, roughly the sign of $G_{ij}$ when
$\lvert G_{ij}\rvert\gg\epsilon$. The $O(\eta)$ term comes from each factor's own update, and the $O(\eta^2)$ term from the product of
the two updates. Gradient descent has the
identity map in place of $D_\epsilon$, which is why its defect is zero.

\subsection{The spectral transfer function}

\begin{proposition}[The spectral transfer function]
\label{prop:spectral}
Suppose $\Phi(G) = \psi(GG^\top)\,G$ for a scalar function $\psi$, meaning: diagonalize the symmetric
positive-semidefinite Gram matrix $GG^\top$ and replace each eigenvalue $\lambda$ by $\psi(\lambda)$. Given the singular value
decomposition $G = \sum_i \sigma_i u_i v_i^\top$, it follows that
\[
\Phi(G) \;=\; \sum_i h(\sigma_i)\, u_i v_i^\top,
\qquad \text{where} \quad h(\sigma) \coloneqq \psi(\sigma^2)\,\sigma.
\]
\end{proposition}

So the update keeps the gradient's singular vectors $u_i, v_i$ and only changes each singular value
$\sigma_i$ into $h(\sigma_i)$.  Whenever an inverse or
root of a singular positive-semidefinite matrix appears later, we apply it on the matrix's range and
set it to zero on the nullspace.

\paragraph{Instances.} For memoryless one-step methods, $h$ has a closed form.
\begin{itemize}\setlength{\itemsep}{1pt}\setlength{\parsep}{0pt}\setlength{\parskip}{0pt}
\item \emph{Gradient descent:} $h(\sigma) = \eta\sigma$.
\item \emph{Common-scalar-preconditioned GD:} $h_t(\sigma) = \eta\sigma/a_t$, which is a
time-rescaled GD (Theorem~\ref{thm:transfer}).
\item \emph{The exact polar map, and support-restricted undamped one-step Shampoo:} $h \equiv \eta$.

\end{itemize}

\paragraph{Aggressiveness.} The ratio $h(\sigma)/h(\sigma_{\max})$ says how much a rule boosts
small singular values relative to the largest one. The ratio alone does not
rank real optimizers, because real optimizers do not apply their one-step formula. Shampoo sums
gradients over all past steps and adds $\lambda I$, and Muon acts on
momentum. For instance, the idealized exact polar map and
undamped one-step Shampoo both have $h \equiv \eta$, yet the two deployed optimizers behave very differently (Table~\ref{tab:phase}).

\subsection{More than two factors}

With more than two factors, $W = U_1 U_2 \cdots U_D$, interior factors see gauge rotations on both
sides, not just one.

\begin{corollary}[Deeper factorizations]
\label{cor:depth}
Assume now a composite weight matrix $W = U_1U_2\cdots U_D$ of depth $D$, with $U_i\in\R^{k_{i-1}\times k_i}$ and loss $L=f(W)$. The loss is then unchanged by the following move at each interface $i$
between neighbouring factors:
\[
(U_i, U_{i+1}) \mapsto (U_iQ_i,\, Q_i^\top U_{i+1}),
\qquad Q_i \in \OO(k_i).
\]
Then gradients transform covariantly at each interface: each factor's gradient picks up the
interface rotation on the same side as the factor does. Consequently:
\begin{enumerate}\setlength{\itemsep}{1pt}\setlength{\parsep}{0pt}\setlength{\parskip}{0pt}
\item[(a)] \emph{Outer factors.} $U_1$ and $U_D$ see a one-sided action, and the earlier conditions
apply to them directly.
\item[(b)] \emph{Interior factors.} An interior factor $U_i$ sees both sides at once,
$G \mapsto Q_{i-1}^\top G\, Q_i$, and right-equivariance alone (the condition
$\Phi(GQ)=\Phi(G)Q$ used so far) is not sufficient there. What it needs is bi-equivariance,
\[
\Phi(Q_L^\top G\, Q_R) = Q_L^\top \Phi(G)\, Q_R
\qquad\text{for all orthogonal } Q_L, Q_R \text{ of matching sizes},
\]
with state transforming to match. The gradient-only, momentum, and $\msign$ rules satisfy this. A
fixed non-scalar left multiplier $\Phi(G) = CG$ does not: it is right-equivariant but fails on the
left.
\item[(c)] \emph{Non-equivariance carries over directly.} Any fixed nonlinear coordinate-wise
update rule fails to be equivariant at any interface where the internal dimension satisfies
$k_i \ge 2$.
\item[(d)] \emph{A common scalar.} If all factor gradients are scaled by the same positive scalar
$a(t)$, then the dynamics for all $D$ factors reduce to a single time-rescaled gradient flow, and
the transfer theorem continues to hold under its original conditions. Equivariance at individual
interfaces does not by itself imply the transfer.
\end{enumerate}
\end{corollary}

\subsection{Momentum and the small-step flow}

\begin{lemma}[Momentum preserves the small-step flow]
\label{lem:momentum}
Fix $\beta\in[0,1)$ and a gauge-invariant scalar preconditioner $a$, a function of the parameters
$\theta$. Take the first-moment recursion started from
$m_0 = 0$, with bias correction and update
\[
m_{t+1} = \beta m_t + (1-\beta)g_t,
\qquad
\hat m_{t+1} = \frac{m_{t+1}}{1-\beta^{t+1}},
\qquad
\theta_{t+1} = \theta_t - \eta\,\frac{\hat m_{t+1}}{a_t},
\]
where $g_t = \nabla L(\theta_t)$ and $a_t = a(\theta_t)$. We compare the iterates with the flow in
physical time $s = t\eta$, joining consecutive iterates by straight lines. So $[0,T]$ below is an
interval of physical time, and the iterate count grows as $\eta\to0$. Assume the following.
\begin{enumerate}\setlength{\itemsep}{1pt}\setlength{\parsep}{0pt}\setlength{\parskip}{0pt}
\item[(A1)] $\nabla L$ is locally Lipschitz.
\item[(A2)] $a$ is locally Lipschitz and bounded away from $0$ and $\infty$ on the region visited.
\item[(A3)] The iterates stay in a common compact set on $[0,T]$ for all small $\eta$.
\end{enumerate}
Then as $\eta\to0$ the piecewise-linear path through the iterates converges uniformly on $[0,T]$ to the flow $\dot\theta=-\nabla L(\theta)/a(\theta)$ of
Theorem~\ref{thm:transfer}. So the theorem's pathwise conclusions carry over to the iterates in
this limit.
\end{lemma}

\paragraph{Scope.}
\begin{itemize}\setlength{\itemsep}{1pt}\setlength{\parsep}{0pt}\setlength{\parskip}{0pt}
\item \emph{Time horizon.} The lemma covers a finite time horizon $[0,T]$. It does not take the
small-step and long-time limits together. So the agreement at the end of training reported in
\S\ref{sec:flowlimit} is an empirical observation.

\item \emph{Scalar-Adam.} This lemma does \emph{not} cover scalar-Adam's own second-moment
recursion. Its denominator depends on discrete state, and at fixed $\beta_2$ its memory window
depends on $\eta$. \S\ref{sec:flowlimit} checks it numerically.
\end{itemize}

\subsection{Balancedness}

Gradient flow keeps $U^\top U - V^\top V$ constant \citep{arora2018optimization,du2018algorithmic}.

\begin{proposition}[Balancedness]
\label{prop:balancedness}
Let $B_t = U_t^\top U_t - V_t^\top V_t$.
\begin{enumerate}\setlength{\itemsep}{1pt}\setlength{\parsep}{0pt}\setlength{\parskip}{0pt}
\item For standard gradient flow, and in particular in the setting of Theorem~\ref{thm:transfer},
$B_t$ is conserved during training.
\item For diagonally preconditioned flows of the form $\dot U = -D_U \odot \nabla_U L$ (and
analogously for $V$), where $D_U$ and $D_V$ hold one scale per entry, $B_t$ is conserved when $D_U$ and $D_V$ are one and the same positive scalar.
Otherwise it can drift, at a rate bounded by how far each departs from that common scalar.
\end{enumerate}

\end{proposition}

Balancedness is not necessary for good low-rank recovery. Neither Muon nor Shampoo conserves
$B_t$, and the $\norm{B_T}_F$ column of Table~\ref{tab:zoo} puts both far above gradient
descent's. Yet Muon still recovers near exactly, and Shampoo still beats every coordinate-wise method. The initialization scale is $10^{-3}$, so the
starting imbalance is negligible and those final values are drift built up during training.

\subsection{The phase boundary}
\label{app:twoscale}

\paragraph{An analytically tractable model.} The decoupled-mode model reduces the factored dynamics
to independent scalar components, each following its own one-dimensional ODE. The model also makes clear why the
initialization scale influences one schedule but not the other. The two schedules are GD's,
$h(\sigma)=\eta\sigma$, which we call \emph{greedy} because it grows large modes first, and the
polar map's, $h\equiv\eta$, which grows every mode at the same rate.

\begin{proposition}[Greedy vs.\ equal-rate schedules]
\label{prop:boundary}
Assume the following setting, the idealization adopted in \citet{arora2019implicit} and
\citet{li2021towards}.
\begin{itemize}\setlength{\itemsep}{1pt}\setlength{\parsep}{0pt}\setlength{\parskip}{0pt}
\item The factored dynamics decompose into independent, balanced scalar components $w = uv$, with
$u = v = \sqrt{w}$, and the target is aligned (it shares the modes' singular vectors).
\item The target modes satisfy $s_1 > s_2 > 0$, and all are initialized at $w_i(0) = w_0$ with
$0 < w_0 \le s_2/2$.
\item $T_1$ is the first time the leading mode reaches $(1-\delta)s_1$, with $\delta$ fixed such
that $0 < \delta < 1 - s_2/s_1$.
\end{itemize}
Then the two schedules separate as follows.
\begin{enumerate}\setlength{\itemsep}{1pt}\setlength{\parsep}{0pt}\setlength{\parskip}{0pt}
\item[(i)] \emph{Greedy schedule, $h(\sigma) = \eta\sigma$.} Each mode follows logistic growth,
$\dot w_i = 2\eta\, w_i(s_i - w_i)$, and at time $T_1$ the smaller mode remains inactive:
\[
w_2(T_1) \;\le\; C_{\delta,\rho}\, s_1 \Big(\tfrac{w_0}{s_1}\Big)^{1 - \rho}
\xrightarrow[\;w_0 \to 0\;]{} 0,
\qquad \rho \coloneqq s_2/s_1,\quad
C_{\delta,\rho} = 2\big(\tfrac{1-\delta}{\delta}\big)^{\rho}.
\]

\item[(ii)] \emph{Equal-rate schedule, $h \equiv \eta$.} All modes grow at the same rate, since
$\sqrt{w_i}$ increases linearly with time, and saturate at their respective $s_i$. By $T_1$, the
smaller mode is already fully fitted, $w_2(T_1) = s_2$, regardless of $w_0$. 
\end{enumerate}
\end{proposition}

\section{Proofs}
\label{app:proofs}

\subsection{Transfer theorem, formal statement}
\noindent\textbf{Theorem~\ref{thm:transfer} (transfer theorem), stated in full.}
Assume the following.
\begin{itemize}\setlength{\itemsep}{1pt}\setlength{\parsep}{0pt}\setlength{\parskip}{0pt}
\item The clock $a(t)>0$ is measurable\rev{ and} $1/a$ is locally integrable.
\item The trajectory $\theta = (U,V)$ is locally absolutely continuous (differentiable almost
everywhere and equal to the integral of its derivative) and satisfies
$\dot\theta = -\nabla L(\theta)/a(t)$.
\end{itemize}
Define the time change and the reparameterized trajectory
\[
\zeta(t) = \int_0^t \frac{du}{a(u)},
\qquad
\widetilde\theta(\zeta) \coloneqq \theta(t(\zeta)),
\]
where $t(\zeta)$ is the inverse of $\zeta(t)$. In words, $\zeta$ is a new clock that runs fast when
$a$ is small and slow when $a$ is large. Then the following hold.
\begin{enumerate}\setlength{\itemsep}{1pt}\setlength{\parsep}{0pt}\setlength{\parskip}{0pt}
\item \emph{Same path.} $\widetilde\theta'(\zeta) = -\nabla L(\widetilde\theta(\zeta))$ for almost
every $\zeta$, so both trajectories follow the same path in parameter space, the whole of it if
$\int_0^\infty du/a(u)=\infty$ and only an initial segment of it otherwise.
\item \emph{Same limit.} If in addition $\int_0^\infty du/a(u) = \infty$, the transformed time
reaches infinity. So if the gradient flow converges, the preconditioned version converges to the
same limit.
\item \rev{\emph{Uniform over rotations.} If in addition $a$ is determined solely by gauge-invariant
statistics of the trajectory up to time $t$, the time change is the same from every rotated start,
so items 1 and 2 hold uniformly over the gauge orbit.}
\end{enumerate}
\rev{Items 1 and 2 use only $a>0$. The gauge enters only through item 3.}

\subsection{Maximal isometric gauge}
This lemma is why the gauge we study is $\OO(k)$: among the maps $(U,V)\mapsto(UA,VA^{-\top})$, which preserve the product $UV^\top$, it is the largest group that also preserves parameter norms.
\begin{lemma}[Orthogonal gauge as maximal isometric symmetry]
\label{lem:max-isometric}
For any $A\in\mathrm{GL}(k)$, the transformation $(U,V)\mapsto(UA,\, VA^{-\top})$ preserves the
product $UV^\top$. This transformation preserves the Euclidean product metric
$\norm{U}_F^2+\norm{V}_F^2$ for all $(U,V)$ if and only if $A\in\OO(k)$.
\end{lemma}
\begin{proof}
\emph{The product is preserved.} For $A\in\mathrm{GL}(k)$,
\[
(UA)(VA^{-\top})^\top = UAA^{-1}V^\top = UV^\top.
\]

\emph{Isometry implies orthogonal.} Suppose the action is an isometry for every $(U,V)$. Take
$V=0$ and let $U$ have a single nonzero row $u\in\R^{1\times k}$, which is arbitrary. Then
$\norm{uA}_2=\norm{u}_2$ for every $u$, hence $AA^\top=I$. Since $A$ is square, this is equivalent
to $A\in\OO(k)$.

\emph{Orthogonal implies isometry.} If $A\in\OO(k)$, then $A^{-\top} = A$, and both right
multiplications preserve Frobenius norms.
\end{proof}

\subsection{Gradient covariance}

\begin{lemma}[Gradient covariance]
\label{lem:grad-cov}
Let $L(U,V) = f(UV^\top)$ and let $Q\in\OO(k)$. Then
$\nabla_U L(UQ,VQ) = \bigl(\nabla_U L(U,V)\bigr) Q$ and
$\nabla_V L(UQ,VQ) = \bigl(\nabla_V L(U,V)\bigr) Q$.
\end{lemma}
\begin{proof}
Let $W = UV^\top$ and $G = \nabla f(W)\in\R^{n\times n}$. By the chain rule,
\[
\nabla_U L = GV, \qquad \nabla_V L = G^\top U.
\]
At the transformed point $(UQ,VQ)$ the represented matrix is unchanged, $W' = UQ(VQ)^\top = W$,
so $\nabla f(W') = G$. Therefore
\begin{align*}
\nabla_U L(UQ,VQ) &= G\,(VQ) = (GV)\,Q = \bigl(\nabla_U L(U,V)\bigr)\,Q,\\
\nabla_V L(UQ,VQ) &= G^\top (UQ) = (G^\top U)\,Q = \bigl(\nabla_V L(U,V)\bigr)\,Q. \qedhere
\end{align*}
\end{proof}

\subsection{Proposition~\ref{prop:equivariant} (equivariant class)}
\begin{proof}
Write primed quantities for the run started at $(U_0Q, V_0Q)$, with its state initialized by the
map $\sigma_Q$ constructed below. We show by induction on $t$ that
\[
U'_t = U_tQ, \qquad V'_t = V_tQ,
\]
with the state transforming as stated. If the two
runs are gauge-related at step $t$, Lemma~\ref{lem:grad-cov} gives $G'_t = G_tQ$ for both factor
gradients. Each case below then shows that the update is also right-multiplied by $Q$, so the runs
are still gauge-related at step $t+1$.

\textbf{Case 1: GD and momentum.} For GD,
\[
U'_{t+1} = U'_t - \eta G'_t = (U_t - \eta G_t)\,Q.
\]
Momentum buffers are linear in the gradients. With $\sigma_Q: M\mapsto MQ$, induction gives
\[
M'_t = \beta M'_{t-1} + G'_t = M_tQ,
\]
and the update again right-multiplies by $Q$. For Nesterov, the gradient is evaluated at the
lookahead point $U_t + \beta(U_t - U_{t-1})$. In the primed run this is the same point
right-multiplied by $Q$, so by Lemma~\ref{lem:grad-cov} the lookahead gradient is covariant too.

\textbf{Case 2: scalar-preconditioned Adam.} The first moment is an EMA, which is linear, so $M'_t = M_tQ$. The second-moment state is the single scalar
\[
\nu_t = \beta_2\,\nu_{t-1} + (1-\beta_2)\,\mathrm{mean}\bigl(G_t^{\odot2}\bigr),
\qquad \nu_0 = 0,
\]
where the mean runs over all entries of both factors. The mean of $G^{\odot2}$ is $\norm{G}_F^2$
divided by the entry count, so $\nu_t$ sees the gradients only through their Frobenius norms.
Since $\norm{G_tQ}_F = \norm{G_t}_F$,
\[
\nu'_t = \nu_t,
\qquad
\hat\nu'_t = \frac{\nu'_t}{1-\beta_2^t} = \hat\nu_t .
\]
The state map is therefore
\[
\sigma_Q:(M_U, M_V, \nu)\mapsto(M_UQ,\, M_VQ,\, \nu),
\]
with $\sigma_Q(0)=0$ as Definition~\ref{def:equivariance} requires. The update
$-\eta\,\hat M_t/(\sqrt{\hat\nu_t}+\epsilon)$ is a scalar multiple of a covariant matrix, hence
covariant.

\textbf{Case 3: Muon.} With momentum $M'_t = M_tQ$ as in Case~1, it suffices to show
$\msign(MQ) = \msign(M)\,Q$. Let $M = A\Sigma B^\top$ be an SVD. Then
$MQ = A\Sigma(Q^\top B)^\top$ is an SVD of $MQ$, because $Q^\top B$ still has orthonormal columns.
So
\[
\msign(MQ) = A\,(Q^\top B)^\top = AB^\top Q = \msign(M)\,Q.
\]
The same holds for the Newton--Schulz iteration
\[
X_{j+1} = \gamma_1X_j + \gamma_2\,(X_jX_j^\top)X_j + \gamma_3\,(X_jX_j^\top)^2 X_j,
\qquad
X_0 = \frac{M}{\norm{M}_F+\epsilon_{\mathrm{NS}}}.
\]
Since $\norm{MQ}_F=\norm{M}_F$, $X_0\mapsto X_0Q$. If $X_j \mapsto X_jQ$, then $X_jX_j^\top$ is
invariant, so every term right-multiplies by $Q$. This holds at every iteration and at any truncation.

\textbf{Case 4: Shampoo.} The left accumulator is invariant and the right one is conjugated:
\[
L'_t = \sum_{u\le t} G_uQQ^\top G_u^\top = L_t,
\qquad
R'_t = \sum_{u\le t} Q^\top G_u^\top G_uQ = Q^\top R_tQ.
\]
For any spectral function $K$ we have $K(Q^\top RQ) = Q^\top K(R)\,Q$, because $Q^\top RQ$ has the
same eigenvalues as $R$, with eigenvectors rotated by $Q^\top$. Take
$K(R) = (R+\lambda I)^{-1/4}$. The damping is covered, since
$Q^\top RQ + \lambda I = Q^\top(R+\lambda I)Q$. Hence
\begin{align*}
(L'_t + \lambda I)^{-1/4}\, G'_t\, (R'_t + \lambda I)^{-1/4}
&= (L_t+\lambda I)^{-1/4}\, G_tQ\, Q^\top (R_t+\lambda I)^{-1/4} Q \\
&= \bigl[(L_t+\lambda I)^{-1/4}\, G_t\, (R_t+\lambda I)^{-1/4}\bigr]\, Q. \qedhere
\end{align*}
\end{proof}

\subsection{Proposition~\ref{prop:coordwise} (coordinate-wise rules break the gauge)}
\begin{proof}
\pstep{1}{two scalar identities} Take $k\ge2$ and let $Q_\theta$ rotate the first two coordinates
by $\theta$, acting on a row from the right:
$(x,y)\,Q_\theta = (x\cos\theta - y\sin\theta,\; x\sin\theta + y\cos\theta)$. Apply $\Phi(GQ)=\Phi(G)Q$
to a matrix $G$ whose only nonzero row is $(x, y, 0, \dots)$. Entry by entry, it reads
\begin{equation}
\begin{aligned}
\phi(x\cos\theta - y\sin\theta) &= \phi(x)\cos\theta - \phi(y)\sin\theta,\\
\phi(x\sin\theta + y\cos\theta) &= \phi(x)\sin\theta + \phi(y)\cos\theta,
\end{aligned}
\label{eq:coordwise-pair}
\end{equation}
for all $x, y, \theta$.

\pstep{2}{$\phi(0)=0$} Setting $x=y=0$ and $\theta=\pi/2$ in the first identity of \eqref{eq:coordwise-pair} gives $\phi(0) = -\phi(0)$, so $\phi(0)=0$.

\pstep{3}{homogeneity on $[-1,1]$} Setting $y=0$ in the first identity of \eqref{eq:coordwise-pair} and
using $\phi(0)=0$ gives $\phi(x\cos\theta) = \phi(x)\cos\theta$. As $\theta$ ranges over $[0,\pi]$,
$\cos\theta$ sweeps all of $[-1,1]$. Hence
\begin{equation}
\phi(cx) = c\,\phi(x) \qquad\text{for every } c\in[-1,1] \text{ and every } x\in\R.
\label{eq:coordwise-hom}
\end{equation}

\pstep{4}{linearity} Take any $x$ and any $u$ with $0<|u|\le|x|$. Choosing
$c = u/x \in[-1,1]$ in \eqref{eq:coordwise-hom} gives $\phi(u) = (u/x)\,\phi(x)$, that is, $\phi(u)/u = \phi(x)/x$.
Any two nonzero reals are both dominated in absolute value by a common third real. So the ratio
$\phi(x)/x$ takes a single value $c$ on $\R\setminus\{0\}$, and with $\phi(0)=0$, $\phi(x) = c\,x$ for all $x\in\R$.

\pstep{5}{consequence} Any nonlinear fixed entrywise map therefore fails equivariance already for a
$2\times2$ rotation. The standard zero-state first steps of Adam and RMSProp are nonlinear maps of
this form (Adam's first step is $-\eta D_\epsilon(G)$, Proposition~\ref{prop:adam-onestep}), and so
are signSGD and Lion's outer-sign update.

\textbf{Adafactor.} Its second-moment surrogate is rank-one factored rather than entrywise, so
Steps~1--4 do not apply and we check it directly. Adafactor keeps the accumulator
$v = \mathrm{EMA}(g^2)$ and replaces the per-coordinate second moment by the rank-one
reconstruction from its row and column sums:
\[
r_i = \sum_j v_{ij}, \qquad c_j = \sum_i v_{ij}, \qquad
\hat v_{ij} = \frac{r_i\,c_j}{\sum_l r_l}, \qquad
\text{update } \frac{g_{ij}}{\sqrt{\hat v_{ij}}}.
\]
Take a single step, so that $v = g^{\odot2}$ after bias correction. Adafactor's constant positive
factors multiply both sides of the test equally, so we drop them. Under $G\mapsto GQ$ the row sums
stay the same, but the column sums become the diagonal of $Q^\top G^\top GQ$, which depends on the
basis:
\[
r_i = (GG^\top)_{ii},
\qquad
c_j = (Q^\top G^\top G\,Q)_{jj}.
\]
Take $G=\mathrm{diag}(1,\delta)$ with $0<\delta\ne1$ and
the $\theta=\pi/4$ rotation $Q=2^{-1/2}\begin{psmallmatrix}1&1\\-1&1\end{psmallmatrix}$. It leaves
$r=(1,\delta^2)$ fixed and sends
\[
c=(1,\delta^2) \;\mapsto\; \tfrac12(1+\delta^2)\,(1,1).
\]
The resulting zero-state updates are
\[
\Phi(G)=\sqrt{1+\delta^2}\,\mathrm{diag}(1,\delta^{-1}),
\qquad
\Phi(GQ)=\begin{psmallmatrix}1&1\\-1&1\end{psmallmatrix},
\]
whereas
\[
\Phi(G)\,Q=\frac{\sqrt{1+\delta^2}}{\sqrt2}
\begin{psmallmatrix}1&1\\-\delta^{-1}&\delta^{-1}\end{psmallmatrix}\ne\Phi(GQ).
\]
The two sides differ at order $1$ for fixed $\delta$. A small scalar stabilizer added to $\hat v$
perturbs both sides continuously, so a sufficiently small stabilizer does not reconcile them.
\end{proof}

\subsection{Proposition~\ref{prop:adam-onestep} (exact first-step Adam defect)}
\begin{proof}
For either factor and its gradient $G$, zero initial state and bias correction give
\[
m_1=(1-\beta_1)G,\qquad v_1=(1-\beta_2)G^{\odot 2},
\qquad \hat m_1=G,\qquad \hat v_1=G^{\odot 2}.
\]
Since $\sqrt{G^{\odot2}} = \lvert G\rvert$, the first Adam update is exactly $-\eta D_\epsilon(G)$.
In particular, it does not depend on $\beta_1$ or $\beta_2$ under the stated convention. By
Lemma~\ref{lem:grad-cov}, the gradients in the rotated run are $\widetilde G_U=G_UQ$ and
$\widetilde G_V=G_VQ$. Multiply the rotated run's first iterate by $Q^\top$ on the right. This does
not change its product $\widetilde W_1$, and it puts both runs in the same coordinates:
\begin{align*}
\widetilde U_1Q^\top &= U_0-\eta D_\epsilon(G_UQ)Q^\top = U_0-\eta\bigl[D_U+E_Q(G_U)\bigr],\\
\widetilde V_1Q^\top &= V_0-\eta D_\epsilon(G_VQ)Q^\top = V_0-\eta\bigl[D_V+E_Q(G_V)\bigr].
\end{align*}
Since $\widetilde W_1 = (\widetilde U_1Q^\top)(\widetilde V_1Q^\top)^\top$ and
$W_1 = (U_0-\eta D_U)(V_0-\eta D_V)^\top$, subtracting gives
\[
\widetilde W_1 - W_1
= \bigl(U_0-\eta[D_U+E_Q(G_U)]\bigr)\bigl(V_0-\eta[D_V+E_Q(G_V)]\bigr)^\top
- (U_0-\eta D_U)(V_0-\eta D_V)^\top .
\]
Expanding both products, the terms $U_0V_0^\top$, $-\eta U_0D_V^\top$, $-\eta D_UV_0^\top$ and
$\eta^2D_UD_V^\top$ cancel, and what remains is the identity \eqref{eq:adam-defect}.

A $k{=}2$ witness: take $n=1$, $f(w) = \tfrac12 w^2$, $U_0 = V_0 = (1,0)$ and $Q = 2^{-1/2}\begin{psmallmatrix}1&1\\-1&1\end{psmallmatrix}$. Then $G_U=G_V=(1,0)$ while $G_UQ=G_VQ=(2^{-1/2},2^{-1/2})$, and the two products after one Adam step are $(1-\eta/(1+\epsilon))^2$ and $(1-\eta/(2^{-1/2}+\epsilon))^2$, which differ whenever $0 < \eta < 2^{-1/2}+\epsilon$.
\end{proof}

\subsection{Theorem~\ref{thm:structure} (structure)}
\begin{proof}
\rev{The statement is Proposition~4 of \citet{villar2021scalars} with their Lemma~1, applied row by row.}

\emph{If ($\Leftarrow$): every rule of the form $H(GG^\top)G$ is equivariant.} Since
$(GQ)(GQ)^\top = GG^\top$,
\[
\Phi(GQ) = H(GG^\top)\,GQ = \Phi(G)\,Q.
\]

\emph{Only if ($\Rightarrow$): every equivariant memoryless rule has this form.} Assume
$\Phi(GQ) = \Phi(G)\,Q$ for every $Q\in\OO(k)$. Let $G^{+}$ be the Moore--Penrose pseudoinverse and
$P = G^{+}G$ the orthogonal projector onto $\mathrm{row}(G)$, the span of the rows of $G$. We build
$H$ from the candidate
\[
X(G) \coloneqq \Phi(G)\,G^{+} \in \R^{n\times n}
\]
by showing that it reproduces $\Phi$ (Step~1), is gauge-invariant (Step~2), and depends on $G$
only through $GG^\top$ (Step~3). 

\pstep{1}{$\Phi(G)$ ignores directions outside $\mathrm{row}(G)$} The matrix $R \coloneqq 2P - I$ is
a reflection: it fixes $\mathrm{row}(G)$ and flips its orthogonal complement. So $R$ is orthogonal,
and $GR = G$ because $GP = G$. Equivariance with $Q=R$ gives
\[
\Phi(G) = \Phi(GR) = \Phi(G)\,R = 2\,\Phi(G)P - \Phi(G).
\]
Hence $\Phi(G) = \Phi(G)P = \Phi(G)G^{+}G$, that is,
\begin{equation}
X(G)\,G = \Phi(G).
\label{eq:structure-reproduce}
\end{equation}

\pstep{2}{$X$ is gauge-invariant} For orthogonal $Q$, at any rank, $(GQ)^{+} = Q^\top G^{+}$ (a
standard property of the pseudoinverse). Therefore
\[
X(GQ) = \Phi(GQ)\,(GQ)^{+} = \Phi(G)\,Q\,Q^\top G^{+} = X(G).
\]

\pstep{3}{equal Grams lie on one gauge orbit} Let $G_1, G_2 \in \R^{n\times k}$ have
$G_1G_1^\top = G_2G_2^\top = E\Lambda E^\top$, where $\Lambda$ is the $r\times r$ diagonal matrix of
the nonzero eigenvalues, all strictly positive ($\Lambda\succ0$), and $E$ holds the matching
eigenvectors. Both $G_i$ then
have the same left singular vectors and singular values, and differ only in their right singular
vectors:
\[
G_i = E\Lambda^{1/2}F_i^\top,
\qquad F_i\in\R^{k\times r} \text{ with orthonormal columns}.
\]
Complete each $F_i$ to $\widetilde F_i\in\OO(k)$ and set $Q = \widetilde F_1\widetilde F_2^\top$.
Then $F_1^\top\widetilde F_1 = \begin{bmatrix} I_r & 0\end{bmatrix}$, so
\[
G_1Q = E\Lambda^{1/2}\,F_1^\top\widetilde F_1\,\widetilde F_2^\top
= E\Lambda^{1/2}\begin{bmatrix} I_r & 0\end{bmatrix}\widetilde F_2^\top
= E\Lambda^{1/2}F_2^\top = G_2.
\]
With Step~2, $X(G_2) = X(G_1Q) = X(G_1)$.

\pstep{4}{conclusion} By Steps~2 and~3, $H(GG^\top)\coloneqq X(G) = \Phi(G)G^{+}$ is well
defined, and \eqref{eq:structure-reproduce} gives
\[
\Phi(G) = H(GG^\top)\,G \qquad\text{at every rank.} \qedhere
\]
\end{proof}

\subsection{Proposition~\ref{prop:spectral} (spectral transfer)}
\begin{proof}
\pstep{1}{the formula} Take the compact SVD $G = \sum_i \sigma_i u_iv_i^\top$. Then
$GG^\top = \sum_i \sigma_i^2\,u_iu_i^\top$, and functional calculus gives
\[
\psi(GG^\top) = \sum_i \psi(\sigma_i^2)\,u_iu_i^\top + \psi(0)\,P_0,
\qquad
P_0 = I - \sum_i u_iu_i^\top .
\]
Here $P_0$ projects onto $\ker(GG^\top)$. Since $P_0G = 0$, that term drops out:
\[
\psi(GG^\top)\,G = \sum_i \psi(\sigma_i^2)\,\sigma_i\,u_iv_i^\top = \sum_i h(\sigma_i)\,u_iv_i^\top .
\]

\pstep{2}{the instances} These follow by inspection. For support-restricted one-step undamped
Shampoo,
\[
(GG^\top)_{\mathrm{supp}}^{-1/4}\,G\,(G^\top G)_{\mathrm{supp}}^{-1/4}
= \sum_i \sigma_i^{-1/2}\sigma_i\sigma_i^{-1/2}u_iv_i^\top
= \sum_i u_iv_i^\top
= \msign(G). \qedhere
\]
\end{proof}

\subsection{Corollary~\ref{cor:depth} (deeper factorizations)}
\begin{proof}
Write $L(U_1,\dots,U_D) = f(U_1\cdots U_D)$ and fix an interface $i$. Call the product of the
factors before the interface the prefix $P$, and the product after it the suffix $S$:
\[
P = U_1\cdots U_{i-1}, \qquad S = U_{i+2}\cdots U_D, \qquad G = \nabla f(W).
\]
The chain rule gives
\[
\nabla_{U_i}L = P^\top G\,(U_{i+1}S)^\top,
\qquad
\nabla_{U_{i+1}}L = (PU_i)^\top G\,S^\top .
\]
Under $(U_iQ,\, Q^\top U_{i+1})$ the represented $W$ is unchanged, hence so is $G$, and
\begin{align*}
\nabla_{U_i}L &\mapsto P^\top G\,(Q^\top U_{i+1}S)^\top = (\nabla_{U_i}L)\,Q,\\
\nabla_{U_{i+1}}L &\mapsto (PU_iQ)^\top G\,S^\top = Q^\top(\nabla_{U_{i+1}}L).
\end{align*}
So each factor's gradient picks up the rotation on the same side as the factor. This is the only
property the proofs of Propositions~\ref{prop:equivariant} and~\ref{prop:coordwise} use, transposed
for the left action. Items (a)--(c) follow by rerunning those proofs on each side of every
interface, and (d) is Theorem~\ref{thm:transfer} applied to all $D$ factors at once.
\end{proof}

\subsection{Theorem~\ref{thm:transfer} (transfer theorem)}
\begin{proof}

\pstep{1}{the new clock} Let $\zeta_{\max} = \int_0^\infty du/a(u)$. Then $\zeta$ is a strictly increasing bijection onto $[0,\zeta_{\max})$: it is continuous because
$1/a$ is locally integrable, and strictly increasing because $1/a>0$. Let $t(\zeta)$ be its inverse
and $\tilde\theta(\zeta) = \theta(t(\zeta))$.

\pstep{2}{the chain rule} For almost every $\zeta$,
\[
\frac{d\tilde\theta}{d\zeta}
= \dot\theta\,\frac{dt}{d\zeta}
= \Bigl(-\frac{\nabla L(\theta)}{a}\Bigr)\cdot a
= -\nabla L(\tilde\theta).
\]
So $\tilde\theta$ is a Carath\'eodory gradient-flow solution (one that satisfies the ODE almost
everywhere) with the same initial condition.

\pstep{3}{the same path} The two trajectories traverse the same set of points under a monotone
reparameterization. That set is the whole gradient-flow path when $\zeta_{\max}=\infty$, and
otherwise its prefix $\tilde\theta([0,\zeta_{\max}))$. So any property of the gradient-flow
\emph{path} on the traversed portion holds for the scalar-preconditioned flow.

\pstep{4}{the same limit} Suppose $\zeta_{\max} = \infty$. This holds in particular whenever $a$ is
bounded above, since then $\int^\infty du/a(u)$ diverges. The reparameterization then also reaches
the gradient-flow limit, so the limit-point characterizations of
\citet{gunasekar2017implicit,arora2019implicit,li2021towards} transfer as well, under their
hypotheses.

\pstep{5}{uniform over the gauge orbit} \rev{If} the clock $a$ is computed from gauge-invariant statistics, then by Lemma~\ref{lem:grad-cov} the rotated path $\theta(t)Q$ solves the flow from the
rotated start with the same clock. So the reparameterization is the same from every rotated start,
and the statement holds uniformly over the gauge orbit.
\end{proof}

\subsection{Lemma~\ref{lem:momentum} (momentum preserves the small-step flow)}
\begin{proof}
\emph{Idea.} Momentum replaces the current gradient by a weighted average of recent gradients. When
the step size is small, the recent iterates all sit close to the current one, so this average is
within $O(\eta)$ of the current gradient. The iteration is then Euler's method for
$\dot\theta = -\nabla L(\theta)/a(\theta)$ with a small error at each step, and Euler's method
converges to the flow as $\eta\to0$. 

\pstep{1}{a convex combination} Unrolling $m_{t+1} = \beta m_t + (1-\beta)g_t$ from $m_0 = 0$
gives
\[
m_{t+1} = (1-\beta)\sum_{j=0}^{t}\beta^{\,t-j}g_j,
\qquad
(1-\beta)\sum_{j=0}^{t}\beta^{\,t-j} = 1-\beta^{t+1}.
\]
So the bias-corrected quantity
\[
\hat m_{t+1} \;=\; \frac{m_{t+1}}{1-\beta^{t+1}}
\;=\; \sum_{j=0}^{t} w^{(t)}_j\, g_j,
\qquad
w^{(t)}_j = \frac{(1-\beta)\beta^{\,t-j}}{1-\beta^{t+1}},
\qquad
\sum_{j=0}^{t} w^{(t)}_j = 1,
\]
is an average of past gradients.

\pstep{2}{the lag is $O(\eta)$} Fix a compact time interval $[0,T]$. By assumption (A3), there is a
compact set $K$ that contains the flow and every discrete trajectory on $[0,T]$, for all small
$\eta$.  On $K$, let $\Lambda$ be a Lipschitz constant for $\nabla L$, which
exists by (A1), and let $\bar G = \sup_K\norm{\nabla L}$ be the largest gradient norm. By (A2),
$\inf a > 0$ on $K$. Since $\hat m$ is an average of gradients, $\norm{\hat m}\le\bar G$, so each
step moves the iterate by at most $\eta\bar G/\inf a$. After $i$ steps,
\[
\norm{\theta_t - \theta_{t-i}} \le \frac{i\,\eta\bar G}{\inf a}.
\]
Hence
\begin{align}
\bigl\lVert \hat m_{t+1} - \nabla L(\theta_t)\bigr\rVert
&= \Bigl\lVert \sum_{j\le t} w^{(t)}_j\bigl(\nabla L(\theta_j) - \nabla L(\theta_t)\bigr)\Bigr\rVert
\nonumber\\
&\le \Lambda\,\frac{\eta \bar G}{\inf a}\sum_{j\le t} w^{(t)}_j\,(t-j) \nonumber\\
&\le \Lambda\,\frac{\eta \bar G}{\inf a}\cdot\frac{\beta}{1-\beta}.
\label{eq:momentum-lag}
\end{align}
The last line bounds the weighted average lag $\sum_j w^{(t)}_j(t-j)$ by the average lag of the untruncated
geometric weights, $\sum_{i\ge0} i(1-\beta)\beta^{i} = \beta/(1-\beta)$. So $\hat m_{t+1}$ differs from $\nabla L(\theta_t)$ by $O(\eta)$, uniformly on $[0,T]$. The constant
is proportional to the memory length $\beta/(1-\beta)$, which is held fixed.

\pstep{3}{a perturbed Euler scheme for a Lipschitz field} By \eqref{eq:momentum-lag}, each step equals
the Euler step $\theta_t - \eta\nabla L(\theta_t)/a_t$ plus an error of size
$\eta\cdot O(\eta) = O(\eta^2) = o(\eta)$. So the iteration is a perturbed Euler discretization of
$\dot\theta = -\nabla L(\theta)/a(\theta)$. The right-hand side of this ODE is Lipschitz on $K$.
Here $\nabla L$ is Lipschitz by (A1), and $a$ is Lipschitz and bounded below by (A2). A quotient of
Lipschitz functions whose denominator is bounded away from zero is Lipschitz on a compact set.

\pstep{4}{convergence} Euler's method with per-step errors of size $O(\eta^2)$ on a Lipschitz field
converges uniformly on $[0,T]$ as $\eta\to0$, at first order in $\eta$. This is the standard
convergence estimate for one-step methods (a discrete Gr\"onwall argument).

\pstep{5}{inside the equivariant class} The preconditioner $a$ is gauge-invariant by assumption. The averaging is linear in the gradients, hence equivariant by
Proposition~\ref{prop:equivariant}(1). So the limit is taken within the equivariant class, and the
conclusion holds uniformly over the gauge orbit. 
\end{proof}

\subsection{Proposition~\ref{prop:balancedness} (balancedness)}
\begin{proof}
Let $G = \nabla f(W)$.

\textbf{Gradient flow.} By the chain rule of Lemma~\ref{lem:grad-cov}, $\dot U = -GV$ and
$\dot V = -G^\top U$, so
\[
\tfrac{d}{dt}(U^\top U) = -V^\top G^\top U - U^\top G V, \qquad
\tfrac{d}{dt}(V^\top V) = -U^\top G V - V^\top G^\top U.
\]
The two derivatives are equal, so $\dot B = 0$. A time reparameterization by $1/a(t)$ multiplies
both by the same scalar, so conservation persists for the whole class of
Theorem~\ref{thm:transfer}.

\textbf{Diagonal preconditioning.} Now $\dot U = -D_U \odot (GV)$ and
$\dot V = -D_V \odot (G^\top U)$, so
\[
\dot B = -\big[(D_U \odot GV)^\top U + U^\top (D_U \odot GV)\big]
         +\big[(D_V \odot G^\top U)^\top V + V^\top (D_V \odot G^\top U)\big].
\]
For example, $D_U \equiv 1$ and $D_V \equiv 2$ give
$\dot B = U^\top GV + V^\top G^\top U$, which is nonzero in general, so agreeing up to a factor is
not enough. In general, write each preconditioner as a \emph{common} scalar $c$ plus a remainder,
\[
D_U = c\,\mathbf{1}\mathbf{1}^\top + E_U,
\qquad
D_V = c\,\mathbf{1}\mathbf{1}^\top + E_V,
\]
where $\mathbf{1}\mathbf{1}^\top$ is the all-ones matrix of the factor's shape. The common-scalar
terms cancel, while the remainder obeys
\[
\norm{\dot B}_F
\le 2\norm{E_U}_\infty\norm{GV}_F\norm{U}_F
   +2\norm{E_V}_\infty\norm{G^\top U}_F\norm{V}_F,
\]
by $\norm{E\odot Z}_F\le\norm{E}_\infty\norm{Z}_F$, with $\norm{\cdot}_\infty$ the entrywise
maximum, and Frobenius submultiplicativity. So at bounded factor and gradient norms, the drift
vanishes as both factor preconditioners approach the same scalar.
\end{proof}

\subsection{Proposition~\ref{prop:boundary} (greedy vs.\ equal-rate schedules)}
\begin{proof}
\textbf{Reduction to one variable per mode.} Each decoupled mode is a balanced two-factor scalar
factorization $w = uv$ with $u = v = \sqrt{w}$ and per-mode loss $\tfrac12(uv - s)^2$, so
\[
\partial_u L = v\,(w - s), \qquad \partial_v L = u\,(w - s).
\]
The balanced set $u = v$ is invariant under both flows below, since $\dot u = \dot v$ there. So
the reduction to the single variable $w$ is consistent. 

\textbf{(i) Greedy schedule, $h(\sigma) = \eta\sigma$.} The gradient flow is
$\dot u = -\eta\, v(w - s)$, $\dot v = -\eta\, u(w - s)$. Using $u^2 = v^2 = w$ at balance,
\[
\dot w = \dot u\,v + u\,\dot v = -\eta\,(u^2 + v^2)(w - s) = 2\eta\, w\,(s - w).
\]
This is logistic growth, with solution
\[
w(t) = s\,\Bigl(1 + \bigl(\tfrac{s}{w_0} - 1\bigr)\,e^{-2\eta s t}\Bigr)^{-1}.
\]
Solving this formula for the time at which the head reaches $w_1(T_1) = (1-\delta)s_1$ gives
\begin{equation}
e^{2\eta s_1 T_1} = \frac{1-\delta}{\delta}\cdot\frac{s_1 - w_0}{w_0}
\le \frac{1-\delta}{\delta}\cdot\frac{s_1}{w_0}.
\label{eq:boundary-head}
\end{equation}
For the tail, dropping the $1$ in the denominator of the logistic solution gives an upper bound:
\[
w(t) \le \frac{s\,w_0}{s - w_0}\, e^{2\eta s t} \le 2 w_0\, e^{2\eta s t}
\qquad\text{for } w_0 \le s/2.
\]
Apply this to mode~2 at $t = T_1$. Since $\rho = s_2/s_1$, we have
$e^{2\eta s_2 T_1} = \bigl(e^{2\eta s_1 T_1}\bigr)^{\rho}$, and inserting \eqref{eq:boundary-head}
gives
\[
w_2(T_1) \;\le\; 2 w_0 \big(e^{2\eta s_1 T_1}\big)^{\rho}
\;\le\; 2\Big(\tfrac{1-\delta}{\delta}\Big)^{\rho} s_1^{\,\rho}\, w_0^{\,1 - \rho}
\;=\; C_{\delta,\rho}\, s_1 \Big(\tfrac{w_0}{s_1}\Big)^{1 - \rho}.
\]

\textbf{(ii) Equal-rate schedule, $h \equiv \eta$.} This flow replaces the update of every factor
with $\eta$ times the matrix sign of its gradient. In the scalar case, for $0 < w < s$,
\[
\dot u = -\eta\,\mathrm{sign}\bigl(v(w - s)\bigr) = \eta,
\qquad
\dot v = \eta,
\]
because $v>0$ and $w-s<0$, so the sign is $-1$. So both factors grow linearly, and
\[
\sqrt{w(t)} = \sqrt{w_0} + \eta t \qquad\text{until } w \text{ reaches the threshold } s.
\]
There the gradient becomes zero, and since $\msign(0)=0$ the mode stops and stays at $s_i$. Mode $i$ fits at $t_i$, and the head reaches $(1-\delta)s_1$ at $T_1$, where
\[
t_i = \frac{\sqrt{s_i} - \sqrt{w_0}}{\eta},
\qquad
T_1 = \frac{\sqrt{(1-\delta)s_1} - \sqrt{w_0}}{\eta}.
\]
For $\delta < 1 - s_2/s_1$ we have $(1-\delta)s_1 > s_2$, hence $t_2 < T_1$ and $w_2(T_1) = s_2$
exactly, with no dependence on $w_0$.
\end{proof}

\section{Experimental details}
\label{app:details}

\paragraph{The two reported metrics.} The two metrics are
\[
\rec(W) = \norm{W - X^\ast}_F / \norm{X^\ast}_F,
\qquad
\er(W) = \exp\Bigl(-\textstyle\sum_i \pi_i \log \pi_i\Bigr),
\]
with $\pi_i = \sigma_i/\sum_j\sigma_j$ and $0\log0\coloneqq0$ \citep{roy2007effective}.
Throughout, \emph{recovery} means the error $\rec(W)$, so smaller is better.

The sensing runs use full overparameterization $k{=}n$ and no weight decay.

\paragraph{Sensing task (\S\ref{sec:zoo}, \S\ref{sec:phase}).}
\begin{spec}
\item \emph{Target and measurements.} $X^\ast = U^\ast V^{\ast\top}$,
$U^\ast, V^\ast \in \R^{40\times3}$ i.i.d.\ Gaussian scaled by $1/\sqrt{3}$;
$m = 462 = 2\,\mathrm{dof}$, $\mathrm{dof} = 3(80-3) = 231$; $A_i$ i.i.d.\ standard Gaussian;
loss $= \mathrm{mean}_i(\ip{A_i}{W} - y_i)^2$.
\item \emph{Factors.} $U, V \in \R^{40\times40}$ initialized $\mathcal{N}(0, 10^{-6})$ entrywise
($10^{-3}$ scale).
\item \emph{Budget.} $2\times10^4$ steps (extended where noted); interpolation bar $10^{-7}$
(checked every 200 steps); 3 paired seeds \{42, 123, 456\}.
\end{spec}

Optimizer hyperparameters:
\begin{spec}
\item \emph{Adam, scalar-Adam.} $\beta = (0.9, 0.999)$, $\epsilon = 10^{-8}$.
\item \emph{RMSProp.} Classical form, second-moment decay $0.99$ and no first moment,
$\epsilon = 10^{-8}$.
\item \emph{Muon.} Momentum 0.9, Newton--Schulz 5 iterations, coefficients
$(\gamma_1,\gamma_2,\gamma_3) = (3.4445, -4.7750, 2.0315)$ \citep{jordan2024muon}.
\item \emph{Shampoo.} Accumulated (non-EMA) $L, R$, damping $\lambda = 1$, inverse-root refresh
every 20 steps.
\item \emph{Adafactor.} Row/column EMA factored second moment.
\item \emph{Lion.} $\beta_1 = 0.9$ inner, 0.99 EMA.
\item \emph{signum.} Sign of an EMA momentum with $\beta_1 = 0.9$ (Adam's $\beta_1$) and no second
moment, no clipping.
\end{spec}

Table~\ref{tab:grids} lists the learning-rate grids and the rate each method's best-recovery rule selects.

\begin{table}[tbp]
\centering\small
\caption{\textbf{Selected rates minimize recovery error among runs that reach interpolation.}
These grids produce Table~\ref{tab:zoo}. Selection and evaluation use the same three seeds.}
\label{tab:grids}
\begin{tabular}{llc}
\toprule
method & grid & selected \\
\midrule
GD          & $\{10^{-2}, 3\times10^{-2}, 10^{-1}, 3\times10^{-1}\}$ & $10^{-2}$ \\
scalar-Adam & $\{10^{-3}, 3\times10^{-3}, 10^{-2}, 3\times10^{-2}\}$ & $10^{-3}$ \\
Muon        & $\{3\times10^{-3}, 10^{-2}, 3\times10^{-2}, 10^{-1}\}$ & $10^{-1}$ \\
Shampoo     & $\{3\times10^{-2}, 10^{-1}, 3\times10^{-1}, 1\}$ & $3\times10^{-2}$ \\
\midrule
Adam        & $\{10^{-3}, 10^{-2}, 3\times10^{-2}\}$ & $10^{-2}$ \\
RMSProp     & $\{10^{-4}, 3\times10^{-4}, 10^{-3}, 3\times10^{-3}, 10^{-2}\}$ & $10^{-4}$ \\
signum      & $\{10^{-4}, 3\times10^{-4}, 10^{-3}, 3\times10^{-3}\}$ & $3\times10^{-3}$ \\
Lion        & $\{10^{-4}, 3\times10^{-4}, 10^{-3}, 3\times10^{-3}\}$ & $3\times10^{-4}$ \\
Adafactor   & $\{3\times10^{-4}, 10^{-3}, 3\times10^{-3}, 10^{-2}\}$ & $3\times10^{-4}$ \\
\bottomrule
\end{tabular}
\end{table}
\paragraph{Adam-$p$ (\S\ref{sec:dial}).} The fixed-rate arm re-runs the full $p$-sweep
at each single rate of the shared grid $\{10^{-3}, 3\times10^{-3}, 10^{-2}, 3\times10^{-2},
10^{-1}\}$, same update rule, seeds, and interpolation bar
(\texttt{experiments/precond\_dial\_fixed\_lr.py}).

Both arms (best rate and fixed rate) use a $3\times10^4$-step budget
(\texttt{experiments/precond\_dial\_scalar\_check.py}), longer than the zoo's $2\times10^4$.

The CPU code follows Algorithm~\ref{alg:dial}. The GPU port behind the $n{=}128$ dial rows omits the outer $\epsilon$ for $0<p<1$ and divides by $\bar s + \epsilon$ rather than $\bar s + 2\epsilon$ at $p{=}0$. Neither difference reaches the digits we report.

\paragraph{Attention task (\S\ref{sec:attention}).}
\begin{spec}
\item \emph{Task.} Modular addition modulo 47 on input sequences of the form $(a, b, {=})$, with
40\% of the data used for training (fixed split).
\item \emph{Model.} A 2-layer transformer with 4 attention heads, model dimension
$d_{\mathrm{model}} = 64$, using RMSNorm \citep{zhang2019rmsnorm} in pre-normalization configuration
and SiLU-activated \citep{elfwing2018silu} MLPs of width 256. The model omits biases and does not
apply query--key normalization (qk-norm), so the gauge holds exactly.
\item \emph{Training.} Full-batch optimization over 1500 steps, executed deterministically on CPU
(\texttt{torch.use\_\allowbreak deterministic\_\allowbreak algorithms}).
\item \emph{Gauge.} An orthogonal transformation $A_h$ per head, derived via QR decomposition of a
Gaussian matrix, applied to the row-blocks of $W_Q$ and $W_K$.
\item \emph{Learning rates.} Adam uses $10^{-3}$; SGD with heavy-ball momentum ($0.9$, unnormalized,
as for Muon's momentum here) uses $0.5$; scalar-Adam is set to $3\times10^{-3}$, while Muon uses
$0.02$ on matrix parameters, with the standard hybrid applying Adam at $10^{-3}$ to the embeddings
and the output head. These are each family's usual default at this scale and were not selected by a
grid search.  Each twin pair trains at a
single rate, so the split inside a pair cannot be an artifact of the rate.
\end{spec}

Drift is the relative
$\ell_2$ logit distance on the validation set,
$\norm{Z - Z'}_F \big/ \tfrac12(\norm{Z}_F + \norm{Z'}_F)$, where $Z$ and $Z'$ are the validation
logit matrices of the two copies.

QK drift is computed the same way, but on $W_Q^\top W_K$ for each attention head rather than on the
output logits. Per-head ratios are averaged with equal weight, counting every head in every layer
once. The $56\%$ quoted in \S\ref{sec:attention} is that mean.

\paragraph{Run records.}
Tables~\ref{tab:zoo}, \ref{tab:ladder} and~\ref{tab:flowlimit} are generated directly from the run records. In particular, the zoo reads \texttt{optimizer\_zoo\_bias.jsonl}; the ladder, its dispersion and its holdout cut read \texttt{zoo\_n\{64,96,128,192,256\}.jsonl}; and Table~\ref{tab:flowlimit}'s matched-loss comparison reads \texttt{\{pavia,indianpines\}\_obscenter.jsonl} for gradient descent and \texttt{scalaradam\_lr\_\{pavia,indianpines\}\_obscenter.jsonl} for the learning-rate grid from $3\times10^{-3}$ down to $10^{-6}$, the sweep of \S\ref{sec:flowlimit}. The run records and the scripts that generate these tables are released with the code. All numbers in Tables~\ref{tab:realdata} and~\ref{tab:flowlimit}, for both scenes and both densities, were produced by a single experimental harness, running in \texttt{float64} precision on GPU and using the train-only selection rule (Appendix~\ref{app:c6}). The two training-loss levels quoted for these tables, $3\times10^{-5}$ and $\le10^{-5}$, are the same runs read at two depths of fit.

\paragraph{Dataset details.}
The hyperspectral image of the Indian Pines area of north-western Indiana was taken by the AVIRIS sensor and made available as an archive with 220 spectral bands \citep{baumgardner2015indianpines}. We use the widely-adopted corrected version with 200 bands that excludes the water-absorption bands \citep{grana2011hsiscenes}. The hyperspectral image of Pavia University was taken by the ROSIS sensor and is distributed by the same source; it has 103 spectral bands. For matrix completion, both datasets are reshaped so that each pixel forms a row. To construct our task, we draw 2000 rows uniformly at random from the whole spatial grid with a fixed seed. From the resulting matrix $X\in\R^{N\times B}$, with $N=2000$ rows and $B$ bands, we observe a random subset $\Omega\subseteq[N]\times[B]$ of entries and predict the rest. We report the root-mean-square error on the held-out entries $\Omega^c$, as mean and standard deviation over four seeds, each with its own observation mask and initialization.

We center each band on the mean of its observed entries, separately for every seed and density, so that no held-out entry enters the preprocessing. Bands with no observed entries remain unshifted (that is, are assumed to have zero mean). The reference intrinsic rank and the spectral tails are read off the complete matrix, and no method sees them.

We minimize the entrywise completion loss at factorization rank $48$, using the train-only selection rule of Appendix~\ref{app:c6} and initialization scale $10^{-2}$.  Table~\ref{tab:realdata} displays two data densities corresponding to $m/\mathrm{dof}_{24}$ ratios of $\approx1.15$ and $\approx1.9$, which are densities of $\approx0.15$ and $\approx0.25$ on Indian Pines and $\approx0.28$ and $\approx0.46$ on Pavia University. We count $\mathrm{dof}_{24}$ at rank $24$, where the singular-value spectra have settled. At that rank the singular-value tail carries about $0.10$ of the Frobenius norm of the Indian Pines data. This is an amplitude, like $\tau$, not an energy fraction. Checkpoints are compared at matched training-set loss, in the range $3\times10^{-3}$ to $10^{-5}$.  

Doubling the cap to rank $96$ widens the GD--Adam gap (Appendix~\ref{app:rankcap}), so the rank-$48$ numbers are conservative.

\paragraph{FlowAdam variants (Appendix~\ref{app:repair}).}
We use the clipping modes \texttt{percoord} (each entry clipped to $[-1,1]$) and \texttt{globalnorm} (the whole velocity rescaled to norm at most $c$) of \citet{singh2026flowadam}, with $3\times10^{4}$ steps (\texttt{experiments/flowadam\_upgrade.py}) and the zoo's three paired seeds. 

\paragraph{Discrepancies between tables.}
Values repeated across tables can differ slightly, because the learning-rate grids, seed counts and step limits differ. For example, the grids of Table~\ref{tab:lrcurves} extend below the zoo's selection grids, so GD and scalar-Adam read lower there than in Table~\ref{tab:zoo}. The dial uses one shared grid for every $p$, so its endpoints differ too: at $p{=}1$ it reads $0.570$ (at $3\times10^{-3}$, a rate the zoo's Adam grid lacks) against the zoo's $0.5734$, and at $p{=}0$ with $n{=}128$ it reads $0.40$ against the ladder's $0.345$.

\paragraph{Precision and hardware.}
The matrix-sensing runs of \S\ref{sec:zoo} with their controls, and the attention twins of Table~\ref{tab:attention}, run on CPU in PyTorch's default float32. The following run on GPU in float64: the H100 replication of Appendix~\ref{app:c9}, the ten-seed tail sweep of Table~\ref{tab:phase}, Appendix~\ref{app:c5} at-scale twins, the LoRA twins of Table~\ref{tab:lora}, the merges of Appendix~\ref{app:mergedetails}, the dial at scale (Appendix~\ref{app:dialdetails}), and the \S\ref{sec:realdata} hyperspectral experiments. Table~\ref{tab:gaugeorbit} is CPU-based float64. For the GPU runs we enabled PyTorch deterministic kernels. The GPT-2 fine-tunes of Appendix~\ref{app:gpt2ft} run on GPU in float32, with the gauge matrix built and applied in float64. CPU runs used an Apple M1, and each GPU job a single NVIDIA H100.

\section{The control battery}
\label{app:controls}

We re-ran the key results of the main paper under a battery of controls.

\paragraph{Summary.} Class ordering held under schedule changes and the interpolation bar (Appendix~\ref{app:c1}), across learning rates and selection rules (Appendix~\ref{app:c2}) and across initialization scales (Table~\ref{tab:initscale}), and the attention twins held across seeds, model depths and noise levels (Appendix~\ref{app:c5}).

\begin{table}[tbp]
\centering
\caption{\textbf{The two optimizer groups separate on recovery error.}
$40\times40$ rank-$3$ sensing with $m=2\times$dof and three paired seeds. All runs interpolate. Recovery is mean $\pm$ standard deviation and other entries are means. Bal is $\norm{U^\top U-V^\top V}_F$, and ``cosine'' marks learning-rate decay. Anchors cite the first report of each bias: GD \citep{gunasekar2017implicit}, Adam \citep{wilson2017marginal}.}
\label{tab:zoo}
\vspace{2pt}
\small
\setlength{\tabcolsep}{3.6pt}
\begin{tabular}{llccccc}
\toprule
 & method & recovery $\downarrow$ & erank & bal & train loss & notes \\
\midrule
 & nuclear-norm ref. & \textit{0.0335} & \textit{3.30} & --- & --- & convex reference, App.~\ref{app:nuclear} \\
\midrule
\multirow{4}{*}{\rotatebox{90}{\scriptsize equivariant}}
 & Muon (cosine)        & $\mathbf{0.0000}\pm0.0000$ & 2.95 & 1.65 & $5.7\times10^{-8}$ & near-exact, $6.8\times10^{-6}$ \\
 & GD                   & $0.1312\pm0.0374$ & 4.51 & 0.06 & $6.5\times10^{-8}$ & Gunasekar anchor \\
 & scalar-Adam ($p{=}0$)& $0.2010\pm0.0381$ & 5.43 & 0.14 & $3.5\times10^{-8}$ & equivariant; flow proxy \\
 & Shampoo              & $0.2856\pm0.0393$ & 6.95 & 1.47 & $4.5\times10^{-8}$ & \\
\midrule
\multirow{5}{*}{\rotatebox{90}{\scriptsize coord.-wise}}
 & Lion (cosine)        & $0.4248\pm0.0849$ & 10.58 & 7.35 & $8.0\times10^{-8}$ & \\
 & signum (cosine)      & $0.4454\pm0.0695$ & 7.83 & 1.79 & $6.5\times10^{-8}$ & low-rank-but-wrong \\
 & RMSProp (cosine)     & $0.5266\pm0.0828$ & 12.80 & 4.74 & $6.0\times10^{-8}$ & needs decay, App.~\ref{app:c7} \\
 & Adafactor            & $0.5430\pm0.0831$ & 10.72 & 3.40 & $5.6\times10^{-8}$ & factored diag.\ breaks it \\
 & Adam                 & $0.5734\pm0.0554$ & 14.37 & 5.37 & $1.2\times10^{-11}$ & Wilson anchor \\
\bottomrule
\end{tabular}
\end{table}

\subsection{Schedules and the interpolation bar}
\label{app:c1}

Algorithms with a constant update norm (e.g., Muon, signum, Lion) require LR decay in order to get to interpolation. We check that Muon's advantage is not just an effect of its decay, with two arms:
\begin{itemize}
\item In the first, we run all methods with both constant and cosine schedules, stopping at interpolation. It tests whether Muon still beats Adam and Shampoo under the same decayed schedule.
\item In the second, we run all methods for the entirety of $2\times10^4$ steps using cosine decay . It tests whether other methods match Muon when given equally small late steps.
\end{itemize}

\begin{table}[tbp]
\centering\small
\caption{\textbf{The two optimizer groups still separate under shared cosine decay.}
Three-seed means at each method's best-recovery rate. Runs stop at training loss below $10^{-7}$, and steps counts updates to that point.}
\label{tab:cosine9}
\begin{tabular}{llccccc}
\toprule
method & class & recovery & $\er$ & selected lr & steps & train loss \\
\midrule
Muon        & equivariant     & 0.0000 & 2.95  & 0.1    & 19800 & $5.7\times10^{-8}$ \\
GD          & equivariant     & 0.1312 & 4.51  & 0.01   & 3000  & $8.0\times10^{-8}$ \\
scalar-Adam & equivariant     & 0.2010 & 5.43  & 0.001  & 3200  & $5.4\times10^{-8}$ \\
Shampoo     & equivariant     & 0.2856 & 6.95  & 0.03   & 2533  & $5.4\times10^{-8}$ \\
\midrule
Lion        & coordinate-wise & 0.4248 & 10.58 & 0.0003 & 19000 & $8.0\times10^{-8}$ \\
signum      & coordinate-wise & 0.4454 & 7.83  & 0.003  & 19600 & $6.5\times10^{-8}$ \\
RMSProp     & coordinate-wise & 0.5266 & 12.80 & 0.0001 & 18133 & $6.0\times10^{-8}$ \\
Adafactor   & coordinate-wise & 0.5430 & 10.72 & 0.0003 & 2667  & $4.1\times10^{-8}$ \\
Adam        & coordinate-wise & 0.5734 & 14.37 & 0.01   & 400   & $1.3\times10^{-11}$ \\
\bottomrule
\end{tabular}
\end{table}

Muon does not interpolate at a constant rate, so the constant-versus-cosine comparison covers GD, Adam, scalar-Adam and Shampoo, and the ordering is the same in both arms. Table~\ref{tab:cosine9} gives all nine methods under shared cosine decay.

We also reran the phase diagram experiment from \S\ref{sec:phase} using cosine LR decay (on the three-seed CPU sweep). GD, Adam and Shampoo recover within noise of their constant-rate values at all values of $\tau$, and Muon shifts by no more than $0.02$. The ordering is unchanged, with Adam worst at all six values of $\tau$, though at the largest all methods are within seed variance.

\paragraph{RMSProp.}\label{app:c7} Without requiring methods to achieve interpolation (\S\ref{sec:setup}), it's possible to score RMSProp on a run that never fits the data. At a constant rate, RMSProp stalls at training loss $2.9\times10^{-4}$, and still at $2.8\times10^{-4}$ after $6\times10^4$ steps ($3\times$ the budget), because of the $\epsilon$-bounded denominator of its update rule. At that plateau, RMSProp's recovery is $0.5270$. With decay, RMSProp does achieve interpolation and falls within the coordinate-wise cluster, at $0.5266$ (Table~\ref{tab:cosine9}). So decay lets RMSProp interpolate without changing where it lands.

\paragraph{Signum.}  Judged at its constant-LR training-loss plateau (training loss $5\times10^{-2}$, where it stalls without interpolating), its score is $0.009$, which would falsely suggest that sign descent keeps the low-rank bias. With decay it interpolates, and its $0.445$ shows the bias is lost at this budget.

\subsection{Learning-rate curves and selection rules}
\label{app:c2}\label{app:c3}

Table~\ref{tab:lrcurves} shows recovery across step sizes for five reference methods. Best recovery is an outcome-based criterion, so we use it in \S\ref{sec:zoo} only to show each method at its best; the conclusion rests on the whole curve. 

We see that every equivariant method gets error below $0.3$ at some learning rate in the grid, while none of the coordinate-wise optimizers can at any rate on its grid within this horizon. 

 Gradient descent and scalar-Adam do better with smaller but still interpolating learning rates, as gradient-flow theory predicts for GD (for scalar-Adam this is empirical; see the scope notes after Lemma~\ref{lem:momentum}). Muon is near-exact at the larger rates and worse at smaller ones, consistent with equal-rate spectral growth \citep{kang2026uniform}, as well as with the effect of a fixed number of Newton--Schulz iterations near small singular values.

Table~\ref{tab:lrcurves} includes some smaller LRs than the selection procedure used, in order to probe the flow limit (the $\to$fl rates).  The remaining coordinate-wise optimizers (RMSProp, Lion, Adafactor, and signum) were swept over different LR grids, and are not reported in Table~\ref{tab:lrcurves}. They all belong to the poor-recovery ($\ge 0.42$) group under both criteria.

\begin{table}[tbp]
\centering\small
\caption{\textbf{Adam's recovery stays nearly flat across learning rates.}
Entries are mean recovery errors over three sensing seeds.
``$\to$fl'' marks the smallest tested interpolating rate. ``n/i'' means the run does not interpolate. ``div'' means it diverges.}
\label{tab:lrcurves}
\begin{tabular}{llccccc}
\toprule
method & & \multicolumn{5}{c}{learning-rate grid} \\
\midrule
\multirow{2}{*}{GD}
 & lr       & 0.003$^{\to\mathrm{fl}}$ & 0.01 & 0.03 & 0.1 & \\
 & recovery & 0.113 & 0.131 & 0.437 & div & \\
\midrule
\multirow{2}{*}{Adam}
 & lr       & 0.0003$^{\to\mathrm{fl}}$ & 0.001 & 0.003 & 0.01 & 0.03 \\
 & recovery & 0.560 & 0.576 & 0.570 & 0.573 & 0.581 \\
\midrule
\multirow{2}{*}{scalar-Adam}
 & lr       & 0.0003$^{\to\mathrm{fl}}$ & 0.001 & 0.003 & 0.01 & \\
 & recovery & 0.165 & 0.201 & 0.256 & 0.357 & \\
\midrule
\multirow{2}{*}{Muon (cosine)}
 & lr       & 0.003 & 0.01 & 0.03 & 0.1 & \\
 & recovery & 0.915 & 0.748 & 0.000 & 0.000 & \\
\midrule
\multirow{2}{*}{Shampoo}
 & lr       & 0.01 & 0.03 & 0.1 & 0.3 & \\
 & recovery & (0.881 n/i) & 0.286 & div & div & \\
\bottomrule
\end{tabular}
\end{table}

\paragraph{Selection-rule stress test.} Under a shared cosine schedule the best-recovery rule sorts all nine methods correctly, and the fastest-interpolating rule sorts $8/9$: Muon would be marked as failing with recovery $0.91$.

``Fastest-interpolating'' means reaching $10^{-7}$ in the fewest iterations within $2\times10^4$. For Muon this is the smallest rate in the grid, $0.003$ (Table~\ref{tab:lrcurves}).

\paragraph{Initialization scale.} Table~\ref{tab:initscale} sweeps the initialization scale for the five reference methods. This is the pattern Proposition~\ref{prop:boundary} predicts. GD's greedy schedule depends on the initialization scale, and an equal-rate schedule like Muon's does not.

\begin{table}[tbp]
\centering\small
\caption{\textbf{Larger initialization worsens GD and scalar-Adam recovery, while Muon remains near-exact.}
Entries are mean recovery errors over three sensing seeds.}
\label{tab:initscale}
\begin{tabular}{lccccc}
\toprule
init scale & GD & Adam & scalar-Adam & Muon & Shampoo \\
\midrule
$10^{-3}$          & 0.131 & 0.573 & 0.201 & 0.000 & 0.286 \\
$3\times10^{-3}$   & 0.203 & 0.575 & 0.260 & 0.000 & 0.352 \\
$10^{-2}$          & 0.318 & 0.584 & 0.358 & 0.000 & 0.575 \\
\bottomrule
\end{tabular}
\end{table}

\subsection{Attention twins across seeds, depths and noise levels}
\label{app:c5}

We run Adam on six gauge-twin pairs (two initialization seeds $\times$ three rotation draws) in the CPU setting of \S\ref{sec:attention}. At step $0$, the twins are identical up to $\sim2\times10^{-7}$ drift, the rounding floor, but by step $1$ they have drifted by $3.4$ to $5.4\times10^{-3}$.  At the end of training, drift is $0.69$--$0.80$ but validation accuracy is comparable ($|\Delta\mathrm{acc}| \le 0.0023$). 

For SGD and scalar-Adam, over two (seed, draw) pairs, the maximum final drifts are $3.0\times10^{-4}$ and $3.5\times10^{-6}$.

Adam's noise twins, with perturbation sizes from $10^{-7}$ to $10^{-5}$, end at drifts from $1.6\times10^{-5}$ to $1.0\times10^{-3}$. Adam's smallest final gauge drift is $662\times$ the final drift of the noise twin $100\times$ stronger than the $10^{-7}$ one.

For the case of Muon with seed $42$, the gauge and noise twins' drifts differ by $0.04$ decades on average ($|\log_{10}(\text{gauge}/\text{noise})|$ after step $100$). Muon's gauge twin thus behaves like a noise twin.

\rev{\paragraph{Effective rank of the learned query-key change.} These are the runs of \S\ref{sec:attention} (same rates, seeds $42$, $123$ and $456$), with effective rank taken on the change from initialization and averaged over heads and seeds. At $1500$ steps each mean varies by at most $0.9$ across seeds. Adam and Muon end at a higher training loss ($3.4\times10^{-4}$ against $4\times10^{-5}$), so we train them longer at the same rates, up to $8000$ steps, and also compare all four where each first reaches $4\times10^{-5}$. Validation accuracy is about $0.40$ for all four.
\begin{center}\small
\begin{tabular}{lcccc}
\toprule
 & SGD & scalar-Adam & Adam & Muon \\
\midrule
after $1500$ steps & $12.3$ & $14.2$ & $17.1$ & $17.8$ \\
at training loss $4\times10^{-5}$ & $12.3$ & $14.2$ & $17.3$ & $14.3$ \\
\bottomrule
\end{tabular}
\end{center}}

\begin{table}[tbp]
\centering
\caption{\textbf{Adam's rotated twin separates at the first step, far above its noise control.}
Entries are relative logit distances on validation inputs for modular addition, computed in CPU \texttt{float32}.
$A{=}I$ is an identical rerun, and step $0$ shows the rounding floor.}
\label{tab:attention}
\vspace{2pt}
\small
\begin{tabular}{llcccc}
\toprule
optimizer & twin type & step 0 & step 1 & step 100 & step 1500 (final) \\
\midrule
Adam & gauge   & $1.8\times10^{-7}$ & $\mathbf{3.6\times10^{-3}}$ & $6.5\times10^{-1}$ & $7.7\times10^{-1}$ \\
Adam & $A{=}I$ & $0$ & $0$ & $0$ & $0$ \\
Adam & noise ($10^{-7}$) & $2.6\times10^{-7}$ & $2.9\times10^{-7}$ & $1.7\times10^{-5}$ & $1.6\times10^{-5}$ \\
SGD  & gauge   & $1.8\times10^{-7}$ & $2.9\times10^{-7}$ & $3.7\times10^{-5}$ & $2.0\times10^{-5}$ \\
scalar-Adam & gauge & $1.8\times10^{-7}$ & $1.6\times10^{-7}$ & $4.1\times10^{-6}$ & $4.3\times10^{-6}$ \\
Muon & gauge   & $1.8\times10^{-7}$ & $2.2\times10^{-7}$ & $2.3\times10^{-2}$ & $8.5\times10^{-1}$ \\
Muon & noise ($10^{-7}$) & $2.6\times10^{-7}$ & $2.9\times10^{-7}$ & $4.0\times10^{-2}$ & $8.7\times10^{-1}$ \\
\bottomrule
\end{tabular}
\end{table}

\paragraph{Scaling the twin diagnostic (GPU, \texttt{float64}).} Table~\ref{tab:attnscale} covers larger models and text. For the 2L/$d{=}64$ and 4L/$d{=}128$ problems, the equivariant twins' final values are still negligible, with scalar-Adam values $6.6\times10^{-15}$ and $6.0\times10^{-12}$ and SGD values $2.5\times10^{-9}$ and $1.6\times10^{-10}$. The ratio of Adam's smallest final gauge-twin value to the stronger noise twin's final value exceeds $500\times$ in the $d{=}64$ case. On text, Adam's two twins end at nearly identical validation losses ($1.579$ against $1.585$).

At 6L/$d{=}256$, trained to convergence, the gauge twins of scalar-Adam and Muon end about as far apart as their noise twins ($1.8$ against $1.7$, and $0.58$ against $0.83$), and on text scalar-Adam's gauge and noise twins both end at $0.44$.

At step $1$ (Table~\ref{tab:attnscale}), Adam's gauge twins differ by $10^{-3}$ to $10^{-2}$ in every configuration, while the other methods sit at machine precision, $12$--$13$ orders of magnitude lower. The step-$1$ ratio ranges from $39\times$ to $158\times$.

\begin{table}[tbp]
\centering\small
\caption{\textbf{Adam's first-step split persists in larger models and on text.}
Relative logit distances in GPU \texttt{float64}, over six Adam initialization--rotation pairs and two per equivariant method; text twins share minibatches.
``eq., step 1'' is the largest over SGD, scalar-Adam and Muon, and ``ratio'' divides Adam's smallest step-1 gauge distance by that of a $100\times$ stronger noise twin.}
\label{tab:attnscale}
\begin{tabular}{lcccc}
\toprule
config & Adam gauge, step 1 & eq., step 1 & Adam, final & ratio (step 1) \\
\midrule
mod-47, 2L, $d{=}64$          & $4.0$--$5.7\times10^{-3}$ & $5.4\times10^{-16}$ & $0.69$--$0.78$ & $93\times$ \\
mod-97, 4L, $d{=}128$         & $8.2$--$11\times10^{-3}$  & $7.3\times10^{-16}$ & $0.66$--$0.75$ & $158\times$ \\
mod-97, 6L, $d{=}256$         & $6.5$--$7.2\times10^{-3}$ & $1.1\times10^{-15}$ & $0.61$--$0.64$ & $94\times$ \\
text (char-LM), 6L, $d{=}256$ & $2.2$--$2.4\times10^{-3}$ & $8.0\times10^{-16}$ & $0.36$--$0.37$ & $39\times$ \\
\bottomrule
\end{tabular}
\end{table}

\subsection{Learning-rate selection on hyperspectral data}
\label{app:c6}

Learning rates in \S\ref{sec:realdata} are selected by a train-only rule: lowest training loss within a budget of $3\times10^4$ steps, ties broken by fewer steps. This keeps held-out data out of rate selection, so that one RMSE is not used both to pick the rate and to report it. Two of the four evaluation seeds are also used for rate selection, and the other two are independent of it. On Indian Pines at density $0.15$, gradient descent's train-only rate is $30$, the same rate held-out error would pick. Only this rate reaches a training loss of $10^{-5}$ within budget. Rates $10$, $3$ and $1$ stop at $3\times10^{-5}$, $3\times10^{-4}$ and $3\times10^{-3}$.

Under this protocol, at every rate on Adam's grid (held-out RMSE/effective rank of $0.0174/17$, $0.0201/20$ and $0.0263/28$), the selected GD run achieves both a lower held-out error and a lower effective rank than Adam.

\begin{table}[tbp]
\centering
\caption{\textbf{GD has lower held-out error at these training-selected rates.}
Entries are held-out root-mean-square errors (RMSE, $\times10^{-2}$), shown as mean $\pm$ standard deviation over four seeds.
Training loss is matched at $\le10^{-5}$.
$m/\mathrm{dof}_{24}$ counts observations per degree of freedom at rank $24$. Parentheses give Muon's effective rank (erank).}
\label{tab:realdata}
\footnotesize
\begin{tabular}{llccc}
\toprule
dataset & $m/\mathrm{dof}_{24}$ & GD & Adam & Muon (erank) \\
\midrule
\multirow{2}{*}{Indian Pines} & $1.15$ & $1.514 \pm 0.012$ & $2.628 \pm 0.045$ & $3.389 \pm 0.142$ (36) \\
                              & $1.9$  & $1.262 \pm 0.013$ & $1.740 \pm 0.038$ & $3.209 \pm 0.129$ (35) \\
\multirow{2}{*}{Pavia Univ.}  & $1.15$ & $0.867 \pm 0.027$ & $1.507 \pm 0.115$ & $1.371 \pm 0.079$ (10) \\
                              & $1.9$  & $0.631 \pm 0.010$ & $0.813 \pm 0.026$ & $0.788 \pm 0.033$ \phantom{0}(8) \\
\bottomrule
\end{tabular}
\end{table}

\paragraph{Best-rate selection on an extended grid.} We extend Adam's grid three orders of magnitude further down and pick each method's rate by held-out RMSE. The seed-paired differences in held-out error are below (positive means GD ahead; $95\%$ $t$-intervals, $3$ degrees of freedom). One cell is resolved in each direction, and the other two are not.
\begin{center}\small
\begin{tabular}{lcc}
\toprule
 & $m/\mathrm{dof}_{24}{=}1.15$ & $m/\mathrm{dof}_{24}{=}1.9$ \\
\midrule
Indian Pines & $-3.1\%$ $[-4.6, -1.7]$ & $+0.2\%$ $[-0.8, +1.1]$ \\
Pavia Univ.  & $+2.9\%$ $[-1.7, +7.6]$ & $+7.6\%$ $[+5.4, +9.9]$ \\
\bottomrule
\end{tabular}
\end{center}
So we do not claim from Table~\ref{tab:realdata} that moving from Adam to GD reduces error. GD's lower effective rank, in contrast, holds at every learning rate tested. At the best-performing learning rate for each method, Adam achieves a higher effective rank in all four cells, with differences of $+0.15$, $+0.44$, $+0.60$ and $+0.63$ from the first Indian Pines cell through the last Pavia cell, statistically resolvable in three of them.

Table~\ref{tab:limitrank} compares Adam and scalar-Adam at common small rates. At $\eta=10^{-6}$ the comparison resolves all four problem cells, and the two rules' rank offsets from GD differ by a factor of roughly $65$ to $1300$.

\paragraph{Edge of stability.} Large rates can help GD through edge-of-stability effects \citep{cohen2021gradient}. To rule this out, we rerun GD at the smaller rate $3$ on four seeds. As a function of training loss, the held-out RMSE of GD(lr${=}3$) matches that of GD(lr${=}10$) to within $10^{-5}$. GD(lr${=}3$) has a $12\%$ advantage over Adam at the deepest losses it achieves on all four seeds, that is, training loss $3\times10^{-4}$ or lower. GD's advantage is $44\%$ for the deeper losses reached by higher learning rates. So, at these rates and this density, the learning rate only sets how far along this trajectory GD gets within the step budget.

\paragraph{Step counts of GD and scalar-Adam.} At the smallest learning rate tested in \S\ref{sec:flowlimit}, all sixteen pairs of GD and scalar-Adam runs agree to $10^{-5}$ on held-out error and $0.02$ on effective rank, but GD needs only $\sim10^4$ steps to get there and scalar-Adam $6\times10^4$. That difference in step counts is the kind of time change Theorem~\ref{thm:transfer} describes. For scalar-Adam this agreement is empirical (see the scope notes after Lemma~\ref{lem:momentum}).

\paragraph{Muon on Indian Pines.} Muon trails Adam and GD on Indian Pines (Table~\ref{tab:realdata}), where its effective rank stays in the mid-$30$s at both densities. The performance deficit of Muon on Indian Pines could reflect Muon's weakness on spectral tails (\S\ref{sec:phasediag}) or a step budget that is too short. The conclusions of \S\ref{sec:realdata} use only GD and Adam, so are unaffected.

\subsection{Results on four scenes}
\label{app:tailtest}

With only two scenes, tail size and sampling density vary together, so their effects cannot be told apart. We therefore add two more hyperspectral images \citep{grana2011hsiscenes}, Salinas and Botswana, whose tails fall below and between those of Pavia and Indian Pines (Table~\ref{tab:tailtest}).

\paragraph{Sampling densities.} We count degrees of freedom at each scene's own target effective rank $r^*$ (Table~\ref{tab:tailtest}) instead of at rank $24$, so the Indian Pines and Pavia cells here sit at different densities from those of Table~\ref{tab:realdata}. Densities are set to give $m/\mathrm{dof}{=}1.15$ and $m/\mathrm{dof}{=}1.9$ on every scene. Fit rank is held constant at $48$ on all four scenes. Tail amplitude is measured beyond the fit rank, and the scenes' tail ordering is the same beyond rank $8$, $16$, $24$ and $48$.

\paragraph{Held-out error across scenes.} The best-rate held-out gap does not follow tail size: Salinas, which has the thinnest tail, has Adam's largest advantage, $-11.0\%$ at $m/\mathrm{dof}{=}1.9$.  The sweep of \S\ref{sec:phasediag} varies $\tau$ with all else fixed, which these scenes cannot do.

\begin{table}[tbp]
\centering
\caption{\textbf{Four hyperspectral scenes: held-out gap and small-step rank offsets.}
$\tau$ is tail amplitude beyond the fit rank $48$, and $m/\mathrm{dof}$ is counted at the target's effective rank $r^*$. Gap is the paired held-out RMSE difference at each method's best rate (positive favors GD; $95\%$ $t$-intervals over four seeds), and the last two columns subtract GD's effective rank at $\eta=10^{-6}$.}
\label{tab:tailtest}
\vspace{2pt}
\small\setlength{\tabcolsep}{5pt}
\begin{tabular}{lcccccc}
\toprule
dataset & $\tau$ & $r^*$ & $m/\mathrm{dof}$ & gap (\%) & scalar-Adam & Adam \\
\midrule
Salinas       & $0.0073$  & $7$   & $1.15$ & $+2.7$ {\scriptsize$[-3.4,+8.9]$} & $+0.046$ & $+7.641$ \\
              &           &       & $1.9$ & $-11.0$ {\scriptsize$[-14.7,-7.3]$} & $+0.040$ & $+3.903$ \\
Pavia Univ.   & $0.0094$  & $9$   & $1.15$ & $-6.5$ {\scriptsize$[-15.0,+2.0]$} & $+0.023$ & $+2.788$ \\
              &           &       & $1.9$ & $-7.0$ {\scriptsize$[-10.5,-3.5]$} & $+0.007$ & $+1.978$ \\
Botswana      & $0.0280$  & $15$  & $1.15$ & $-6.5$ {\scriptsize$[-10.1,-3.0]$} & $+0.047$ & $-2.250$ \\
              &           &       & $1.9$ & $-0.5$ {\scriptsize$[-1.7,+0.7]$} & $+0.022$ & $-1.163$ \\
Indian Pines  & $0.0515$  & $34$  & $1.15$ & $-0.4$ {\scriptsize$[-1.4,+0.5]$} & $+0.010$ & $+0.810$ \\
              &           &       & $1.9$ & $+0.2$ {\scriptsize$[-0.4,+0.8]$} & $+0.004$ & $+0.827$ \\
\bottomrule
\end{tabular}
\end{table}

\paragraph{Small-step statistics.} Scalar-Adam stays within $0.05$ effective ranks of gradient descent on all eight cells, while Adam is $0.8$--$7.6$ effective ranks away, above gradient descent's effective rank on Salinas, Pavia and Indian Pines and below gradient descent's on Botswana.  Unsigned matrix distance also separates Adam from gradient descent on all cells (Table~\ref{tab:pdistfour}).

\subsection{Doubling the rank cap}
\label{app:rankcap}

The rank cap of $48$ is above the reference's intrinsic rank of $24$, but Muon reaches effective rank in the mid-$30$s on Indian Pines, so the cap could be doing some of the selecting. We double the rank cap to $96$ on one cell per scene (the lower density), leaving everything else unchanged.  Adam's effective rank grows much more than GD's, and the GD--Adam gap widens.

\paragraph{Indian Pines.} GD's RMSE changes slightly, from $1.514\times10^{-2}$ to $1.579\times10^{-2}$, and its effective rank rises from $13.6$ to $20.3$. Adam's rises from $27.7$ to $58.3$. The gap increases from $+42.4\%$ to $+56.5\%$, and GD still wins across all four seeds. 

\paragraph{Pavia.} Adam's effective rank increases from $10.7$ to $24.9$. GD's increases from $6.3$ to $8.3$, but it still outperforms Adam, with the gap increasing from $+42.4\%$ to $+64.1\%$ across four seeds. GD's RMSE increases from $0.867\times10^{-2}$ to $0.971\times10^{-2}$.

\paragraph{Muon.} On Pavia at rank $96$ the train-only rule picks lr $0.03$, with held-out RMSE $11.2\times10^{-2}$ at effective rank $93.6$, compared to $1.371\times10^{-2}$ at rank $48$; at lr $0.1$ it reads $0.943\times10^{-2}$ at effective rank $9.8$. On Indian Pines, its RMSE goes from $3.389\times10^{-2}$ at rank $48$ to $3.068\times10^{-2}$ at rank $96$, and its effective rank from $35.9$ to $45.0$.

\subsection{H100 replication}
\label{app:c9}

To check that the separation at $n{=}40$ is not an artifact of rate selection and persists at larger sizes, we run a larger experiment on H100 GPUs. We consider sensing at $n \in \{64, 96, 128, 192, 256\}$, planted rank $3$, $m = 2\,\mathrm{dof}$, and use \texttt{float64} arithmetic. For each setting, we use $10$ random seeds and train for $4\times10^4$ steps. Following the best-recovery rule tested in Appendix~\ref{app:c3}, three seeds are used to select the learning rate, and we then rerun at this rate on the remaining seven ``holdout'' seeds.

At every size, the worst equivariant mean (GD, scalar-Adam, Muon) is below the best coordinate-wise mean (Adam, RMSProp, Lion, Adafactor), both on the holdout seeds and on all ten seeds.

\paragraph{Per-seed and per-run separation.} Per seed, the worst equivariant optimizer beats the best coordinate-wise one on $10/10$ seeds at every size except $n{=}96$, where it does on $8/10$. Per run, the groups separate except at $n{=}96$, where one coordinate-wise run reaches $0.144$ against a worst equivariant run of $0.355$. Both exceptions come from RMSProp's instability at its chosen rate.

\begin{table}[tbp]
\centering\footnotesize
\setlength{\tabcolsep}{4.2pt}
\caption{\textbf{GD, scalar-Adam, and Muon beat Adam, RMSProp, Lion, and Adafactor at every tested size.}
Entries are ten-seed mean recovery errors for rank-$3$ sensing with a $4\times10^4$-step budget on H100 \texttt{float64}; sc-Adam is scalar-Adam.}
\label{tab:ladder}
\begin{tabular}{r cccc ccccc}
\toprule
 & \multicolumn{4}{c}{equivariant} & \multicolumn{5}{c}{coordinate-wise} \\
\cmidrule(lr){2-5}\cmidrule(lr){6-10}
$n$ & GD & sc-Adam & Muon & Shampoo$^\dagger$ & Adam & RMSProp & Lion & Adafactor & signum$^\ddagger$ \\
\midrule
64  & 0.143 & 0.195 & \textbf{0.000} & 0.646 & 0.574 & 0.596 & 0.483 & 0.612 & 0.000 \\
96  & 0.202 & 0.279 & \textbf{0.000} & 0.595 & 0.609 & 0.665$^\S$ & 0.548 & 0.687 & 0.000 \\
128 & 0.248 & 0.345 & \textbf{0.000} & 0.659 & 0.643 & 0.617 & 0.602 & 0.745 & 0.000 \\
192 & 0.346 & 0.468 & \textbf{0.000} & 0.838 & 0.696 & 0.669 & 0.760 & 0.817 & 0.000 \\
256 & 0.410 & 0.548 & \textbf{0.000} & 0.918 & 0.741 & 0.708 & 0.895 & 0.867 & 0.000 \\
\bottomrule
\end{tabular}\\[2pt]
{\scriptsize $^\dagger$Shampoo uses fixed, untuned damping $\lambda{=}1$.
$^\ddagger$Signum recovers near-exactly with longer annealing (Appendix~\ref{app:c10}).
$^\S$One RMSProp seed at $n{=}96$ fails to interpolate and is excluded from the mean.}
\end{table}

\paragraph{Failed runs.} One seed of RMSProp at $n{=}96$ fails to converge, having a final training loss of $6.5$.  Including it would lower RMSProp's mean error from $0.665$ to $0.615$ and flatter it, so Table~\ref{tab:ladder} reports the average over the other $9$, which do interpolate. These range from recoveries of $0.14$ to $0.90$. All other results in Table~\ref{tab:ladder} interpolate $10/10$, with no clipped or replaced results.

\paragraph{Learning-rate grid scaling.} We scale the learning-rate grids with $n$ empirically, so that all methods keep interpolating. The grids for Muon and Shampoo are scaled by $\sqrt{40/n}$. The grids for Adam-family methods are scaled by $\sqrt{n/40}$, and the grid for GD by $40/n$. All grids reduce to the main-text grids at $n{=}40$. 

The target's Frobenius norm grows with $n$ while the initialization scale stays at $10^{-3}$, so with a fixed step count every method must travel farther. For gradient descent, this change in scale is accounted for in the gradient, so learning rates can scale down. In contrast, the magnitudes of the per-coordinate updates for Adam-family methods are independent of the scale of the gradient, so learning rates can scale up. 

\paragraph{Schedule control.} We also rerun the ladder with cosine schedules for every method. Muon, signum and Lion already decay in their baseline (Appendix~\ref{app:c1}). We find that, except for the unstable behavior of RMSProp at $n{=}64$ and $n{=}96$, the reported recoveries agree to three decimal places between the two schedules. On the mean over all $10$ seeds, RMSProp's recovery error rises from $0.596$ to $0.610$ at $n{=}64$ and from $0.615$ to $0.767$ at $n{=}96$ under the cosine schedule, and it stays in the coordinate-wise group.

\paragraph{Holdout seeds.} On the $7$ seeds not used for rate selection, entries of Table~\ref{tab:ladder} move by at most $0.018$, except RMSProp at $n{=}96$.

\paragraph{Ladder ratios.} The ratio of the Adam-family method error to GD error decreases with increasing $n$. This is because GD's error grows with $n$ at a fixed step budget, not because the Adam-family methods improve. This is not a flow-limit effect, since increasing $n$ at a fixed budget is not a small-step limit.

\paragraph{The ladder pipeline at $n{=}40$.} Rerun at the zoo's size, the H100 pipeline recovers the numbers of Table~\ref{tab:zoo} for Lion ($0.424$ here and $0.425$ there) and Adafactor ($0.543$ in both), while Adam stays flat under the extended grid and obtains $0.550$ at its finest interpolating rate. In contrast, as the step sizes are lowered, GD and scalar-Adam improve to $0.108$ and $0.143$, respectively, on their finest grids. GD stalls at an average training loss of $1.1\times10^{-7}$, just above the $10^{-7}$ bar, so we use its $0.108$ only as a step-size diagnostic. This matches the flow-limit behavior of Table~\ref{tab:lrcurves}.

\subsection{Annealed signum at long horizons}
\label{app:c10}

The zoo's two-cluster separation (\S\ref{sec:zoo}) is defined under a fixed computational budget.
One coordinate-wise method reaches near-exact recovery when training is extended:
signum, using cosine decay over a prolonged annealing schedule.

At the paper's standard budget of $2\times10^4$ steps, signum's recovery error is $0.445$,
in the coordinate-wise group. When extended to the ladder's $4\times10^4$-step horizon, recovery drops
to near-zero (below $10^{-5}$) for all $n \ge 64$ (across all $10$ random seeds). The transition
begins at $n{=}40$, where the outcome depends on platform and seed. On
H100 in \texttt{float64}, two of the three seeds recover, for a mean error of $0.013$. The original
CPU-based implementation ($4\times10^4$ steps, lr $3\times10^{-3}$) gives a mean of $0.054$.

Signum lies outside Theorem~\ref{thm:transfer}, which concerns flows with memoryless updates and a scalar preconditioner shared across the factors.

\paragraph{Rank of the update.} One candidate explanation is that signum's momentum
concentrates into a rank-one structure. If the momentum were an outer product, then since
$\mathrm{sign}(\sigma u v^\top) = \mathrm{sign}(u)\mathrm{sign}(v)^\top$ for $\sigma > 0$, the
signed update would itself be rank-one and aligned with the planted signal.

At $4\times10^4$ steps, the successful runs at lr $3\times10^{-3}$ show momentum singular-value ratios
$\sigma_1/\sigma_2 \approx 1.1$--$1.3$ at convergence. The
signed updates themselves have even lower values ($1.05$--$1.17$). The leading singular component carries $10\%$ to $24\%$ of the squared
Frobenius norm (\texttt{experiments/signum\_c9\_audit.py}), against the $100\%$ a genuinely
rank-one matrix would carry.

\paragraph{Comparison with the common-scalar case.} It differs in two ways.

First, it emerges from large learning rates and long annealing schedules rather than from a smooth
flow limit. At a learning rate of $3\times10^{-3}$, the average recovery over three random seeds
improves from $0.445 \to 0.054 \to 0.000$ as the cosine-annealing schedule is extended from
$2\times10^4 \to 4\times10^4 \to 6\times10^4$ steps. The decay period stretches with the horizon $T$. 
At a learning rate of $10^{-3}$, recovery remains
poor ($0.65$--$0.72$) throughout the schedule. A linear decay schedule has similar threshold
behavior ($0.479 \to 0.103 \to 0.000$ at analogous durations), suggesting this is
not specific to cosine annealing.

Second, this success breaks down when a spectral tail is added to the planted matrix. In
the tail sweep described in \S\ref{sec:phase}, signum at lr $3\times10^{-3}$ and $4\times10^4$
steps shows recovery values of $0.025$, $0.163$, $0.290$, $0.408$, $0.629$, and $0.787$ at
$\tau = 0$, $0.05$, $0.10$, $0.20$, $0.35$, and $0.50$ respectively (averaged over three seeds).
 Common-scalar methods follow the smoother degradation patterns seen in
Table~\ref{tab:phase}.

 In the configurations we
tested, recovery succeeds in a restricted regime: an exactly low-rank target, large steps, and a
long anneal. It shows that the common-scalar route is not the
only one, while remaining the one that is robust here and open to analysis.

\section{Additional results and derivations}
\label{app:relocated}

\subsection{The spectral-tail sweep: design and table}
\label{app:phasedesign}

\begin{table}[tbp]
\centering
\caption{\textbf{Muon gives way to GD as the target's spectral tail grows.}
Entries are mean recovery errors over ten sensing seeds in \texttt{float64}.
$\tau$ is tail amplitude, and $\tau^2$ is its fraction of target energy.
Bold marks the best error and ties. A tie means the means differ by less than the larger seed standard deviation.
}
\label{tab:phase}
\vspace{2pt}
\small
\begin{tabular}{cccccl}
\toprule
$\tau$ & GD & Adam & Muon & Shampoo & regime \\
\midrule
0.00 & 0.112 & 0.542 & \textbf{0.000} & 0.334 & Muon near-exact \\
0.05 & 0.150 & 0.543 & \textbf{0.095} & 0.343 & Muon \\
0.10 & 0.214 & 0.555 & \textbf{0.191} & 0.375 & Muon \\
0.20 & \textbf{0.351} & 0.597 & \textbf{0.354} & 0.449 & boundary ($\approx 4\%$ tail); Muon cedes \\
0.35 & \textbf{0.549} & 0.685 & 0.580 & 0.576 & GD \\
0.50 & \textbf{0.727} & 0.776 & 0.752 & \textbf{0.723} & tail regime; GD/Shampoo within noise \\
\bottomrule
\end{tabular}
\end{table}

The sweep asks how recovery changes as the target gains a spectral tail. The planted target mixes
a rank-$3$ signal $X_3$ with a dense tail $E_\perp$:
\[
X = \sqrt{1-\tau^2}\, X_3 + \tau E_\perp .
\]
The tail $E_\perp$ has rank $n-3$ and lives on the orthogonal complements of $X_3$'s column and
row spaces. Both parts have unit Frobenius norm. That support makes them orthogonal, so
$\norm{X}_F=1$, $\tau\in[0,1]$, and $\tau^2$ is exactly the fraction of the target's energy in the tail.

Sensing uses $m=2\,\mathrm{dof}$ counted at the rank-$3$ target, as elsewhere. All methods
interpolate. The sweep is a separate $10$-seed run on its own two-rate grid:
$\{10^{-2}, 3\times10^{-2}\}$ for GD, $\{3\times10^{-3}, 10^{-2}\}$ for Adam,
$\{10^{-3}, 3\times10^{-3}\}$ for scalar-Adam, and $\{3\times10^{-2}, 10^{-1}\}$ for Muon and
Shampoo.

That grid is held fixed across every $\tau$.
Within each $(\tau, \text{method})$ cell, the rate is selected by the rule of \S\ref{sec:zoo}. The chosen rate can move with $\tau$, so $\tau^\ast$ is where the best-rate curves cross, not where fixed-rate runs cross.

For each seed, we locate the crossover by linear interpolation of the Muon-minus-GD recovery gap between the two grid values of $\tau$ where it changes sign. Across the ten seeds it ranges from $0.108$ to $0.288$, with mean $0.200$ (standard deviation $0.067$) and $95\%$ $t$-interval $[0.15, 0.25]$ for the mean.

Seeds are $\{42,\allowbreak 123,\allowbreak 456,\allowbreak 789,\allowbreak 1011,\allowbreak 1213,\allowbreak 1415,\allowbreak 1617,\allowbreak 1819,
2021\}$ and $\tau$ runs over $\{0,\allowbreak 0.05,\allowbreak 0.1,\allowbreak 0.15,\allowbreak 0.2,\allowbreak 0.25,\allowbreak 0.3,\allowbreak 0.325,\allowbreak 0.35,\allowbreak 0.375,\allowbreak 0.4,\allowbreak 0.45,
0.5,\allowbreak 0.6\}$. Because this grid is coarser than the zoo's, its $\tau{=}0$ column reproduces
Table~\ref{tab:zoo}'s \emph{ordering} rather than its exact digits (for example, GD reads $0.112$
here against $0.131$ there).  Table~\ref{tab:phase} and Figure~\ref{fig:phase} show six of the fourteen $\tau$ values; at $\tau{=}0.6$, not shown, Muon ($0.846$) sits above Adam ($0.837$), within seed noise.

\subsection{The norm form of the criterion}
\label{app:normform}

\noindent\textbf{Corollary~\ref{cor:norm} (norm form of the criterion), stated in full.}\
Let $\norm{\cdot}$ be a norm on $\R^{n\times k}$ and let
$\mathcal{D}(G) = \argmax_{\norm{\Delta}\le1}\ip{G}{\Delta}$ be the corresponding steepest-descent
correspondence.
\emph{(i) Memoryless.} If $\norm{\cdot}$ is right-invariant, meaning $\norm{MQ} = \norm{M}$
for all $Q\in\OO(k)$, then $\mathcal{D}(GQ) = \mathcal{D}(G)\,Q$. A single-valued selection
$\Delta(G)\in\mathcal{D}(G)$ is therefore gauge-equivariant when the maximizer is unique, and
otherwise under an equivariant selection such as Muon's support-restricted polar factor. A
right-invariant norm is exactly one that sees $M$ only through the Gram matrix $MM^\top$.
\emph{(ii) Stateful.} If the norm carries optimizer state, $\norm{\cdot}_{s}$, and that state is
gauge-covariant in the sense of Definition~\ref{def:equivariance}, so that
$\norm{\Delta Q}_{\sigma_Q(s)} = \norm{\Delta}_{s}$, the same conclusion holds.

\begin{proof}
The whole argument is a change of variables on the unit ball.

\emph{(i)} Right-invariance says the ball $B = \{\Delta : \norm{\Delta}\le1\}$ is unchanged by
right multiplication, $BQ = B$ for every $Q\in\OO(k)$. So $\Delta \mapsto \Delta Q^\top$ maps $B$
onto itself. Substituting $\Delta = \Delta' Q$ in the maximization and using
$\ip{GQ}{\Delta' Q} = \tr(Q^\top G^\top \Delta' Q) = \tr(G^\top\Delta') = \ip{G}{\Delta'}$ gives
\[
\mathcal{D}(GQ) \;=\; \argmax_{\Delta\in B}\ip{GQ}{\Delta}
\;=\; \Bigl(\argmax_{\Delta'\in B}\ip{G}{\Delta'}\Bigr)Q \;=\; \mathcal{D}(G)\,Q .
\]
A single-valued selection therefore satisfies $\Delta(GQ) = \Delta(G)Q$ whenever the maximizer is
unique, and under any selection rule that commutes with the gauge otherwise. For the Gram statement,
$\norm{\cdot}$ is right-invariant exactly when it is constant on right-$\OO(k)$ orbits, and the
argument in the proof of Theorem~\ref{thm:structure} shows that two matrices of the same shape lie
on one such orbit precisely when their Grams agree. A right-invariant norm therefore factors through
$MM^\top$, and conversely any function of $MM^\top$ is right-invariant. The converse of the corollary fails: $\Phi(G)=-G$ is equivariant, with $H=-I$, but is never a steepest-descent direction for $G\neq0$, since then $\ip{G}{-G}<0$.

\emph{(ii)} The state-dependent hypothesis $\norm{\Delta Q}_{\sigma_Q(s)} = \norm{\Delta}_{s}$ says
the ball at the transformed state is the transformed ball, $B_{\sigma_Q(s)} = B_s Q$. The same
substitution then gives $\mathcal{D}_{\sigma_Q(s)}(GQ) = \mathcal{D}_{s}(G)\,Q$, which is
Definition~\ref{def:equivariance} for the pair (update, state).
\end{proof}

\paragraph{Shampoo and part (ii).}

The norm that generates Shampoo's step is $N_t(\Delta) = \norm{L_t^{1/8}\Delta R_t^{1/8}}_F$, not
the quarter powers the update itself carries. For
$N(\Delta) = \norm{A\Delta B}_F$ with $A,B$ symmetric positive definite, the change of variables
$Y = A\Delta B$ makes the steepest-descent direction proportional to $A^{-2}GB^{-2}$. The eighth
powers are therefore what return Shampoo's $L_t^{-1/4}GR_t^{-1/4}$. Damping replaces $L_t, R_t$ by
$L_t + \lambda I, R_t + \lambda I$ throughout.

That norm is \emph{not} right-invariant in $\Delta$ when the metric factors are held fixed, since
$Q$ does not commute past $R_t^{1/8}$. What rescues it is that the metric itself transforms with
the gauge. Under the gauge, $L_t$ is invariant while $R_t \mapsto Q^\top R_t Q$. The chain of
equalities then reads:
$\norm{L_t^{1/8}(\Delta Q)(Q^\top R_tQ)^{1/8}}_F = \norm{L_t^{1/8}\Delta R_t^{1/8}Q}_F =
\norm{L_t^{1/8}\Delta R_t^{1/8}}_F$. Shampoo is therefore an instance of part~(ii), not part~(i),
and only in its idealized form.

The exponent does not enter this argument.  Deployed variants that graft an entrywise Adam step size onto the direction fall outside
part~(ii). A cached full-matrix preconditioner stays inside it, provided both runs share the refresh
schedule and the cache is transformed with the gauge.

\paragraph{Scope of the norm correspondence.}

Three qualifications limit how far Corollary~\ref{cor:norm} goes.
\begin{enumerate}
\item The condition is right-invariance, not the more familiar unitary-invariance of a symmetric gauge function of the singular values.  For example, the weighted Frobenius norm $\norm{M}_D = \norm{DM}_F$ with fixed, non-scalar $D$ is right-invariant without being unitarily invariant.
\item The characterizations cited in \S\ref{sec:theory} refer to the unfactored parameter $\theta$ and the direction of the asymptotic iterates during separable classification. By contrast, Proposition~\ref{prop:adam-onestep} is about the exact algebraic identity for one step. The same coordinate geometry appears at the first step and at infinity, but neither result implies the other.
\item Muon is steepest descent under the spectral norm only for the exact polar factor. Finite Newton--Schulz only approximates that direction, but it is still exactly equivariant (Proposition~\ref{prop:equivariant}(3)).
\end{enumerate}

\subsection{Fixed instance across latent bases}
\label{app:gaugeorbit}

The twin experiments elsewhere in the paper compare an $\mathrm{O}(k)$ rotation against a $10^{-7}$ perturbation, so moving outside the symmetry and moving far happen together there. This experiment separates the two. On the sensing task of \S\ref{sec:setup}, with fixed data and fixed initialization, it runs Adam, gradient descent and scalar-Adam from $(U_0,V_0)$ in four kinds of arm: identity, signed-permutation $Q=PS$, dense orthogonal $Q\in\mathrm{O}(k)$, and noise. The dense arm uses four independent Haar-random rotations per instance.  

Adam is equivariant under signed permutations but not under general rotations (\S\ref{sec:setup}). Moreover, the $\ell_\infty$ ball is signed-permutation-invariant but not dense-rotation-invariant, so Corollary~\ref{cor:norm}, applied to that subgroup, makes the same prediction. So the criterion says in advance which basis changes Adam cannot see.

\paragraph{Implementation check.}
Numerical covariance requires Definition~\ref{def:equivariance}'s correspondence of parameters, not just of their product $UV^\top$. Over the first $40$ steps, Adam's maximum relative parameter error, $\norm{U_t^{(Q)}Q^\top - U_t}_F/\norm{U_t}_F$, was $3.4\times10^{-16}$ under $Q=PS$, and $0.91$ under dense $Q$'s. Gradient descent and scalar-Adam had worst-case relative error $<10^{-15}$ for both kinds of $Q$. This protocol is implemented in \texttt{experiments/equivariance\_balance\_probe.py}.

\paragraph{Results.}
In Table~\ref{tab:gaugeorbit}, the two gauge arms move the parameters by nearly the same amount ($1.41$ vs.\ $1.42$). Only the dense rotation changes Adam's fitted matrix, by $0.332$, about one third of its size. Adam is equivariant under signed permutations, so the other arm leaves $W$ unchanged up to floating-point precision.  The dense arms' $\mathrm{d}_W$ reads $0.94$ at step $1$, $0.48$ at step $10$ and $0.33$ at step $100$, while the signed-permutation arm stays at most $3.3\times10^{-16}$ at each of those steps. Proposition~\ref{prop:adam-onestep} gives the step-$1$ gap exactly.

\begin{table}[tbp]
\centering
\caption{\textbf{Dense rotations change Adam's fitted matrix, while signed permutations do not.}
Twelve sensing instances use four dense rotations each. All runs reach the interpolation threshold.
Rates are selected on identity runs and then fixed.
``in move'' is relative parameter displacement at initialization.
$\mathrm{d}_W=\norm{W^{(Q)}-W^{(I)}}_F/\norm{W^{(I)}}_F$ measures fitted-matrix distance.
$\mathrm{d}_{\mathcal A}=\norm{\mathcal A(\Delta W)}_2/\norm{\mathcal A(W^{(I)})}_2$ measures distance in observed measurements.
Dense rows average all $48$ runs, with minimum and maximum distances in brackets.
Its standard deviation is across twelve instances, using each instance's mean over rotations.}
\label{tab:gaugeorbit}
\small
\setlength{\tabcolsep}{4.4pt}
\begin{tabular}{llcccc}
\toprule
method & arm & in move & $\mathrm{d}_W$ & $\mathrm{d}_{\mathcal{A}}$ & recovery \\
\midrule
\multirow{4}{*}{Adam}
 & identity & $0$ & $0$ exactly & $0$ exactly & $0.5444 \pm 0.0468$ \\
 & signed perm. & $1.41$ & $5.0\times10^{-16}$ & $2.6\times10^{-16}$ & $0.5444 \pm 0.0468$ \\
 & dense rotation & $1.42$ & $\mathbf{0.332}$ {\scriptsize$[0.276,0.402]$} & $4.1\times10^{-6}$ & $0.5366 \pm 0.0540$ \\
 & noise ($10^{-7}$) & $10^{-7}$ & $3.8\times10^{-7}$ & $7.0\times10^{-12}$ & $0.5444 \pm 0.0468$ \\
\midrule
\multirow{4}{*}{GD}
 & identity & $0$ & $0$ exactly & $0$ exactly & $0.0960 \pm 0.0289$ \\
 & signed perm. & $1.41$ & $5.6\times10^{-16}$ & $2.6\times10^{-16}$ & $0.0960 \pm 0.0289$ \\
 & dense rotation & $1.42$ & $8.3\times10^{-16}$ & $2.8\times10^{-16}$ & $0.0960 \pm 0.0289$ \\
 & noise ($10^{-7}$) & $10^{-7}$ & $3.4\times10^{-9}$ & $1.8\times10^{-12}$ & $0.0960 \pm 0.0289$ \\
\midrule
\multirow{4}{*}{scalar-Adam}
 & identity & $0$ & $0$ exactly & $0$ exactly & $0.1770 \pm 0.0404$ \\
 & signed perm. & $1.41$ & $4.0\times10^{-16}$ & $2.8\times10^{-16}$ & $0.1770 \pm 0.0404$ \\
 & dense rotation & $1.42$ & $7.1\times10^{-16}$ & $3.1\times10^{-16}$ & $0.1770 \pm 0.0404$ \\
 & noise ($10^{-7}$) & $10^{-7}$ & $6.9\times10^{-9}$ & $1.4\times10^{-12}$ & $0.1770 \pm 0.0404$ \\
\bottomrule
\end{tabular}
\end{table}

   On this task Adam's recovery does move from gauge to gauge, but we do not offer that as an explanation of the gap in mean recovery seen in Figure~\ref{fig:zoo}. Within a given instance of four dense gauges, recovery varies with a standard deviation of $0.017$, while between instances it varies with a standard deviation of $0.047$.  On average, the paired difference in recovery from the identity basis is $-0.008\pm0.017$ over the twelve instances, which does not suggest a systematic difference in recovery ($t=-1.6$ on $11$ degrees of freedom).

The initialization is isotropic Gaussian, so its distribution is rotation-invariant. No basis can therefore be best for Adam on average over initializations.

\subsection{Merge experiment details}
\label{app:mergedetails}

\paragraph{Merge metric definition.}
Table~\ref{tab:lora} reports the held-out loss of the merged model, averaged in the product space, minus the average of the held-out loss of the two adapted models, in units of nats. We also report the relative distance between the twins' products $\Delta W = BA$ at step $1$, the first step at which they can differ, and at the end of fine-tuning. The distance is the symmetrized relative Frobenius distance
\[
d(\Delta W_1, \Delta W_2) \;=\;
\frac{\norm{\Delta W_1 - \Delta W_2}_F}{\tfrac12\bigl(\norm{\Delta W_1}_F + \norm{\Delta W_2}_F\bigr)},
\]
taken per adapted matrix and averaged over them, with $d \coloneqq 0$ when both norms vanish. The noise twin is the same optimizer with $A$ perturbed by $10^{-7}$ or $10^{-5}$ and no basis change. Noise-twin rows are single-seed, while the gauge and optimizer rows use three draws. All runs are rank $8$ in GPU \texttt{float64}.

\begin{table}[tbp]
\centering
\caption{\textbf{LoRA gauge twins.} Relative distance between adapter products $\Delta W=BA$
at step $1$ and after fine-tuning; merge cost in nats (negative is better).
$\dagger$: Muon's modular-addition adapters underperform the frozen backbone, so these cells
are not interpreted as merge results.}
\label{tab:lora}
\footnotesize
\begin{tabular}{lccc@{\hskip 1.2em}ccc}
\toprule
& \multicolumn{3}{c}{modular addition} & \multicolumn{3}{c}{character LM} \\
\cmidrule(lr){2-4}\cmidrule(lr){5-7}
optimizer & step $1$ & final & merge & step $1$ & final & merge \\
\midrule
Adam            & $\mathbf{7.9\times10^{-1}}$ & $9.7\times10^{-1}$  & $-0.15$ & $\mathbf{7.9\times10^{-1}}$ & $2.8\times10^{-1}$  & $<10^{-5}$ \\
SGD             & $7.2\times10^{-16}$ & $3.6\times10^{-15}$ & $0$     & $7.5\times10^{-16}$ & $3.6\times10^{-15}$ & $0$ \\
scalar-Adam     & $7.3\times10^{-16}$ & $6.3\times10^{-15}$ & $0$     & $7.7\times10^{-16}$ & $6.6\times10^{-15}$ & $0$ \\
Muon            & $1.9\times10^{-15}$ & $8.4\times10^{-1\,\dagger}$  & $-0.18^{\,\dagger}$ & $2.3\times10^{-15}$ & $1.0\times10^{-10}$ & $0$ \\
\midrule
\emph{Adam, noise ($10^{-7}$)} & $1.4\times10^{-6}$ & $7.8\times10^{-6}$ & $0$ & $3.2\times10^{-6}$ & $2.0\times10^{-6}$ & $0$ \\
\emph{Adam, noise ($10^{-5}$)} & $1.8\times10^{-4}$ & $7.8\times10^{-4}$ & $0$ & $5.1\times10^{-4}$ & $2.2\times10^{-4}$ & $0$ \\
\emph{Adam, different init} & $1.4$ & $1.4$ & $-0.09$ & $1.4$ & $1.3$ & $+3\times10^{-3}$ \\
\bottomrule
\end{tabular}
\end{table}

\paragraph{A larger LoRA run.}
In a run not shown in Table~\ref{tab:lora}, we adapted all feed-forward layers with rank $16$ adapters for $6000$ training steps.  We found the merge cost to be $-2.6\times10^{-4}$ nats. This is not evidence of linear-mode-connectivity without inspecting the loss along the path of interpolation \citep{frankle2020linear, neyshabur2020transferred, wortsman2022model}.

\paragraph{Gauge-twin merges (\S\ref{sec:attention}).}
For the merge results in \S\ref{sec:attention}, the twins are full attention models, not adapters, trained end-to-end on the same stream from one random initialization and its gauge rotation. We average the final parameters of the two trained models after aligning them. That is, for each head, we find the orthogonal transformation $A$ that applies identically to $Q$ and $K$ in each attention matrix of each layer and minimizes $\norm{W_Q^1 - AW_Q^2}_F^2 + \norm{W_K^1 - AW_K^2}_F^2$, the distance between the two models' factors, apply this map, and only then compute the averaged model's held-out loss.

Scalar-Adam's gauge twins agree to floating-point error and so merge at no cost. Muon's gauge twins drift only as far as its noise twins and also merge at no cost, so numerical drift alone does not produce a merge cost. Adam's $10^{-7}$ noise twins merge at no cost too, which rules out sensitivity to small perturbations as the source. Alignment is checked separately: SGD's gauge twins cost $+0.68$ to merge without it and exactly $0$ with it, so Adam's residual is not a failed alignment.

The twin models agree fairly closely in performance; Adam-twins have held-out losses that agree to $0.013$ nats within an arm. Adam's gauge twins cost $+0.052$ nats to merge on text and $+0.005$ on modular addition (scored on the training objective the twins interpolate). Different-init twins cost $+1.72$ and $+4.42$, so the gauge cost is $3\%$ and $0.1\%$ of an ordinary difference.  

\paragraph{Per-head alignment.}
A per-head orthogonal alignment can only undo the gauge, so it can align two runs only if their invariant $W_Q^\top W_K$ agrees. Adam's twins' invariants do not agree, so nothing in that family closes the gap, and other reparameterizations such as head permutation or neuron alignment \citep{ainsworth2023git} are outside this test. 

\paragraph{The fork-time sweep.}
\label{app:forksweep}

The merges above fork the twins at initialization, and this sweep asks whether the cost survives a shared prefix of training.

We pretrain a base model for $K$ steps and make two copies of it, one of them rotated by the gauge. The twins continue training for another $1500$ steps. Their optimizer states are reset.  Both branches see the same sequence of training data. 

For modular addition, the Adam-related merge penalty is gone by $K{=}25$, the shortest prefix we tested.  For the character LM, the cost is $+0.026$ to $+0.059$ over three random gauge draws at $K{=}25$ and $K{=}50$, compared to $+0.035$ to $+0.048$ at $K{=}0$, a separate run from the $+0.052$ above. It then drops seventeen-fold to $+0.0022$ to $+0.0031$ at $K{=}100$.  It continues to decrease afterwards: at $K{=}300$ we observe a slight improvement, $-0.004$. The determinism and noise twins merge at exactly $0$ at every $K$ on both tasks.  We see the same qualitative picture for signSGD, the other coordinate-wise optimizer we sweep: it has a cost of $+0.019$ to $+0.032$ for $K{=}25$ and is still higher than the noise arm at $K{=}100$.

At $K{=}300$ on text, Adam's gauge twins still differ in $W_Q^\top W_K$ by $0.28$, against $0.54$ at $K{=}0$. So a free merge does not mean the twins found the same solution. Merged without undoing the rotation, they do worse than the twins by $+0.6$ to $+0.7$ nats.

\paragraph{GPT-2-small.}
On the public GPT-2-small checkpoint, applying the head gauge of \S\ref{sec:attention} to each head's $Q$ and $K$ columns and bias slices shifts held-out loss by only $1.2\times10^{-7}$ nats. Fine-tuned rotated and unrotated copies merge at no cost after alignment . A pretrained checkpoint is the limit of a long shared prefix.

\label{app:gpt2ft}

\paragraph{GPT-2 fine-tune: model and run.}
GPT-2-small \citep{radford2019gpt2} has $12$ layers, hidden dimension $d{=}768$, and $12$ heads. We initialize from the public HuggingFace \texttt{gpt2} checkpoint (revision \texttt{607a30d783dfa663caf39e06633721c8d4cfcd7e}) via \texttt{transformers} $5.14$ and fine-tune on byte-pair-encoded tiny-shakespeare \citep{karpathy2015unreasonable} with context $128$, batch $8$, and $500$ steps in \texttt{float32} with eager attention and deterministic kernels. 

\paragraph{Arm definitions.}
\begin{itemize}\setlength{\itemsep}{1pt}\setlength{\parsep}{0pt}\setlength{\parskip}{0pt}
\item \emph{Gauge} ($\times3$): apply an independent orthogonal $A \in \R^{d_{\mathrm{head}}\times d_{\mathrm{head}}}$ to every head in every layer, leaving the $V$ block untouched. In the fused \texttt{c\_attn} projection of shape $(d, 3d)$, head $h$ reads columns $[h\,d_{\mathrm{head}}, (h{+}1)\,d_{\mathrm{head}})$ of the $Q$ and $K$ blocks. The gauge acts as $Q_h \mapsto Q_hA$, $K_h \mapsto K_hA$ with the matching bias slices transformed identically, and $A$ is drawn and applied in \texttt{float64}.
\item \emph{$A{=}I$ twin}: a determinism check.
\item \emph{Noise twins}: add iid $\mathcal{N}(0,\sigma^2)$ to the same \texttt{c\_attn} $Q/K$ slices and bias slices, with $\sigma=10^{-7}$ or $10^{-5}$ (stress case).
\end{itemize}
Alignment uses the per-head Procrustes alignment above, with the bias slice stacked as an extra row of the block.

\paragraph{Results.}
We run two rates, because a null result means more at an aggressive rate. At the standard rate $5\times10^{-5}$, zero-shot held-out loss is $4.577$, Adam reaches $3.548$, scalar-Adam $3.764$; the determinism twin merges at exactly $0$ and the $10^{-7}$-noise twin at $0$ to five decimal places. Adam's aligned merge costs $-7\times10^{-5}$ nats, so it is free. We additionally test an aggressive learning rate $5\times10^{-4}$. There scalar-Adam's held-out loss is $6.62$, meaning it has diverged, so we do not read its cells at that rate. At $5\times10^{-4}$, Adam's aligned merge improves on the twins' mean by $0.046$ nats. Without alignment, merging at the standard rate costs $+0.85$ to $+1.19$ nats.

\subsection{The nuclear-norm reference solve}
\label{app:nuclear}

Muon outperforms the convex benchmark that gradient-flow theory points to. The reference solve of Table~\ref{tab:zoo} solves $\min_{X\in\R^{n\times n}}\norm{X}_\ast$ s.t.\ $\ip{A_i}{X} = y_i$ for $i = 1,\dots,m$ on each zoo instance (\texttt{experiments/nuclear\_norm\_reference.py}). The solver certifies feasibility but not optimality, as no duality gap is computed; its constraint residual is $7\times10^{-16}$. On two seeds it returns $X^\ast$; on the third it returns a feasible point with strictly smaller nuclear norm of $38.416$ versus $X^\ast$'s $38.509$, meaning the relaxation is not tight there, and Muon still recovers $X^\ast$. So Muon's solution need not be the nuclear-norm minimizer. On a separate, $60{\times}60$ rank-$5$ problem that is not tabulated, recovery error is $0.000$ for Muon versus $0.077$ (gradient descent) and $0.465$ (Adam) (\texttt{experiments/zoo\_size\_check.py}).

\subsection{Adam-$p$: further results and pseudocode}
\label{app:dialdetails}

\paragraph{The best-rate and fixed-rate arms.}
Recovery decreases monotonically along the best-rate arm of Figure~\ref{fig:dial}, going from $0.570$ to $0.459$, $0.348$, $0.260$ and $0.201$ as $p$ goes from $1$ to $0$. This arm was obtained by taking the best recovery over a common grid of learning rates for each $p$. Sweeping at fixed learning rates that interpolate gives the recoveries in Table~\ref{tab:dialfixed}, improving monotonically from $p{=}1$ to $p{=}0$ at every rate.
\begin{table}[b]
\centering
\caption{\textbf{Adam-$p$ recovery at fixed learning rates,} at the two ends of the dial.}
\label{tab:dialfixed}
\small
\begin{tabular}{lcccc}
\toprule
learning rate & $10^{-3}$ & $3\times10^{-3}$ & $10^{-2}$ & $3\times10^{-2}$ \\
\midrule
$p{=}1$ & $0.576$ & $0.570$ & $0.573$ & $0.581$ \\
$p{=}0$ & $0.201$ & $0.256$ & $0.357$ & $0.493$ \\
\bottomrule
\end{tabular}
\end{table}  How much of the recovery is restored depends on the learning rate; for example, the largest rate $10^{-1}$ stays in the range $0.55$--$0.62$  and there the ordering is not monotone ($0.600$ at $p{=}1$, $0.617$ at $p{=}0$). 

\paragraph{The geometric-mean variant and balancedness.}
The geometric-mean variant leads to a qualitatively identical picture, with recovery decreasing monotonically as $p$ falls. At $p{=}0$ we get a recovery of $0.229$ (within $0.03$ of the Adam-$p$ endpoint) and a mean effective rank of $5.8$. We do not use it in \S\ref{sec:dial} because its $p{=}0$ limit is not gauge-invariant. The balancedness drift of Proposition~\ref{prop:balancedness} also falls with the dial, landing between Adam's and scalar-Adam's values in Table~\ref{tab:zoo}.

\paragraph{Larger scales.}
The picture is the same at other sizes on the H100. At $n{=}128$, recovery decreases monotonically from $0.67$ to $0.40$ as $p$ decreases from $1$ to $0$ (averaged over five random seeds). Mean effective rank falls monotonically from $36.6$ to $12.0$. At $n{=}40$, recovery decreases monotonically going from $0.52$ to $0.14$ as $p$ decreases from $1$ to $0$, averaged over ten random seeds and a different, finer, grid of learning rates that reaches smaller rates. Its $p{=}1$ endpoint sits below the $0.56$ to $0.58$ that \S\ref{sec:zoo} quotes for Adam. It uses ten H100 seeds rather than three CPU ones, and both endpoints are selected at the learning rate $3\times10^{-4}$, below every rate the zoo grid carries.

\begin{algorithm}[tbp]
\caption{\textbf{Adam-$p$} (\S\ref{sec:dial}). Setting $p{=}1$ recovers standard Adam; $p{=}0$ gives scalar-Adam (Proposition~\ref{prop:equivariant}(2)). Only the denominator moves with $p$: the gradient, moment estimates, bias corrections, step size, and (zero) weight decay are the same throughout. $\odot$, $\oslash$ and $(\cdot)^{\odot q}$ are elementwise multiplication, division and power.}
\label{alg:dial}
\begin{algorithmic}[1]
\Require factors $U, V$; step size $\eta$; $(\beta_1,\beta_2) = (0.9, 0.999)$;
$\epsilon = 10^{-8}$; $p \in [0,1]$; weight decay $0$
\State $m_\theta \gets 0$, \; $v_\theta \gets 0$ \; for each factor $\theta \in \{U, V\}$
\For{$t = 1, 2, \dots$}
  \For{each factor $\theta \in \{U, V\}$}
    \State $g_\theta \gets \nabla_\theta L(U,V)$
      \Comment{$L(U,V) = f(UV^\top)$; covariant: $g_\theta \mapsto g_\theta Q$}
    \State $m_\theta \gets \beta_1 m_\theta + (1-\beta_1)\, g_\theta$, \quad
           $v_\theta \gets \beta_2 v_\theta + (1-\beta_2)\, g_\theta^{\odot 2}$
    \State $\hat m_\theta \gets m_\theta / (1 - \beta_1^{\,t})$, \quad
           $\hat v_\theta \gets v_\theta / (1 - \beta_2^{\,t})$
  \EndFor
  \State $\bar s \gets \bigl(\operatorname{mean}(\hat v_U, \hat v_V)\bigr)^{1/2}$
         \label{alg:line:scalar}
      \Comment{one mean, pooled over \emph{both} factors: gauge-invariant}
  \For{each factor $\theta \in \{U, V\}$}
    \State $D_\theta \gets \bigl(\hat v_\theta^{\odot 1/2} + \epsilon\bigr)^{\odot p}
           \odot (\bar s + \epsilon)^{1-p}$
      \Comment{$p{=}1$: per-coordinate; $p{=}0$: one shared scalar}
    \If{$p < 1$}
      \State $\theta \gets \theta - \eta\, \hat m_\theta \oslash (D_\theta + \epsilon)$
    \Else
      \State $\theta \gets \theta - \eta\, \hat m_\theta \oslash D_\theta$
        \Comment{$D_\theta = \hat v_\theta^{\odot 1/2} + \epsilon$ here: stock Adam}
    \EndIf
  \EndFor
\EndFor

\end{algorithmic}
\end{algorithm}

\paragraph{Equivariance of Line~\ref{alg:line:scalar}.} By pooling $\hat v$ over all
entries of the two factors, we make $\bar s$ depend only on the history of
$\norm{G_U}_F^2 + \norm{G_V}_F^2$, which is invariant under gauge transformations by
Lemma~\ref{lem:grad-cov}.  Because taking the entry mean commutes with the exponential moving average, the
pooled $\bar s^{\,2}$ equals the bias-corrected scalar $\nu_t$ of
Proposition~\ref{prop:equivariant}(2) exactly. This makes the parameter trajectory the equivariant
one.

 At $p{=}0$ the stored $v_\theta$ is
elementwise and admits no $\sigma_Q$ of its own, because squaring does not commute with a right
rotation. However, the non-pooled components never reach the update. The step at $p{=}0$ depends on
$v_\theta$ only through the pooled scalar $\bar s$, and otherwise only on $\hat m_\theta$. Both of
these do transform. That is exactly the case Definition~\ref{def:equivariance} allows, since it
quantifies over the state the update reads.

\subsection{The row-pooled control}
\label{app:rowpool}

Adam-$p$ (\S\ref{sec:dial}) changes two things at once. Going from $p{=}1$ to $p{=}0$ removes the
anisotropy (one shared denominator for all coordinates of both factors) and restores gauge equivariance.
This control arm keeps anisotropy, with one denominator per row, and restores only the symmetry.

\paragraph{The rule.} Pool the second moment over each row of a factor matrix, rather
than over each coordinate. Run Adam's first and second moments as usual, with the same
$\beta_1{=}0.9$, $\beta_2{=}0.999$ and $\epsilon{=}10^{-8}$ and the same bias correction. Then, for
each factor $P \in \{U,V\}$ with rows indexed by $i$, use one denominator per row,
\[
\Delta P_{i\cdot} \;=\; -\eta\,\frac{\hat m_{i\cdot}}{\sqrt{\tfrac1k\sum_{j=1}^{k}\hat v_{ij}}+\epsilon}.
\]
The denominator varies from row to row, so the rule is anisotropic. It is also gauge-equivariant.
Under $(U,V)\mapsto(UQ,VQ)$ the gradient rows transform as $G_{i\cdot}\mapsto G_{i\cdot}Q$
(Lemma~\ref{lem:grad-cov}), so $\sum_j \hat v_{ij}$, an average of past $\norm{G_{i\cdot}}^2$, is unchanged and the
update transforms the same way the parameters do. This is the matrix-shaped case of the vector
pooling of \citet{ling2022vectoradam}, and it is an instance of
Definition~\ref{def:equivariance} with $\sigma_Q$ acting on the first moment alone.

\paragraph{Protocol.} The testbed, seeds and rate grid are those of \S\ref{sec:dial}: rates $\{10^{-3}, 3\times10^{-3}, 10^{-2}, 3\times10^{-2}, 10^{-1}\}$, and
the $10^{-7}$ interpolation bar with a $3\times10^{4}$-step cap. Each arm is reported at its best interpolating rate, as in \S\ref{sec:zoo}.

\paragraph{Gauge check.} Running $60$ steps at $\eta{=}10^{-2}$ from a start and from its
right-rotated twin, and comparing the products $UV^\top$, the row-pooled arm differs by
$2.0\times10^{-7}$ in relative Frobenius distance. This is the float32 noise floor of this harness
(Appendix~\ref{app:details}). Adam under the same check differs by $0.34$. 

\paragraph{Result.} Table~\ref{tab:rowpool} gives the three arms. Restoring equivariance while
keeping the anisotropy buys back about $30\%$ of the recovery distance from Adam to scalar-Adam,
and about $60\%$ of the effective-rank distance ($14.49$ to $9.00$, against scalar-Adam's $5.43$). The remaining recovery gap needs the anisotropy
to go as well.

 The three arms interpolate at different rates and in different numbers of steps. At the matched rate $10^{-3}$ they read $0.576$, $0.461$ and $0.201$, so the split is not an artifact of the selection.

\begin{table}[tbp]
\centering\small
\caption{\textbf{Row pooling reduces recovery error, and pooling both factors reduces it further.}
The three methods differ only in second-moment pooling.
Each uses its best interpolating rate on the shared grid, evaluated over three sensing seeds.}
\label{tab:rowpool}
\begin{tabular}{llccccc}
\toprule
arm & pooled over & equivariant & recovery & s.d. & $\er$ & lr \\
\midrule
Adam ($p{=}1$)        & nothing        & no  & 0.5704 & 0.0505 & 14.49 & $3\times10^{-3}$ \\
row-pooled            & a factor row   & yes & 0.4609 & 0.1063 & \phantom{0}9.00 & $10^{-3}$ \\
scalar-Adam ($p{=}0$) & both factors   & yes & 0.2010 & 0.0381 & \phantom{0}5.43 & $10^{-3}$ \\
\bottomrule
\end{tabular}
\end{table}

\subsection{The factor-aware preconditioner of Remark~\ref{rem:necessary}}
\label{app:scaledgd}

Remark~\ref{rem:necessary} needs an equivariant method that still recovers badly, to show that equivariance is not sufficient for good recovery. It is our own variant in the style of ScaledGD \citep{tong2021accelerating}, not the
published algorithm.

\paragraph{The rule.} With damping $\lambda>0$ and rate $\eta$,
\[
U_{t+1} = U_t - \eta\, G^U_t\bigl(V_t^\top V_t + \lambda I_k\bigr)^{-1},
\qquad
V_{t+1} = V_t - \eta\, G^V_t\bigl(U_t^\top U_t + \lambda I_k\bigr)^{-1},
\]
where $G^U_t, G^V_t$ are the gradients of $L$ at $(U_t,V_t)$. It is equivariant: under
$(U,V)\mapsto(UQ,VQ)$ we get $V^\top V \mapsto Q^\top V^\top V Q$ and $G^U \mapsto G^U Q$, so the
whole update picks up one factor of $Q$ on the right, as Definition~\ref{def:equivariance} asks.
Damping is what keeps it defined at the small initialization, where $V^\top V$ is near-singular.

\paragraph{Grid and selection.} Same testbed, seeds, bar and rate grid as
Appendix~\ref{app:rowpool}, crossed with damping $\{10^{-3}, 10^{-2}, 10^{-1}, 1\}$ and selected on best recovery among the cells
that reach the bar. Three cells diverge, all at $\eta{=}10^{-1}$ with
$\lambda \le 10^{-1}$.

\paragraph{Result.} The selected cell is $\eta{=}10^{-1}$, $\lambda{=}1$: recovery $0.6092$
(s.d.\ $0.0135$), $\er 13.68$, train loss
$8.4\times10^{-12}$, $200$ steps.

Recovery is worse than Adam's $0.5734$ (Table~\ref{tab:zoo}), so an equivariant method can finish outside the equivariant
cluster of Figure~\ref{fig:zoo} entirely. Every interpolating cell in the grid but one recovers between
$0.61$ and $0.94$, so this is not one unlucky corner. The exception, $\eta{=}3\times10^{-2}$ at
$\lambda{=}10^{-3}$, interpolates but has recovery error $3.54$.

Damping is the axis that matters. At $\lambda{=}1$ every rate gives $0.61$, and as $\lambda$ falls
toward $10^{-3}$ recovery degrades to $0.85$--$0.88$ at the other rates, with effective rank rising from $13.5$ to
$30$. The reading is the schedule axis of \S\ref{sec:phase}: preconditioning by the other factor's
Gram matrix flattens the spectrum .

\subsection{The protocol}
\label{app:protocol}

  All learning rates are fixed or have a cosine schedule, except one linear-decay control in Appendix~\ref{app:c10}. In the case of optimizers that have a fixed-norm update such as Muon, signum, and Lion, and for RMSProp (Appendix~\ref{app:c7}), we use learning rate decay to give a more even comparison to other methods. This is a per-method criterion rather than a class-level one, informed by observations rather than theory. To ensure that this is not confounding the results, we also try learning rate decay (full schedule) for all optimizers in our experiments (Appendix~\ref{app:c1}). To avoid outcome and speed bias, we show the full recovery vs.\ learning rate curves in our experiments (Appendix~\ref{app:c2}).

\subsection{Further hyperspectral results}
\label{app:predictordist}

Table~\ref{tab:flowlimit} suggests that GD and scalar-Adam find similar predictors despite very different convergence rates, so we measure the distance between the fitted matrices directly. We compare GD to scalar-Adam, matching instance, initialization seed, mask and preprocessing, and measure at matching training-loss values that agree to within $1/50$, with GD at its selected rate as the reference for the $d_W$ of \S\ref{sec:realdata}.
We use four paired seeds, the Pavia University dataset and the two sampling densities of Table~\ref{tab:flowlimit}, with full code in \texttt{predictor\_distance2.py}. Evaluated at the $10^{-5}$ loss value, as the rate decreases, $d_W$ (averaged over seeds) monotonically decreases by a factor of about $33$ to $36$ in both densities. The differences in held-out error and effective rank fall with $d_W$, which explains the matching values in Table~\ref{tab:flowlimit}.  We can rule out noise in the GD reference because the average and maximum relative distance between predictors from different GD runs, at rate $10$ versus rate $30$ evaluated at matching thresholds, are $6\times10^{-5}$ and $1.0\times10^{-4}$, at least fivefold smaller than the minimum $d_W$ value. Table~\ref{tab:pdistfour} extends the rates to $10^{-6}$, where scalar-Adam's $d_W$ on Pavia reaches the GD floor. 

\begin{table}[tbp]
\centering
\caption{\textbf{Scalar-Adam approaches GD's error and effective rank as its learning rate falls.}
Entries show held-out root-mean-square error (RMSE, $\times10^{-2}$) / effective rank, averaged over four seeds.
Training loss is matched at $\le10^{-5}$. Adam-variant rates span $3\times10^{-3}$ to $10^{-6}$.
Slope measures scalar-Adam's log--log approach to GD over the five smallest rates.
Best RMSE and Adam floor are Adam's lowest error and lowest effective rank over its grid.
Gap subtracts scalar-Adam's rank at its smallest reaching rate from Adam's floor on the same seed (paired mean $\pm$ standard deviation). Bold gaps have a $95\%$ $t$-interval that excludes zero.}
\label{tab:flowlimit}
\vspace{2pt}
\small\setlength{\tabcolsep}{4pt}
\begin{tabular}{llcccccc}
\toprule
& & \multicolumn{2}{c}{equivariant} & \multicolumn{4}{c}{coordinate-wise Adam} \\
\cmidrule(lr){3-4}\cmidrule(lr){5-8}
dataset & $m/\mathrm{dof}_{24}$ & GD & scalar-Adam & slope & best RMSE & Adam floor & gap \\
\midrule
\multirow{2}{*}{Indian Pines} & $1.15$ & $1.514$ / $13.57$ & $1.515$ / $13.58$ & $0.97$ & $1.467$ & $13.68$ & $+0.09 \pm 0.11$ \\
                              & $1.9$  & $1.262$ / $13.30$ & $1.262$ / $13.31$ & $0.93$ & $1.264$ & $\mathbf{13.75}$ & $+0.44 \pm 0.06$ \\
\multirow{2}{*}{Pavia Univ.}  & $1.15$ & $0.867$ / $\phantom{0}6.34$ & $0.868$ / $\phantom{0}6.35$ & $1.12$ & $0.892$ & $\phantom{0}\mathbf{6.88}$ & $+0.53 \pm 0.08$ \\
                              & $1.9$  & $0.631$ / $\phantom{0}5.60$ & $0.632$ / $\phantom{0}5.60$ & $0.99$ & $0.679$ & $\phantom{0}\mathbf{6.21}$ & $+0.61 \pm 0.06$ \\
\bottomrule
\end{tabular}
\end{table}

\begin{table}[tbp]
\centering
\caption{\textbf{Scalar-Adam approaches GD's fitted matrix, while Adam remains separated.}
Entries are mean relative Frobenius distances over four seeds, in percent, at matched training loss $\le10^{-5}$.
Floor is the distance between GD runs at two reference rates.
Ratio divides Adam's distance by scalar-Adam's at $10^{-6}$, before rounding.}
\label{tab:pdistfour}
\vspace{2pt}
\small\setlength{\tabcolsep}{4pt}
\begin{tabular}{lccccccccr}
\toprule
& & & \multicolumn{3}{c}{scalar-Adam} & \multicolumn{3}{c}{Adam} & \\
\cmidrule(lr){4-6}\cmidrule(lr){7-9}
dataset & $m/\mathrm{dof}_{24}$ & floor & $3{\cdot}10^{-3}$ & $10^{-4}$ & $10^{-6}$ & $3{\cdot}10^{-3}$ & $10^{-4}$ & $10^{-6}$ & ratio \\
\midrule
Botswana      & $1.15$ & $0.004$ & $22.990$ & $4.008$ & $\mathbf{0.059}$ & $36.22$ & $17.06$ & $\mathbf{17.80}$ & $300\times$ \\
              & $1.9$  & $0.002$ & $17.396$ & $2.971$ & $\mathbf{0.031}$ & $32.37$ & $14.65$ & $\mathbf{15.28}$ & $499\times$ \\
Indian Pines  & $1.15$ & $0.017$ & $21.625$ & $3.804$ & $\mathbf{0.045}$ & $24.93$ & $17.72$ & $\mathbf{17.48}$ & $385\times$ \\
              & $1.9$  & $0.008$ & $10.821$ & $1.916$ & $\mathbf{0.028}$ & $15.68$ & $14.20$ & $\mathbf{13.50}$ & $482\times$ \\
Pavia Univ.   & $1.15$ & $0.009$ & $\phantom{0}5.179$ & $0.847$ & $\mathbf{0.004}$ & $\phantom{0}6.95$ & $\phantom{0}5.61$ & $\phantom{0}\mathbf{6.85}$ & $1610\times$ \\
              & $1.9$  & $0.003$ & $\phantom{0}1.830$ & $0.303$ & $\mathbf{0.003}$ & $\phantom{0}3.58$ & $\phantom{0}3.87$ & $\phantom{0}\mathbf{4.78}$ & $1700\times$ \\
Salinas       & $1.15$ & $0.023$ & $23.271$ & $3.954$ & $\mathbf{0.044}$ & $18.49$ & $11.55$ & $\mathbf{12.79}$ & $293\times$ \\
              & $1.9$  & $0.008$ & $\phantom{0}9.573$ & $1.497$ & $\mathbf{0.020}$ & $\phantom{0}8.10$ & $\phantom{0}6.54$ & $\phantom{0}\mathbf{7.84}$ & $387\times$ \\
\bottomrule
\end{tabular}
\end{table}

\begin{table}[tbp]
\centering
\caption{\textbf{At the smallest rates, scalar-Adam's rank offsets are small while Adam's remain positive on these two scenes.}
Entries subtract GD's effective rank at matched training loss $\le10^{-5}$, averaged over four seeds.
Each cell reports scalar-Adam / Adam at the same learning rate $\eta$.
The grid extends Table~\ref{tab:flowlimit}. 
These are signed rank differences, not matrix distances.}
\label{tab:limitrank}
\vspace{2pt}
\small\setlength{\tabcolsep}{5pt}
\begin{tabular}{lcccc}
\toprule
& \multicolumn{2}{c}{Indian Pines} & \multicolumn{2}{c}{Pavia Univ.} \\
\cmidrule(lr){2-3}\cmidrule(lr){4-5}
$\eta$ & $m/\mathrm{dof}_{24}{=}1.15$ & $1.9$ & $1.15$ & $1.9$ \\
\midrule
$3\times10^{-3}$ & $7.319$ / $6.867$ & $3.524$ / $3.156$ & $2.012$ / $1.920$ & $0.720$ / $0.797$ \\
$10^{-3}$ & $3.696$ / $3.385$ & $1.732$ / $1.883$ & $1.001$ / $0.979$ & $0.349$ / $0.607$ \\
$3\times10^{-4}$ & $1.936$ / $1.404$ & $0.881$ / $1.370$ & $0.509$ / $0.605$ & $0.173$ / $0.631$ \\
$10^{-4}$ & $1.028$ / $0.503$ & $0.465$ / $1.083$ & $0.268$ / $0.533$ & $0.092$ / $0.684$ \\
$3\times10^{-5}$ & $0.435$ / $0.148$ & $0.201$ / $0.708$ & $0.115$ / $0.599$ & $0.041$ / $0.727$ \\
$10^{-5}$ & $0.167$ / $0.108$ & $0.081$ / $0.442$ & $0.045$ / $0.847$ & $0.017$ / $0.931$ \\
$3\times10^{-6}$ & $0.049$ / $0.420$ & $0.025$ / $0.624$ & $0.012$ / $1.191$ & $0.005$ / $1.127$ \\
$10^{-6}$ & $0.012$ / $0.778$ & $0.007$ / $0.802$ & $0.001$ / $1.349$ & $0.001$ / $1.180$ \\
$3\times10^{-7}$ & $-0.001$ / $1.208$ & $-0.000$ / $0.947$ & $-0.003$ / $1.452$ & $-0.001$ / $1.202$ \\
$10^{-7}$ & $\mathbf{-0.005}$ / $\mathbf{1.492}$ & $\mathbf{-0.002}$ / $\mathbf{1.062}$ & $\mathbf{-0.004}$ / $\mathbf{1.539}$ & $\mathbf{-0.001}$ / $\mathbf{1.218}$ \\
\bottomrule
\end{tabular}
\end{table}

Matching loss values across a finite number of checkpoints shows predictor convergence over our range of rates, not convergence to a common limit.

\paragraph{Matched-loss trajectories.}

In Figure~\ref{fig:realtraj}, GD's fitting is consistent with a greedy pattern that fits the largest singular values first, steadily decreasing the test error. Adam's test error actually increases from $0.0253$ to $0.0271$ when the training loss decreases from $3\times10^{-4}$ to $3\times10^{-5}$, and so does the effective rank, from $22$ to $28$. Muon's effective rank stays close to the model's rank cap until quite late in training, then decreases from $46$ to $36$, and its held-out error improves little until then. These observations suggest that deep fitting in the interpolation regime with a gauge-breaking optimizer can hurt generalization performance on the real data.

\begin{figure}[tp]
\centering
\includegraphics[width=\linewidth]{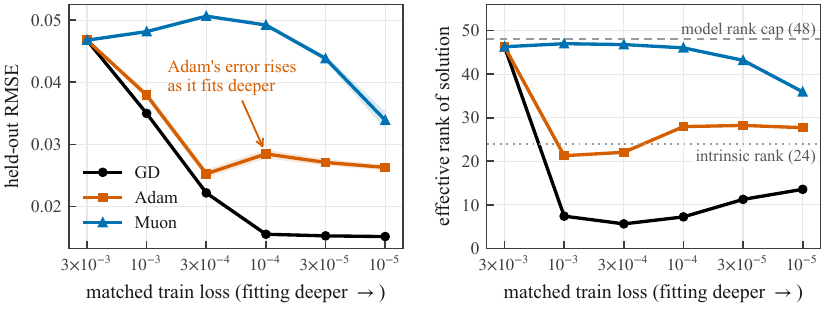}
\caption{\textbf{Held-out performance against training loss on real data} (Indian Pines,
$m/\mathrm{dof}\approx1.15$, $4$ seeds, GPU \texttt{float64}; learning rates from the train-only rule
in Appendix~\ref{app:c6}). \emph{Left:} held-out error, with $\pm1$ s.d.\ bands over the $4$ seeds. \emph{Right:} effective rank.}
\label{fig:realtraj}
\end{figure}

\paragraph{Adam at small learning rates.}
On Pavia at density $0.28$, as the learning rate $\eta$ approaches $0$, Adam's held-out error deteriorates from $0.00892$ to $0.00961$ and its effective rank consistently increases from $6.88$ to $7.69$. The direction holds in all $8$ paired comparisons (four seeds, two statistics).

\subsection{The FlowAdam repair}
\label{app:repair}

Our earlier work \citep{singh2026flowadam} proposed FlowAdam, a variant of Adam that feeds the clipped gradient-flow velocity $v \mapsto \mathrm{clip}(v,-c,c)$ into Adam's momentum. Run without weight decay, it reaches a recovery error of $0.347$, better than Adam's $0.573$ but far from GD's $0.131$. Clipping each coordinate of $v$ separately breaks the gauge (Proposition~\ref{prop:coordwise}). Clipping the norm of the whole vector instead keeps it, and lowers the error to $0.219$, reached at training loss $6\times10^{-7}$. The rest of the gap comes from Adam's coordinate-wise denominator, which also breaks the gauge. Scalar-Adam shares that denominator across all coordinates, so we recommend scalar-Adam over a repaired FlowAdam.

\section{Extended related work}
\label{app:related}

\begin{table}[tbp]
\centering
\caption{\textbf{Prior work studies related symmetries, while this paper tests which fitted solutions optimizers select.}}
\label{tab:litmap}
\vspace{2pt}
\small
\renewcommand{\arraystretch}{1.12}
\begin{tabular}{@{}L{3.0cm}L{4.0cm}L{5.5cm}@{}}
\toprule
work & setting and result & relation to this paper \\
\midrule
\methodlink{ling2022vectoradam}{VectorAdam} \citep{ling2022vectoradam}; \citet{xie2025linfty}
& Examine ambient rotations of vector parameters under general loss functions. Show Adam is sensitive to such rotations, whereas shared scalar updates restore equivariance.
& We focus instead on an internal factorization gauge and on solution selection among interpolants. \\
\methodlink{yen2025lorarite}{LoRA-RITE} \citep{yen2025lorarite}; \citet{depavia2025rotations}
& Study transformations of LoRA factors and rotations in data or feature space. Demonstrate improved optimization through invariance and altered implicit bias in Adam under rotations.
& We isolate the intrinsic gauge of the factorization itself and evaluate recovery performance at matched training loss across a range of optimizers. \\
\citet{silverstein2026symmetry}; \citet{zhang2026polaradamw}
& Analyze per-head query--key (QK) rotations and Schur multiplicity bases. Highlight that optimizer design can intentionally break or preserve symmetry.
& We maintain the gauge structure intact and use it to analyze solutions chosen by standard, deployed optimizers. \\
\citet{lau2026symmetry}; \citet{shirodkar2026ddc}
& Investigate architectural symmetry groups and construct optimizers that exactly respect equivariance.
& These works support symmetry-aware design; we instead assess the practical consequences of preserving or violating such symmetry. \\
\citet{kang2026uniform}; \citet{dragutinovic2026muon}
& Study Muon's spectral dynamics and its impact on simplicity bias.
& We extend this by mapping conditions under which uniform spectral growth helps or harms performance as the target develops a spectral tail. \\
\bottomrule
\end{tabular}
\end{table}

\paragraph{Low-rank bias of gradient flow.}
In the context of the matrix sensing problem, in the underdetermined regime, it was first conjectured by \citet{gunasekar2017implicit}, and shown for commuting measurement matrices, that the overparameterized solution converges to the minimum-nuclear-norm interpolant.  Later, \citet{arora2019implicit} showed that such implicit bias extends to deep matrix factorizations. They advocated a dynamical, rather than norm-based, point of view, later refined in, e.g., \citet{razin2020implicit}, as well as in the context of the greedy low-rank learning dynamics in \citet{li2021towards}.

Gradient flow has been studied in factored models by \citet{arora2018optimization,du2018algorithmic}, who noticed that the balancedness $U^\top U-V^\top V$ is a conserved quantity. The corresponding symmetry is the general linear gauge symmetry $(U,V)\mapsto(UA,VA^{-\top})$ for $A\in\mathrm{GL}(k)$, which, in the Noether sense of \citet{kunin2021neural}, is what conserves it, rather than the orthogonal gauge we study.  Recent work by \citet{chou2024deepmf} studied the dynamics of deep factorizations and showed that the low-rank bias becomes more pronounced with depth. \citet{chou2024weightnorm} showed that weight-normalized models exhibit a low-rank bias robust to the initialization scale. All this work considers gradient descent or flow, and does not discuss optimizers that are deployed in practice.

\paragraph{Margins and norms.}
Implicit bias of Adam and other adaptive optimizers was first considered by \citet{wilson2017marginal}, who showed that adaptive methods find solutions with worse generalization. Recent work has made progress on AdamW \citep{loshchilov2019decoupled}, Adam, and related algorithms. See, for example, \citet{xie2024implicit}, who characterize AdamW through an $\ell_\infty$-norm constraint, \citet{zhang2024implicit}, who show Karush--Kuhn--Tucker convergence for Adam, or \citet{baek2026implicit}, who give per-sample refinements. Most of this work builds on~\citet{soudry2018implicit} and considers the classification margin in separable classification problems. Another line of work, which is not margin-based, includes backward error analyses of Adam and RMSProp, showing that the corresponding modified flows penalize a perturbed one-norm of the loss gradients or, more often, slow down its decrease \citep{cattaneo2024implicit}. \citet{li2025adam} identify a sharpness that Adam reduces near the minimizer manifold.

\paragraph{Rotations of the parameter space.} \citet{zhang2025rotation} find that Adam's training of transformers degrades under random rotations of the parameter space. Our twin protocol allows us to quantitatively measure the degree to which Adam depends on the latent basis, and our planted ground truth lets us validate our theory's predictions directly. Both of these, along with the spectral-tail phase diagram, have no counterpart in this literature.

\paragraph{Continuous-time arguments and time rescaling.}
The change of time in Theorem~\ref{thm:transfer} is a standard reparametrization of autonomous flows, usually used to show acceleration or to infer convergence rates from Lyapunov functions. Here convergence rates are discarded. What matters is that any clock with $a>0$ and $1/a$ locally integrable leaves the path and its limit unchanged, so any implicit-bias statement about them transfers to the rescaled flow. Gauge invariance of the clock is needed only to make this hold uniformly over rotated starts.

\paragraph{Connection to equivariant architectures.}
\citet{lawrence2022implicit} show that for a family of equivariant linear $G$-CNN architectures indexed by any finite group $G$, gradient descent converges to a solution whose Fourier coefficients are low rank under a Schatten norm. There the symmetry is a property of the problem and not of the parameterization, and it is built into the model through an equivariant layer. On the other hand, the group acting on the parameterization in our setting is a redundancy: multiplying both factors in the matrix factorization by the same $\OO(k)$ transformation does not affect the model's predictions. More closely related is the work of \citet{ravanbakhsh2017equivariance}, who give an if-and-only-if characterization of the groups $G$ of transformations under which a given parameter-sharing layer is equivariant, in terms of $G$ being a group of symmetries of the weights. Similarly, Theorem~\ref{thm:structure} characterizes which memoryless update rules are equivariant under the gauge group of our parameterization. The notion of equivariance is different, though. We are asking whether a particular update commutes with a group of transformations, and the results do not imply each other.

\paragraph{Connection to SGD and symmetry.}
\citet{ziyin2024parameter} study continuous symmetries of the parameters and show that the noise introduced by stochasticity systematically biases SGD to move along these symmetries and converge to a common point that is independent of the initialization.   We are concerned with deterministic updates and symmetries and ask if a given update commutes with these, while \citet{ziyin2024parameter} ask how SGD moves along symmetry orbits.

\paragraph{Transformation invariance in optimizer design.}
 \citet{ling2022vectoradam} study the rotation-equivariant case for vector-valued geometric parameters, observing that Adam's per-coordinate moments are not rotation-equivariant and proposing VectorAdam, which uses a single scalar per vector, restoring equivariance. Our $p{=}0$ endpoint uses a single scalar for both factors. VectorAdam pools the second moment over each vector, so it can still scale different vectors differently. Applied to the rows of a factor, it is the analogue of our row-pooled control (Appendix~\ref{app:rowpool}), which is equivariant and still anisotropic across rows.

Independently, \citet{yen2025lorarite} propose LoRA-RITE, a transformation-invariant preconditioner for LoRA factors. Like us, they observe that failure to be transformation-invariant is a source of error in the optimizer. They focus on speedup, whereas we additionally find that lack of invariance can introduce a bias in the resulting solution (\S\ref{sec:zoo}, \S\ref{sec:attention}). 

\paragraph{Relation to equivariant deep learning.}
\citet{zhang2026polaradamw} finds that, in group-equivariant networks, Muon's polar step is equivariant to orthogonal transformations when parameters are written in the Schur (multiplicity) basis, whereas AdamW is not, and proposes PolarAdamW, which violates equivariance in order to allow coordinate-wise adaptation.  \citet{lau2026symmetry} consider optimizers which respect permutation and shift symmetries in neural network parameters, providing a framework which encompasses spectral descent, Muon, Scion, and other optimizers including polar-type updates. \citet{shirodkar2026ddc} considers how to derive equivariant optimizers in a variety of architectural gauges. The paper shows that a version of Adam resists AdamW's over-training collapse and that a Muon variant groks at previously unreachable depth, both consistent with the mechanism we propose here. Closer to our setting, \citet{stupariu2026equivariant} compare Muon to Adam in an equivariant setting.

\paragraph{Implicit bias from stochastic symmetries.}
\citet{aladrah2026implicit} combine continuous symmetries of the loss with SGD's stochasticity to obtain a geometric correction, based on quotient volume, to the effective loss. In particular, they consider the symmetry $(UD, VD^{-1})$ for positive rescaling of the feature map, which is a subgroup of $\mathrm{GL}(k)$ but not of our orthogonal gauge. Our framework is complementary to their noise-dependent mechanism: it offers a deterministic and optimizer-dependent account of implicit biases. 

\paragraph{Structured optimizers.}
Recent work has explored the implicit bias of structured optimizers such as Shampoo \citep{gupta2018shampoo} and Muon \citep{jordan2024muon,liu2025muon}, the latter performing orthogonalization by the matrix-sign function applied to the momentum. \citet{fan2025implicit,gronich2026implicit} attribute a spectral max-margin bias to Muon, and \citet{kang2026uniform} prove uniform growth of singular values under Muon for low-rank adaptation (LoRA) factorizations. Our phase diagram recasts these insights through the lens of Theorem~\ref{thm:structure}.  Concurrently, \citet{beneventano2026spectral} give an alternative justification of Muon's polar step as one-step entropy maximization and analyze Muon's exact singular-value dynamics in underdetermined regression. \citet{shen2026rivervalley} report a consistent trade-off in experiments on mixed-spectrum sensing. Finally, \citet{dong2026muonp} propose a within-class family, Muon with a fractional power of the gradient singular values, interpolating between Muon and GD along the spectral-transfer axis of Proposition~\ref{prop:spectral}. 

\end{document}